\PassOptionsToPackage{table}{xcolor}
\PassOptionsToPackage{authoryear,sort&compress,round}{natbib}
\documentclass[11pt,a4paper,logo]{googlecloud}
\pdftrailerid{redacted}
\usepackage{natbib}
\usepackage[utf8]{inputenc}
\usepackage[T1]{fontenc}
\usepackage{amsmath,amssymb,amsthm}
\usepackage{newtxmath}
\usepackage{graphicx}
\usepackage{adjustbox}
\usepackage[table]{xcolor}
\usepackage{booktabs}
\usepackage{microtype,xspace}
\usepackage{enumitem}
\usepackage{algorithm,algpseudocode}
\usepackage[skins,breakable]{tcolorbox}
\usepackage{hyperref}
\usepackage{url}
\usepackage{cleveref}
\hypersetup{colorlinks=true,allcolors=blue,
  pdftitle={RubricEM: Meta-RL with Rubric-guided Policy Decomposition beyond Verifiable Rewards},
  pdfauthor={Gaotang Li, Bhavana Dalvi Mishra, Zifeng Wang, Jun Yan, Yanfei Chen, Chun-Liang Li, Long T. Le, Rujun Han, George Lee, Hanghang Tong, Chen-Yu Lee, Tomas Pfister}}
\AtBeginEnvironment{algorithm}{\linespread{0.99}\selectfont}

\newtheorem{lemma}{Lemma}
\newtheorem{theorem}{Theorem}
\newtheorem{corollary}{Corollary}
\newtheorem{definition}{Definition}
\newtheorem{assumption}{Assumption}
\newtheorem{remark}{Remark}
\newcommand{\ours}{\textup{\textsc{RubricEM}}\xspace}
\newcommand{\model}{\textup{\textsc{RubricEM}-8B}\xspace}

\definecolor{LightCoral}{RGB}{240,128,128}
\definecolor{LightBlue}{RGB}{173,216,230}
\tcolorboxenvironment{remark}{
    colback=LightCoral!10,       
    colframe=black,     
    arc=1mm,               
    top=.6mm,
    bottom=.6mm,
    boxrule=.5pt,           
    right=1mm,
    fonttitle=\bfseries,   
    left=1mm,              
    attach boxed title to top left={yshift=-2mm, xshift=4mm},
    boxed title style={
        colback=black!40,
        arc=1mm,
    },
    before skip=6pt,   
    after skip=6pt,    
}

\tcolorboxenvironment{theorem}{
    colback=LightBlue!20,       
    colframe=black,     
    arc=1mm,               
    top=.6mm,
    bottom=.6mm,
    right=1mm,
    boxrule=.5pt,           
    fonttitle=\bfseries,   
    left=1mm,              
    attach boxed title to top left={yshift=-2mm, xshift=4mm},
    boxed title style={
        colback=black!40,
        arc=1mm,
    },
    before skip=6pt,   
    after skip=6pt,    
}

\definecolor{tblHeader}{HTML}{F7F8FA}
\definecolor{tblRule}{HTML}{4B5563}
\definecolor{tblMuted}{HTML}{8A8A8A}

\definecolor{tblClosedHead}{HTML}{E3E6EA}
\definecolor{tblClosedRow}{HTML}{F2F4F6}

\definecolor{tblFixedHead}{HTML}{DCE9F5}
\definecolor{tblFixedRow}{HTML}{F0F6FB}

\definecolor{tblOpenHead}{HTML}{EEE5D2}
\definecolor{tblOpenRow}{HTML}{FAF6EC}

\definecolor{tblOursHead}{HTML}{D6ECE5}
\definecolor{tblOursRow}{HTML}{EDF7F4}
\definecolor{tblOursHi}{HTML}{DFF1EA}

\newcommand{\tblbest}[1]{\textbf{#1}}
\newcommand{\tblNA}{\textcolor{tblMuted}{--}}

\newcommand{\tblsection}[2]{%
  \rowcolor{#1}%
  \multicolumn{6}{l}{\textbf{\textit{#2}}}\\
}

\newcommand{\takeawayicon}{%
  \tikz[x=1pt,y=1pt,baseline=0pt]{%
    \draw[line width=0.7pt,line cap=round]
      (3.2,3.2) circle[radius=2.6];
    \draw[line width=0.9pt,line cap=round] (5.1,1.3) -- (7.6,-1.2);
  }%
}
\newcommand{\experimenttakeaway}[1]{%
  \par\nobreak
\begingroup\setlength{\parskip}{3.5pt}\interlinepenalty=10000
\noindent
  \takeawayicon\enspace\textbf{\textit{Takeaway. #1}}\par
  \endgroup
}

\definecolor{bhavanacol}{rgb}{0.90,0.30,0.90}

\renewcommand{\today}{2026-10-15}

\title{RubricEM: Meta-RL with Rubric-guided Policy Decomposition beyond Verifiable Rewards}
\correspondingauthor{Gaotang Li \href{mailto:gaotang3@illinois.edu}{<gaotang3@illinois.edu>} and Bhavana Dalvi Mishra \href{mailto:bhavana@google.com}{<bhavana@google.com>}. \\This work was done while Gaotang Li interned at Google Cloud AI Research.}

\author[1]{Gaotang Li}
\author[2]{Bhavana Dalvi Mishra}
\author[2]{Zifeng Wang}
\author[2]{Jun Yan}
\author[2]{Yanfei Chen}
\author[2]{Chun-Liang Li}
\author[2]{Long T. Le}
\author[2]{Rujun Han}
\author[2]{George Lee}
\author[1]{Hanghang Tong}
\author[2]{Chen-Yu Lee}
\author[2]{Tomas Pfister}

\affil[1]{University of Illinois Urbana-Champaign}
\affil[2]{Google Cloud AI Research}

\begin{document}

\begin{abstract}
Deep research agents plan, search, evaluate evidence, and synthesize long-form reports.
Training these agents with reinforcement learning (RL) is challenging because they operate over long horizons and their outputs lack exact answer verifiers.
To this end, we introduce \ours, a unified RL framework for deep research that combines stagewise credit assignment with reflection-based meta-RL for experience reuse.
Rubrics unify the framework by providing a shared interface for policy execution, judge feedback, and experience reuse.
\ours first instills a rubric-guided scaffold through supervised fine-tuning, teaching the agent to plan with rubrics and organize its trajectory into semantic stages.
Building on this structural prior, we optimize the task policy with Stage-Structured GRPO, using stagewise judge feedback for finer-grained credit assignment.
In parallel, we asynchronously train a shared-backbone reflection meta-policy (with meta-RL) to distill completed trajectories into reusable guidance, contributing through both shared-parameter updates and textual memory.
The resulting \ours-8B model achieves strong performance across four long-form research benchmarks,
outperforming existing open models of larger sizes and approaching proprietary deep-research systems.
Beyond final performance, we perform thorough controlled analyses to understand the key ingredients of \ours.
\end{abstract}

\maketitle

\section{Introduction}
\label{sec:intro}

Deep research agents answer complex information-seeking questions by autonomously planning, searching, evaluating evidence, and synthesizing long-form reports.
Yet learning this capability through end-to-end RL remains challenging: open-ended reports \textbf{lack exact answer verifiers}, and their quality depends on many interdependent decisions along \textbf{long} tool-augmented trajectories.

Existing deep-research training efforts primarily focus on curating and scaling high-quality imitation data~\citep{tongyi2025deepresearch} and learning from verifiable search queries~\citep{jin2025searchr1,song2025r1searcher,nguyen2025sfrdeepresearch}.
For long-horizon agentic RL, complementary lines of work explore finer-grained credit assignment~\citep{qian2025toolrl,xi2026agentprm} and meta-RL for learning from experience~\citep{jiang2026metarl,yang2026mage}.
These approaches are largely developed in \textbf{verifiable} or \textbf{synthetic} settings, where reliable outcome signals support policy learning.
In open-ended research, \textbf{neither intermediate progress nor final reports admit simple automatic verification}, making it difficult to assess individual decisions and the usefulness of past experience.

We tackle this challenge from two complementary angles: \textbf{long-horizon credit assignment}, which provides informative feedback across a trajectory, and \textbf{experience reuse}, which turns completed attempts into guidance for future research.
This leads to our central question: \emph{beyond verifiable rewards, how can agentic RL improve policy learning along both dimensions?}

To this end, we introduce \ours, a \textbf{unified RL framework for training open-ended deep research}.
To our knowledge, \ours is the \textbf{first framework} to demonstrate the feasibility and effectiveness of stagewise credit assignment and learned experience reuse in open-ended, long-horizon RL for deep research, both jointly and individually.
Our unifying design uses \textbf{rubrics as a shared interface} for policy execution, judge feedback, and experience reuse.
As illustrated in Fig.~\ref{fig:fig1}, a common structural prior supports two complementary RL mechanisms.

\ours first instills a \textbf{rubric-guided scaffold} through SFT, teaching the agent to \emph{plan with rubrics}.
The resulting scaffold organizes trajectories into Plan, Research, Review, and Answer with \emph{explicit semantic boundaries}.
During planning, the agent generates task-specific \emph{agent planning rubrics} to guide its subsequent research, review, and writing.
Together, these features establish a structural prior for RL: the stage boundaries define units for stagewise judging and credit assignment, while task trajectories provide structured experience for producing reflection rubrics and meta-RL training.
The scaffold also has its standalone benefits in improving policy execution and subsequent RL effectiveness.
Building on this prior, we introduce two complementary RL mechanisms:

\begin{figure}[t]
    \centering
    \includegraphics[width=\linewidth]{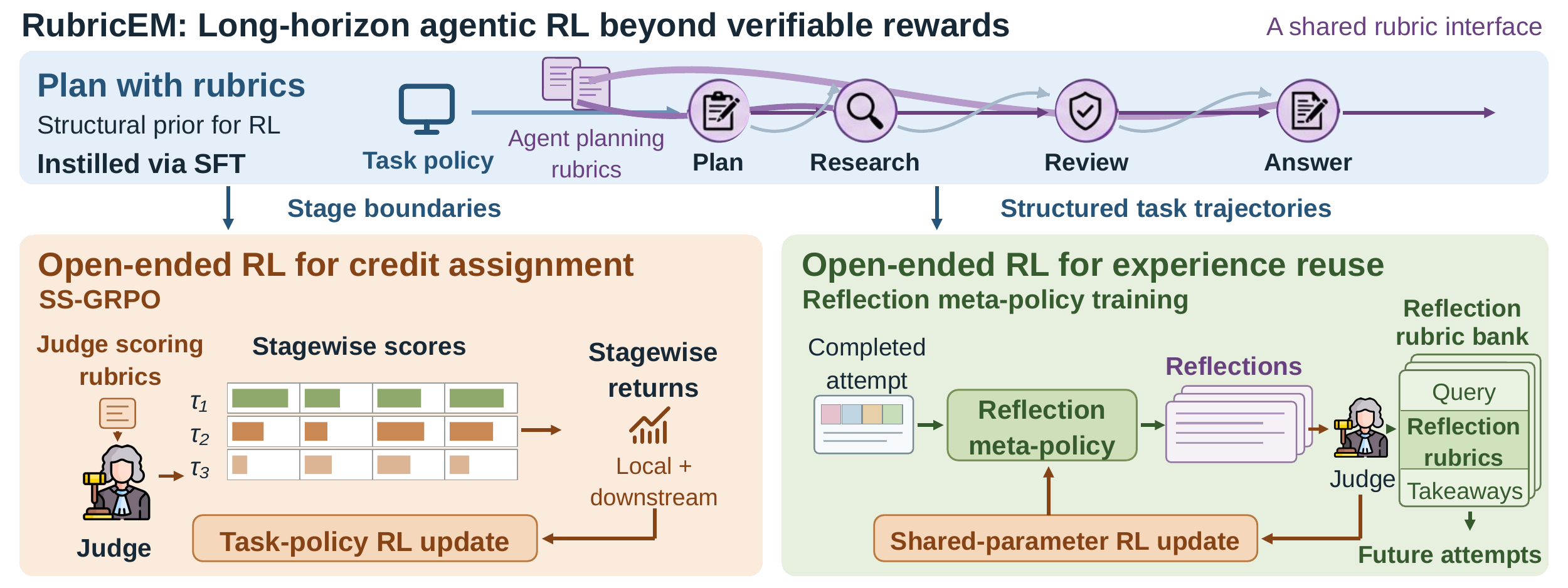}
    \caption{\textbf{\ours: long-horizon agentic RL beyond verifiable rewards.} SFT teaches the agent to \emph{plan with rubrics} for open-ended deep research. \textbf{Agent planning rubrics} guide execution within explicit semantic stages; \textbf{judge scoring rubrics} provide shared stagewise criteria for SS-GRPO. 
    The reflection meta-policy learns to distill these structured trajectories into reusable guidance via \textbf{agent reflection rubrics}.
    Training details appear in Fig.~\ref{fig:detailed_RL}.}
    \label{fig:fig1}
    \vspace{-1em}
\end{figure}

\begin{itemize}[leftmargin=*,itemsep=0.2em,parsep=0pt,topsep=0pt,partopsep=0pt]
    \item \textbf{Open-ended RL for credit assignment: SS-GRPO.}
    We extend evolving-rubric judging~\citep{shao2025dr} from final answers to semantic stages.
    The judge contrasts rollouts for the same query to construct its own shared, stage-specific \emph{judge scoring rubrics} and scores each trajectory against these criteria.
    Stage-Structured GRPO combines the resulting local semantic rewards with downstream outcomes into stagewise returns, assigning early decisions credit for the later work they enable.
    These stage-specific advantages guide policy optimization along the research trajectory.

    \item \textbf{Open-ended RL for experience reuse: reflection meta-policy training.}
    We jointly train a task policy and a reflection meta-policy on one shared backbone.
    The meta-policy distills completed trajectories into guidance and is trained with RL rewards for its judged usefulness.
    Reflection rewards update the shared parameters, while selected reflections enter a reflection rubric bank for within-episode refinement and cross-episode transfer.
    Experience thus contributes through both \emph{shared-parameter learning} and \emph{reusable textual memory}.
\end{itemize}

Controlled ablations demonstrate the benefits of the two RL mechanisms \textbf{separately and together}: both improve average performance over evolving-answer-only GRPO, and their combination performs best (Fig.~\ref{fig:ablation_study}).
Together, these components yield \ours-8B, an 8B agent trained for 1,400 RL steps.
Across four long-form research benchmarks, it achieves an average score of 55.5, the highest among the compared open models, and approaches proprietary deep-research systems.
Further analyses show that the rubric-guided scaffold has standalone benefits in terms of prompting, distillation, and RL.
The trained policy also benefits from inference-time experience reuse compared with untrained ones.
Additionally, despite all RL training being on long-form queries, it also transfers effectively to short-form search tasks.
Figure~\ref{fig:example} illustrates a concrete trajectory of a training step.

\begin{figure}[t]
    \centering
    \includegraphics[width=\linewidth]{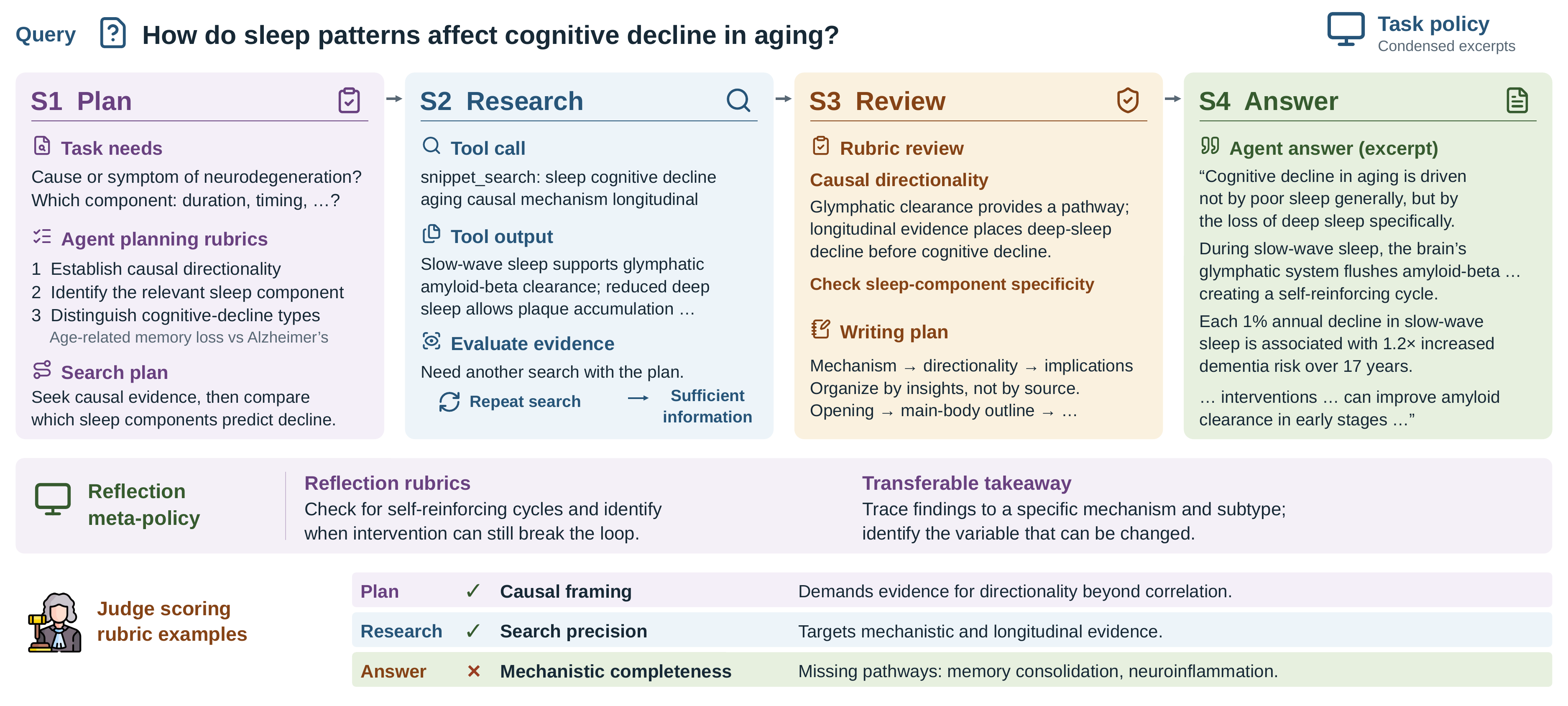}
    \caption{Condensed example of a rubric-guided stage-structured trajectory, meta-policy reflection, and stagewise judge scoring rubrics during a single RL step.}
    \label{fig:example}
\end{figure}

\section{Related Work}
\label{sec:main_related_work}

\textbf{Post-training recipes for deep research agents.}
Many existing open-source training efforts for deep research focus on short-answer tasks with verifiable rewards~\citep{jin2025searchr1,song2025r1searcher,chen2025research,jiang2025deepretrieval,zhao2025rsearch,han2025deep}.
Meanwhile, technical reports from proprietary systems mainly focus on scaling high-quality imitation data and training in verifiable short-form settings~\citep{openai2025deepresearch,google2025geminideepresearch,perplexity2025sonardeepresearch,moonshot2025kimiresearcher}.
Our work focuses on making reinforcement learning effective for \emph{open-ended, long-form} deep research.
The closest work to ours is DR Tulu~\citep{shao2025dr}, which studies end-to-end RL for deep research beyond verifiable rewards.
We build on this foundation with fine-grained credit assignment and a jointly trained reflection meta-policy for denser feedback and reusable guidance.

\textbf{Credit assignment and meta-RL with language models.}
Recent work on agentic reinforcement learning emphasizes finer-grained credit assignment~\citep{qian2025toolrl,xi2026agentprm,zhang2026reasoning}.
However, these methods are largely developed for verifiable settings, leaving open how to assign credit to intermediate decisions in open-ended research.
A related line of work trains meta-policies during reinforcement learning, often referred to as Meta-RL~\citep{jiang2026metarl,yang2026mage}.
These methods are typically evaluated on verifiable or synthetic tasks and introduce dependencies across rollouts that can \emph{add significant sequential training overhead}.
Our work targets open-ended real-world deep research, where \textbf{both intermediate progress and final reports require semantic judgments}.
We train the reflection meta-policy with judge-assessed usefulness rewards and overlap reflection updates with task rollouts through asynchronous execution, \textbf{avoiding significant sequential waiting during long-horizon RL}.
Appendix~\ref{sec:related_work} provides further discussion.

\section{\ours}
\phantomsection
\label{sec:method}

\subsection{Preliminaries and notations}

We study deep research agents for complex information-seeking queries.
Given a query $q \sim \mathcal{D}$, an agent interacts with a tool environment $\mathcal{T}$ and produces a trajectory
\(
\tau = (a_1, o_1, \dots, a_T, o_T),
\)
where $a_t$ denotes the agent emission at turn $t$, which may be either a textual segment or a structured tool call, and $o_t$ is the resulting tool output (with $o_t=\varnothing$ when no tool is invoked).
We consider a language-model-based agent that autoregressively samples the next step $a_t \sim \pi_\theta(a_t \mid h_t), h_t = (q, a_{<t}, o_{<t}),$
and eventually produces a final long-form answer
grounded in retrieved evidence.

\subsection{Structured Reasoning Scaffold}
\label{sec:scaffold}

A central design choice in \ours is to make stage structure explicit.
A \emph{stage} groups decisions with a shared role, such as planning, evidence gathering, review, or synthesis.
In long-horizon research, these stages provide a stable organization over interleaved reasoning and tool use.
Explicit boundaries expose the current decision role and provide comparable units for credit assignment~\citep{xu2025cognitive,feng2026environment}.
We analyze the value of stage information under a compressed-state abstraction, where the representation used for action selection may not fully preserve the current stage.
Let $h$ denote a random decision point along a trajectory induced by the current policy, let $c=\phi(h)$ denote a compressed state representation, let $z$ denote the current stage label, and let \(U(h,a)\) denote the expected downstream value of taking action \(a\) at history \(h\) and then continuing the rollout.

\begin{theorem}[Value of stage information]
\label{thm:stage-info-main}
Under Assump.~\ref{assump:stage-info-setup}, define
\[
V_{\mathrm{flat}}
:=
\mathbb E\!\left[
\max_{a \in \mathcal A} \mathbb E[U(h,a)\mid c]
\right],
\qquad
V_{\mathrm{stage}}
:=
\mathbb E\!\left[
\max_{a \in \mathcal A} \mathbb E[U(h,a)\mid c,z]
\right].
\]

If there exists a measurable set $\mathcal C_0$ with positive probability and two task-relevant stages such that for every $c \in \mathcal C_0$, $p(z\mid c)>0$,  $p(z'\mid c)>0$, and that $\arg\max_{a \in \mathcal A} \mathbb E[U(h,a)\mid c,z]
\;\cap\;
\arg\max_{a \in \mathcal A} \mathbb E[U(h,a)\mid c,z']
=
\varnothing.$
Then 
\[
V_{\mathrm{stage}} > V_{\mathrm{flat}}.
\]
\end{theorem}

\experimenttakeaway{Explicit stages can improve agentic execution by clarifying decision roles.}

In long-horizon agentic execution, the same compressed context may correspond to different decision roles. Theorem~\ref{thm:stage-info-main} shows that when these roles require incompatible optimal actions, conditioning on the stage label can resolve this ambiguity and improve decision value. See Appendix~\ref{app:theory:conditioning} for the proof and Sec.~\ref{subsec:scaffolding_inference} for empirical validation of our scaffold, which is outlined below:

\textbf{Specific stage instantiation.}
We instantiate this structure with four rubric-guided stages:
$\textsc{Plan} \rightarrow \textsc{Research} \rightarrow \textsc{Review} \rightarrow \textsc{Answer}.$ 
Each stage is marked by a stage-level XML tag with a lightweight internal schema (Fig.~\ref{fig:scaffold}).
The outer scaffold is sequential, while within \textsc{Research}, the policy chooses searches, revises its plan and agent planning rubrics, and decides when to enter \textsc{Review}.

\noindent\textbf{\textsc{Plan}.}
Within \texttt{<structured\_plan>}, the agent analyzes explicit and implicit needs in \texttt{<deep\_analysis>}, defines criteria in \texttt{<rubric>}, and proposes a \texttt{<research\_plan>}.
These agent planning rubrics include a knowledge checklist, analytical criteria, and negative constraints on the answer.

\noindent\textbf{\textsc{Research}.}
The agent iteratively issues \texttt{<call\_tool>} actions. 
After each tool response, the agent performs a \texttt{<state\_evaluation>} step, which compares the accumulated evidence against the plan and agent planning rubrics, decides if further search is needed, and optionally revises the \textsc{Plan} in place.

\noindent\textbf{\textsc{Review}.}
Within \texttt{<review>}, the agent maps evidence to criteria in \texttt{<rubric\_review>} and prepares a writing plan with a thesis and section outline.

\noindent\textbf{\textsc{Answer}.}
Within \texttt{<answer>}, the agent synthesizes a long-form response with citations.

\begin{figure}[t]
    \centering
    \includegraphics[width=\linewidth]{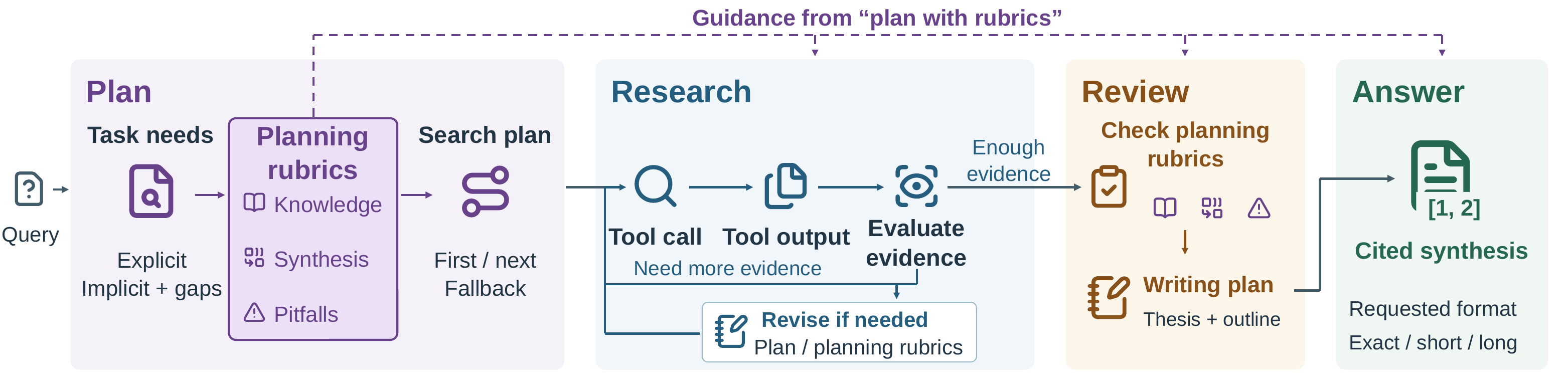}
    \caption{Rubric-guided reasoning scaffold in \ours. Agent planning rubrics defined in Plan guide evidence collection, review, and synthesis throughout the scaffold. Research iterates tool use and state evaluation; Review audits evidence before Answer. Appendix~\ref{app:scaffold_sub_section} details the schema.}
    \label{fig:scaffold}
    \vspace{-1.7em}
\end{figure}

This instantiation follows the task's progression from identifying information needs to collecting evidence, auditing it, and composing an answer.
\emph{Agent planning rubrics} carry quality criteria from planning into search and writing, giving the policy a common reference for execution and self-checking~\citep{panadero2017review,tai2018developing,bai2022constitutional}.
At the same time, agent planning rubrics from different rollouts of a query provide references for discovering shared judge scoring rubrics (Sec.~\ref{sec:ssgrpo}).

The scaffold enables both RL mechanisms: stage boundaries define SS-GRPO credit units, and task trajectories containing agent planning rubrics supply the experience distilled by the reflection meta-policy (Sec.~\ref{sec:evolving-policy}).
We therefore use it as an SFT-induced structural prior for subsequent RL~\citep{zhang2026good}.

\textbf{SFT distillation.}
We distill the scaffold from Gemini-3.1-Pro, prompting the teacher to follow the XML schema above.
Raw teacher traces do not always follow the target scaffold, so rejection sampling removes those that violate stage boundaries, tool-calling syntax, citation format, or grounding constraints.
The resulting SFT corpus teaches Qwen3-8B tool use, citation, stage discipline, and conditioning on agent planning rubrics for subsequent RL.
Appendix~\ref{app:sft-data} details the generation and filtering pipeline.

\subsection{Stage-Structured GRPO}
\label{sec:ssgrpo}

\begin{figure}[t]
    \centering
    \includegraphics[width=0.95\linewidth]{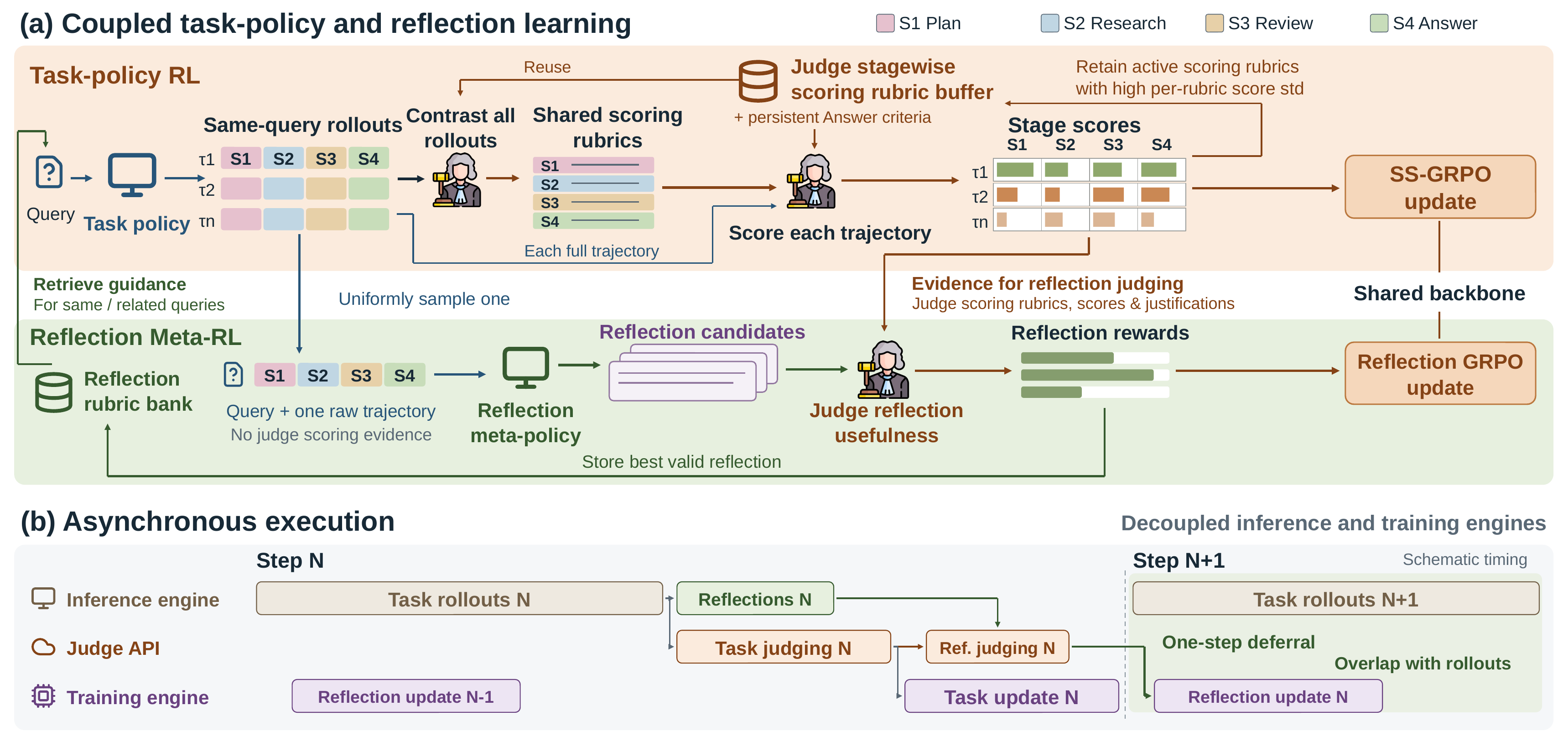}
    \caption{\textbf{Detailed RL training pipeline of \ours}.
In (a), contrasting same-query rollouts produces shared stagewise judge scoring rubrics.
Together with retained criteria, these rubrics provide SS-GRPO rewards and update the judge buffer.
The reflection policy receives the query and one raw trajectory; only its judge receives privileged judge scoring rubrics, scores, and justifications.
All candidate rewards train the shared backbone, while the highest-scored valid reflection enters the reflection rubric bank.
In (b), decoupled engines overlap a deferred reflection update from step $N$ with task rollouts at step $N{+}1$.
Details are provided in Appen.~\ref{app:stagewise_judge} and~\ref{app:async-implementation}; Alg.~\ref{alg:rubricem} gives the formal procedure.}
    \label{fig:detailed_RL}
\end{figure}

Building on the scaffold above, we propose \textbf{Stage-Structured GRPO (SS-GRPO)} for finer-grained credit assignment in deep research.
Prior work motivates denser process supervision for long-horizon agent RL~\citep{yang2026patching,tan2026hindsight,qian2025toolrl,wu2026demystifying,wang2026subgoal}.
In open-ended research, however, the quality of planning, search, review, and synthesis is semantic, task-dependent, and difficult to verify automatically.
SS-GRPO combines stage boundaries from Sec.~\ref{sec:scaffold} with judge scoring rubrics to construct stage-level learning signals (Fig.~\ref{fig:detailed_RL}).

\textbf{Stagewise scores and returns.}
Given a query \(q\), we sample \(n\) rollouts
\(\{\tau_i\}_{i=1}^n \sim \pi_{\theta}(\cdot \mid q)\) and partition each into \(K\) semantic stages; in our instantiation, \(K=4\) corresponding to Plan, Research, Review, and Answer.
Let \(\mathcal B_{i,k}\) denote the policy-generated tokens in stage \(k\) of rollout \(\tau_i\), and let \(R_{i,k}\in[0,1]\) denote the aggregated LLM-judge score under the corresponding stage-specific judge scoring rubrics.
A \emph{stage return} combines local assessment with downstream outcomes, giving early decisions credit for the later work they enable.
For a nonnegative causal stage-dependence matrix \(\Lambda=(\lambda_{k,j})\), with \(\lambda_{k,j}=0\) for \(j<k\) and $\lambda_{k,k}=1$, define
\(
G^{\Lambda}_{i,k}=\sum_{j=k}^{K}\lambda_{k,j}R_{i,j}.
\) 
Thus, each stage keeps its own score while receiving credit from the downstream stage it enables. 
Terminal reward broadcast can be considered as a special case with \(\lambda_{k,j}=\mathbf{1}\{j=K\}\).


\textbf{When stage returns help.}
Intermediate judging recovers process information omitted by terminal-only rewards, but might also introduce judge error.
For unnormalized score-function signals, Thm.~\ref{thm:stage-credit-main} in Appendix~\ref{app:theory:credit} shows that stage-weighted credit is closer to the specified oracle process target than terminal broadcast when the omitted signal outweighs cumulative judge misalignment.
Previous work~\citep{shao2025dr} shows that evolving judge scoring rubrics empirically improve criterion quality and capture failures missed by static criteria, where we adopt and extend their use to improve judge alignment.

\textbf{Judge scoring rubric buffer.}
As shown in Fig.~\ref{fig:detailed_RL} (top), the judge builds shared criteria by contrasting rollouts for the same query.
We note \textbf{agent planning rubrics differ from judge scoring rubrics and never serve as a trajectory's scoring standard}.
An evolving stagewise buffer incorporates newly discovered criteria, reuses discriminative items, and prunes those that no longer distinguish trajectory quality.
Persistent scoring rubrics from the training data remain for Answer. The full details are in Appendix~\ref{app:stagewise_judge}.


\textbf{Stagewise normalization and objective.}
We instantiate SS-GRPO as a critic-free stagewise variant of GRPO by normalizing returns separately within each stage across the rollout group:
\(
A_{i,k}
=
\frac{
G_{i,k}^{\Lambda}
-
\frac{1}{n}\sum_{i'=1}^{n}G_{i',k}^{\Lambda}
}{
\operatorname{Std}_{i'}[G_{i',k}^{\Lambda}]+\epsilon
}.
\)
All generated tokens in the same stage block \(\mathcal B_{i,k}\) share the advantage \(A_{i,k}\); tool-output tokens are masked.
The resulting objective is
\begin{equation}
\mathcal L_{\mathrm{SS\text{-}GRPO}}
=
-\frac{1}{n}\sum_{i=1}^{n}\sum_{k=1}^{K}\sum_{t\in\mathcal B_{i,k}}
\min\!\Big(
\rho_{i,t}A_{i,k},
\operatorname{clip}(\rho_{i,t},1-\eta,1+\eta)A_{i,k}
\Big)
+
\beta D_{\mathrm{KL}}(\pi_\theta\,\|\,\pi_{\mathrm{ref}}), 
\label{eqa:ss-grpo}
\end{equation}

where
\(
\rho_{i,t}
=
\frac{\pi_\theta(a_{i,t}\mid h_{i,t})}
{\pi_{\theta_{\mathrm{old}}}(a_{i,t}\mid h_{i,t})}.
\)
We use group-relative advantages to keep the estimator critic-free, avoiding an additional stage-conditioned value model as judge-defined supervision evolves over expensive long-horizon rollouts.

\subsection{Meta-Policy Training with Reinforcement Learning}
\label{sec:evolving-policy}

Beyond single-trajectory optimization, \ours makes experience reuse part of RL.
A shared backbone serves as both the task policy and a reflection meta-policy: task rollouts provide experience, and the reflection policy is trained with LLM-judge rewards to produce reusable natural-language guidance.
Selected reflections enter a \emph{reflection rubric bank} and condition future attempts.
The meta-policy supports adaptation through textual context.

\textbf{Joint training of the reflection meta-policy.}
From each query's completed rollouts, we uniformly sample one trajectory and generate multiple reflection candidates while task judging proceeds.
The generator receives only the query and raw trajectory as fixed context, with the loss applied to reflection tokens.
Using privileged judge scoring rubrics, scores, and justifications, the reflection judge estimates each candidate's usefulness for future attempts.
All candidate scores provide GRPO signals to the shared backbone; only the highest-scored valid candidate enters the bank.
Because both policies share parameters, the reflection objective supplies an auxiliary RL update to the task policy, while retrieval makes the learned guidance available as context.
Appendix~\ref{app:theory:evolution} formalizes this coupling through a judge-gated surrogate under local positive gradient alignment, with a fixed bank.

\textbf{Coupled agent--judge co-evolution.}
On-policy rollouts expose new criteria and failure modes, updating the judge-side buffer from Sec.~\ref{sec:ssgrpo}.
The judge then scores task trajectories and reflection candidates, and accepted reflections condition future rollouts through the bank.
The agent updates its parameters and memory, while the judge updates its scoring rubric buffer.

\textbf{Reflection rubric bank and two modes of adaptation.}
Each bank item distills a completed trajectory into reflection rubrics and takeaways, complementing the agent planning rubrics that guided the attempt (Fig.~\ref{fig:example}).
They support two adaptation modes: \emph{within-episode refinement}, which retrieves a query's own prior reflection on a repeated attempt, and \emph{cross-episode transfer}, which retrieves reflections from related past questions. 
During training, a two-encounter curriculum first solves each query with cross-episode retrieval and later replays it with its own reflection.
Windowed scheduling leaves time for reflection generation and bank insertion before the repeated attempt.

\subsection{Practical Implementation}
\label{sec:implementation}

\textbf{Efficient judging and data collection.} 
Although \ours requires stagewise rubric judging, we implement it efficiently. 
At each step, we use one structured JSON call per query to generate judge scoring rubrics for all stages and one per trajectory to score all stages (Appendix~\ref{app:stagewise_judge}), requiring \emph{substantially fewer task-judging calls} than DR Tulu's answer-only, per-rubric scoring.
Empirically, stagewise scores provide stable assessments of intermediate process quality (Appendix~\ref{app:answer-masking}).
We also use Batch API generation for SFT data collection to reduce teacher costs. 
As a result, the estimated teacher and judge API costs are 58\% and 60--70\% lower than DR Tulu's, respectively (Appendix~\ref{app:resource-accounting}).

\textbf{Efficient asynchronous execution.} We defer reflection meta-policy updates by one RL step to overlap them with task rollouts (Fig.~\ref{fig:detailed_RL}, bottom). At step \(N\), the inference engine generates tool-augmented task rollouts while the training engine updates the meta-policy on the completed reflection batch from step \(N-1\). Reflection generation and judging for step \(N\) run in the background, preparing the batch for step \(N+1\). This schedule improves engine utilization and avoids sequential waiting in the SS-GRPO loop, a notorious problem in prior meta-RL literature~\citep{jiang2026metarl} (Appendix~\ref{app:async-implementation}).

\section{Experiment}
\label{sec:exp}

\subsection{Experimental Setup}
\label{subsec:exp_setup}

We evaluate \ours on four representative long-form benchmarks:
\textbf{HealthBench}~\citep{arora2025healthbench},
\textbf{ResearchQA}~\citep{yifei2025researchqa},
\textbf{DeepResearchBench (DRB)}~\citep{du2025deepresearch},
and \textbf{ResearchRubrics}~\citep{sharma2026researchrubrics}.
Our experimental setup builds on DR Tulu, broadly following its infrastructure, training-data sources, core hyperparameters, and evaluation protocols.
We use Gemini-flash-grounded Google Search and Semantic Scholar as the search engines. We use Gemini-3-flash as the judge. 
The training data is sampled from diverse public query sources, including realistic search conversations from SearchArena~\citep{miroyan2025search} and research-oriented questions from OpenScholar~\citep{asai2024openscholar}.
The SFT stage includes both short-form and long-form data, while the RL stage exclusively focuses on long-form queries.
We follow benchmark-specific grading protocols developed independently of our training and supported by expert validation and published human-agreement studies. 
For fair comparison, main results and RL ablations use zero-shot prompts without retrieved reflections or memory entries.
Further details on datasets, training and inference setups, evaluation protocols, and the full baseline list are provided in Appen.~\ref{appen:exp}.

\begin{table}[t!]
\centering
\caption{System-level long-form deep research performance. Bold indicates the best reported score within proprietary and non-proprietary categories. 
}
\label{tab:model_performance}

\begingroup
\setlength{\tabcolsep}{4.0pt}
\renewcommand{\arraystretch}{1.00}
\arrayrulecolor{tblRule}

\resizebox{0.85\linewidth}{!}{
\begin{tabular}{l c c c c c}
\toprule
\rowcolor{tblHeader}
\textbf{Model} & \textbf{HealthBench} & \textbf{ResearchQA} & \textbf{DRB} & \textbf{ResearchRubrics} & \textbf{Average} \\
\midrule

\tblsection{tblClosedHead}{Closed Deep Research}
\rowcolor{tblClosedRow}
Claude-Sonnet Search
& \tblNA & 64.3 & 34.5 & \tblNA & \tblNA \\
\rowcolor{tblClosedRow}
Perplexity-Sonar (High)
& \tblNA & 69.1 & 40.7 & \tblNA & \tblNA \\
\rowcolor{tblClosedRow}
Perplexity Deep Research
& \tblNA & 75.3 & 42.3 & 48.7 & \tblNA \\
\rowcolor{tblClosedRow}
Gemini Deep Research
& \tblNA & 68.5 & 48.8 & \tblbest{61.5} & \tblNA \\
\rowcolor{tblClosedRow}
Gemini 3.1 Pro + Search
& 47.5 & 74.5 & 44.4 & 49.1 & 53.9 \\
\rowcolor{tblClosedRow}
GPT-5 + Search
& \tblbest{59.5} & 78.2 & \tblbest{50.7} & 60.5 & \tblbest{62.2} \\
\rowcolor{tblClosedRow}
OpenAI Deep Research
& 53.8 & \tblbest{79.2} & 46.9 & 59.7 & 59.9 \\

\midrule
\tblsection{tblFixedHead}{Fixed Pipeline Deep Research}
\rowcolor{tblFixedRow}
WebThinker QwQ-32B
& 36.5 & 72.8 & 37.9 & 42.2 & 47.4 \\
\rowcolor{tblFixedRow}
WebThinker-32B-DPO
& 39.4 & 74.2 & 40.6 & 41.9 & 49.0 \\
\rowcolor{tblFixedRow}
Ai2 ScholarQA -- Claude Sonnet
& 32.0 & 75.0 & 36.1 & 38.1 & 45.3 \\

\midrule
\tblsection{tblOpenHead}{Open Deep Research Models}
\rowcolor{tblOpenRow}
Search-R1-7B
& -0.1 & 27.9 & 9.5 & 0.0 & 9.3 \\
\rowcolor{tblOpenRow}
WebExplorer-8B
& 33.7 & 64.8 & 36.7 & 33.4 & 42.2 \\
\rowcolor{tblOpenRow}
Tongyi DeepResearch-30B-A3B
& 46.2 & 66.7 & 40.6 & 49.5 & 50.8 \\
\rowcolor{tblOpenRow}
DR Tulu-8B (SFT)
& 38.1 & 68.5 & 39.0 & 38.4 & 46.0 \\
\rowcolor{tblOpenRow}
DR Tulu-8B (RL, 1900 steps)
& \tblbest{50.2} & 74.3 & 43.4 & 46.4 & 53.6 \\

\midrule
\tblsection{tblOursHead}{Ours}
\rowcolor{tblOursRow}
Qwen3-8B + Our Search
& 24.5 & 58.4 & 28.2 & 24.5 & 33.9 \\
\rowcolor{tblOursRow}
\ours{}-8B (SFT)
& 39.0 & 71.8 & 43.0 & 42.8 & 49.2 \\
\rowcolor{tblOursHi}
\ours{}-8B (RL, 1400 steps)
& 49.3 & \tblbest{74.5} & \tblbest{47.8} & \tblbest{50.3} & \tblbest{55.5} \\

\bottomrule
\end{tabular}%
}

\arrayrulecolor{black}
\endgroup
\end{table}

\subsection{Main Results}

Tab.~\ref{tab:model_performance} compares \ours with the strongest baselines from each category following~\citet{shao2025dr}.
We use reported numbers when available, and reproduce the remaining baselines if possible.

\textbf{\ours achieves strong long-form research performance.}
\model-RL achieves the highest average score among non-proprietary deep research systems in our evaluation, reaching 55.5 with an 8B backbone. It surpasses strong open baselines, including DR Tulu-8B-RL, Tongyi DeepResearch-30B-A3B, and WebThinker-32B-DPO. It also compares favorably with proprietary systems: on the benchmarks where both scores are available, \model-RL outperforms Perplexity Deep Research on average, and it remains within 4.4 average points of OpenAI Deep Research while outperforming it on DRB. These results demonstrate competitive long-form research performance with an 8B backbone.

\textbf{RL is effective across four long-form benchmarks.}
Starting from the structured SFT checkpoint, RL improves the average score from 49.2 to 55.5, with gains on all four long-form benchmarks. Although \model-SFT is distilled from Gemini-3.1-Pro, the final \model-RL model surpasses its teacher on average. To understand the contributions of our RL training recipe, Sec.~\ref{subsec:rl_recipes} compares its components under \textbf{the same experimental setup}, showing complementary gains from SS-GRPO and reflection training, with the full recipe performing best. Our results in Tab.~\ref{tab:model_performance} and these ablations use no test-time reflection retrieval, supporting reflection training as a training-time augmentation without introducing any additional inference overheads. 


\section{Empirical Analysis}
\label{sec:empirical_analysis}

\begin{figure}[t]
    \centering
    \includegraphics[width=\linewidth]{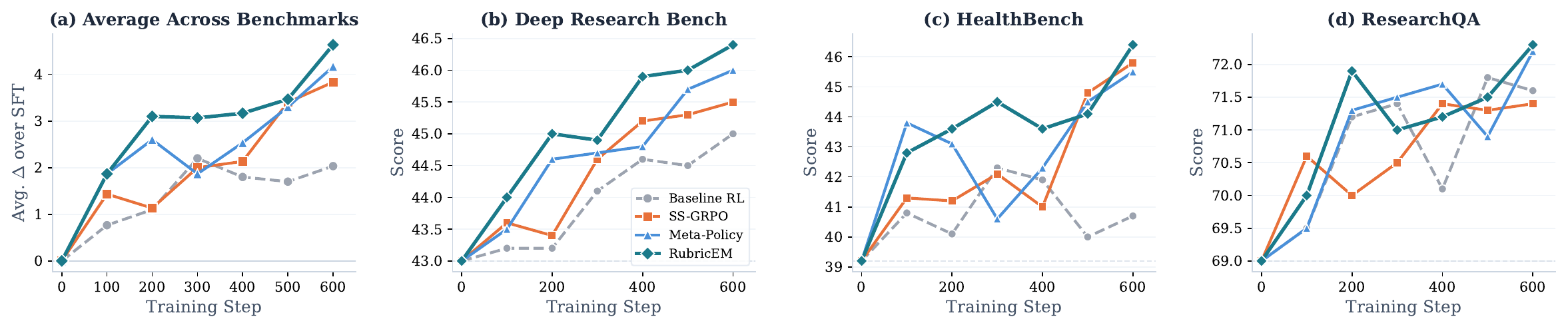}
    \caption{
    Controlled RL-recipe ablations with the same \ours-SFT checkpoint, search environment, and judge model over 600 steps. Baseline-RL applies DR Tulu's evolving-answer GRPO objective; SS-GRPO and reflection training yield complementary gains. 
    }
\label{fig:ablation_study}
\end{figure}

\subsection{RL Training Recipes}
\label{subsec:rl_recipes}

We follow the experimental setup in Sec.~\ref{subsec:exp_setup} and compare four RL recipes.
\textbf{All four runs begin from the same strong \ours-SFT checkpoint}, use the same search backend and environment configurations, and share a 600-step task-rollout budget.\footnote{We use 600 steps to observe differences in learning trends while keeping costs manageable.}
Each run requires approximately 8 days on 16 H100 GPUs.
Every checkpoint is evaluated on the same fixed random set of 100 questions per benchmark\footnote{Evaluating all five ablation runs at six checkpoints on the full HealthBench set would cost more than \$3,000.} (Appendix~\ref{app:ablation-eval}).
The four configurations are:

\begin{center}
\resizebox{0.65\linewidth}{!}{
\begin{tabular}{lcc}
\toprule
\textbf{RL recipe} & \textbf{Evolving judge rubrics} & \textbf{Reflection training} \\
\midrule
Baseline-RL~\cite{shao2025dr} & Answer-only & No \\
SS-GRPO & Stagewise & No \\
Meta-Policy & Answer-only & Yes \\
\ours{} (Full) & Stagewise & Yes \\
\bottomrule
\end{tabular}
}
\end{center}
In other words, Baseline-RL is the \textbf{DR Tulu's evolving-answer GRPO objective instantiated in our settings}. SS-GRPO adds stagewise rewards and stage-dependent credit assignment. Meta-Policy adds reflection training and training-time retrieval from the reflection rubric bank to Baseline-RL.

Figure~\ref{fig:ablation_study} shows that the SS-GRPO and Meta-Policy configurations both independently improve average performance over Baseline-RL under the matched training setup, while the full recipe gives the highest average improvement and the best final scores on the plotted benchmarks. Reflection training likely improves the quality of training rollouts, leading to a stronger task policy even without requiring additional context at test time, similar to our observation in Table~\ref{tab:model_performance}.

\experimenttakeaway{Both SS-GRPO and Meta-RL training improve over DR Tulu (Baseline-RL). 
}

\subsection{Structured Scaffolding and Inference-Time Experience Reuse}
\label{subsec:scaffolding_inference}

We first evaluate the standalone benefits of our proposed reasoning scaffold (Sec.~\ref{sec:scaffold}) for agent execution in terms of \textbf{SFT distillation}, its preparation for \textbf{subsequent RL}, and \textbf{direct prompting}.
First, we compare checkpoints distilled from the same teacher using our scaffold versus an unstructured ReAct-style format.
We then train both checkpoints for 600 RL steps under matched settings using DR Tulu's evolving answer-only GRPO objective.
To assess the scaffold's effect through prompting alone, we compare the two formats on Gemini-3.1-Pro with the same search backend.
Finally, we compare the Baseline-RL and \ours-trained checkpoints from Sec.~\ref{subsec:rl_recipes} with and without experience reuse.
Cross-episode reuse at trial~1 and within-episode refinement at trial~2 form a sequential test-time scaling setup.
Experimental settings are detailed in Appendix~\ref{app:evaluation-details}, with scaffold and retrieval mechanisms described in Appendices~\ref{appen:scaffold} and~\ref{app:bank-retrieval}.

Figure~\ref{fig:ablation_scaffolding_scaling} summarizes the results. Panels (a,b) show that our scaffold improves SFT performance and yields larger subsequent RL gains relative to each checkpoint's SFT baseline. Panel (c) shows that it also improves Gemini-3.1-Pro's DRB performance through prompting alone. In panel (d), both cross-episode and within-episode reuse improve the \ours-trained policy, whereas Baseline-RL, trained without reflection meta-RL, does not improve with retrieval. These results suggest that reflection meta-RL improves the policy's ability to turn retrieved experience into additional test-time gains.

\experimenttakeaway{Our reasoning scaffold provides standalone benefits for agentic execution, while meta-RL induces better test-time scaling behavior.}

\begin{figure}[t]
    \centering
    \includegraphics[width=\linewidth]{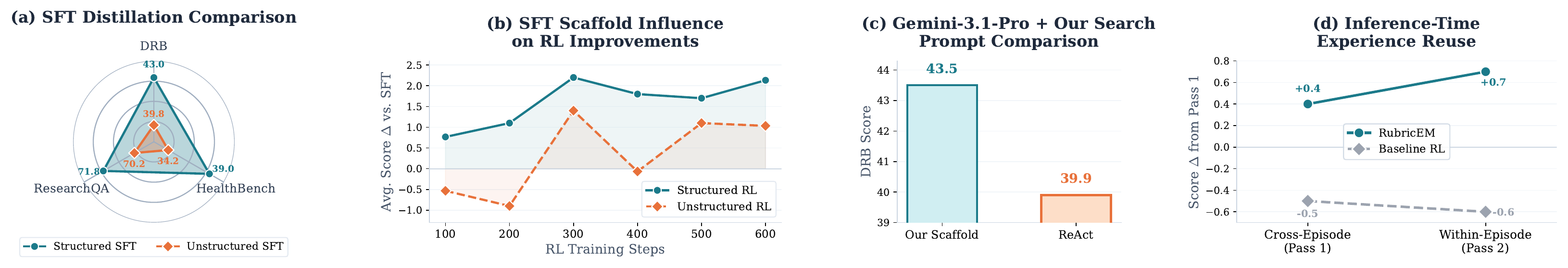}
    \caption{
Scaffolding and experience reuse. (a,b) The scaffold improves SFT and Baseline-RL (DR Tulu's evolving answer-only rubric objective). (c) It improves DRB over ReAct with the same Gemini-3.1-Pro model and search backend. (d) Experience reuse improves the \ours-trained policy but does not improve Baseline-RL.
}
\label{fig:ablation_scaffolding_scaling}
\end{figure}

\subsection{Short-Form Benchmark Performance}

We additionally evaluate short-form transfer on SimpleQA~\citep{wei2024measuring}, 2Wiki~\citep{ho2020constructing}, WebWalker~\citep{wu2025webwalker}, and DeepSearchQA (DSQA)~\citep{gupta2026deepsearchqa}. These benchmarks cover factual answering and multi-hop reasoning questions. SFT includes short-form data, while online RL uses only long-form queries, so this comparison tests \textbf{transfer beyond the RL query distribution}.
Table~\ref{tab:model_performance_2} shows that \ours-SFT already provides strong short-form performance, averaging 67.8. Long-form RL improves all four benchmarks and raises the average to 73.5. In particular, RL shows larger gains on the more challenging ones, WebWalker and DSQA.

\experimenttakeaway{RL on long-form deep research also improves short-form search performance.}

\begin{table}[t]
\centering
\caption{Short-form search performance after SFT and long-form RL.}
\label{tab:model_performance_2}

\begingroup
\setlength{\tabcolsep}{4.0pt}
\renewcommand{\arraystretch}{1.00}
\arrayrulecolor{tblRule}

\resizebox{0.6\linewidth}{!}{%
\begin{tabular}{l c c c c c}
\toprule
\rowcolor{tblHeader}
\textbf{Model} & \textbf{SimpleQA} & \textbf{2Wiki} & \textbf{WebWalker} & \textbf{DSQA} & \textbf{Avg.} \\
\midrule

\tblsection{tblOpenHead}{DR Tulu Baseline}
\rowcolor{tblOpenRow}
DR Tulu-8B (SFT)
& 75.5 & 66.5 & 31.9 & 5.3 & 44.8 \\
\rowcolor{tblOpenRow}
DR Tulu-8B (RL, 1900 steps)
& 80.1 & 68.0 & 39.1 & 8.3 & 49.0 \\

\midrule
\tblsection{tblOursHead}{Ours}
\rowcolor{tblOursRow}
Qwen3-8B + Our Search
& 84.0 & 61.5 & 42.6 & 15.2 & 50.8 \\
\rowcolor{tblOursRow}
\ours{}-8B (SFT)
& 92.1 & 77.5 & 64.7 & 37.0 & 67.8 \\
\rowcolor{tblOursHi}
\ours{}-8B (RL, 1400 steps)
& \tblbest{92.3} & \tblbest{78.8} & \tblbest{70.0} & \tblbest{53.0} & \tblbest{73.5} \\

\bottomrule
\end{tabular}%
}

\arrayrulecolor{black}
\endgroup
\end{table}


\section{Conclusions}
\phantomsection
\label{sec:conclusion}

We presented \ours, a rubric-guided RL framework for open-ended deep research.
Its structured scaffold exposes semantic stages for SS-GRPO credit assignment, while shared-backbone reflection training turns completed attempts into reusable guidance.
Rubrics connect policy execution, judge feedback, and experience reuse.
The resulting 8B policy achieves strong long-form performance and improves short-form search after RL.
Thorough controlled ablations showcase the effectiveness of our proposed recipes.
Appendix~\ref{appen:limitation} discusses limitations and future directions.

\clearpage
\bibliographystyle{plainnat}
\bibliography{ref}

\clearpage
\appendix
\section*{Appendix Contents and Reading Guide}

The guide below links each appendix to the method or experiment it expands.
Appendix titles, subsection numbers, and page numbers are clickable.

\begingroup
\small
\setlength{\tabcolsep}{3pt}
\renewcommand{\arraystretch}{1.12}
\newcommand{\appguidepage}[1]{\textbf{\ref{#1}}\newline p.~\pageref{#1}}
\newcommand{\appguidetitle}[2]{\textbf{\hyperref[#1]{#2}}\newline}
\noindent
\begin{tabular}{@{}p{0.10\linewidth}p{0.64\linewidth}p{0.22\linewidth}@{}}
\toprule
\textbf{Appendix} & \textbf{Contents and purpose} & \textbf{Main text} \\
\midrule
\appguidepage{sec:related_work} &
\appguidetitle{sec:related_work}{Additional Related Works}
Expanded comparisons with deep-research training, credit assignment, and language-agent meta-RL.
& Sec.~\ref{sec:main_related_work} \\
\addlinespace[5pt]

\appguidepage{appen:scaffold} &
\appguidetitle{appen:scaffold}{Details on Structured Scaffolds}
\ref{app:scaffold_sub_section}: stage schema and full system prompt.
\ref{app:sft-data}: teacher prompts, SFT data generation, and filtering.
& Secs.~\ref{sec:scaffold}, \ref{subsec:scaffolding_inference};
Figs.~\ref{fig:scaffold}, \ref{fig:ablation_scaffolding_scaling}(a--c) \\
\addlinespace[5pt]

\appguidepage{app:stagewise_judge} &
\appguidetitle{app:stagewise_judge}{Details on Stagewise Evolving Rubric Evaluation}
\hyperref[app:rubric-roles]{Rubric roles and information flow}.
\ref{app:judge-buffer}: buffer design.
\ref{app:rubric-generation}: criterion generation.
\ref{app:stagewise-scoring}: scoring and \hyperref[app:answer-masking]{answer-masking stability}.
\ref{app:judge-implementation}: pruning, structured calls, and concurrency.
& Secs.~\ref{sec:ssgrpo}, \ref{sec:implementation};
Fig.~\ref{fig:detailed_RL}(a) \\
\addlinespace[5pt]

\appguidepage{app:async-implementation} &
\appguidetitle{app:async-implementation}{Asynchronous Reflection Pipeline and Windowed Curriculum}
\ref{app:pipeline-architecture}: engines and threads.
\ref{app:deferred-training}: deferred updates.
\ref{app:windowed-curriculum}: query scheduling.
\ref{app:reflection-candidates}: reflection sampling.
\ref{app:bank-retrieval}: bank storage and retrieval.
\ref{app:rubric-bank-prompts}: generation, judging, and injection prompts.
& Secs.~\ref{sec:evolving-policy}, \ref{sec:implementation}, \ref{subsec:scaffolding_inference};
Fig.~\ref{fig:detailed_RL} \\
\addlinespace[5pt]

\appguidepage{app:theory} &
\appguidetitle{app:theory}{Theoretical Analysis}
\ref{app:theory:conditioning}: stage-information assumptions and proof of Theorem~\ref{thm:stage-info-main}.
\ref{app:theory:credit}: judge-aligned stagewise credit.
\ref{app:theory:evolution}: shared-backbone task/reflection coupling with a fixed bank.
& Secs.~\ref{sec:scaffold}--\ref{sec:evolving-policy} \\
\addlinespace[5pt]

\appguidepage{appen:exp} &
\appguidetitle{appen:exp}{Experiment Details}
\ref{app:training_details}: training settings and the stage-credit matrix $\Lambda$.
\ref{app:resource-accounting}: cost accounting.
\ref{app:evaluation-details}: ablation experiment setup.
\ref{app:infrastructure}: infrastructure and search tools.
\ref{app:baselines}: baseline systems.
& Secs.~\ref{sec:exp}--\ref{sec:empirical_analysis};
Sec.~\ref{sec:implementation} for costs \\
\addlinespace[5pt]

\appguidepage{app:algorithm} &
\appguidetitle{app:algorithm}{Algorithm}
Algorithm summary of \ours.
& Sec.~\ref{sec:method};
Fig.~\ref{fig:detailed_RL} \\
\addlinespace[5pt]

\appguidepage{appen:limitation} &
\appguidetitle{appen:limitation}{Limitations and Discussions}
Infrastructure and judge limitations; broader implications and future directions.
& \hyperref[sec:conclusion]{Conclusions} \\
\bottomrule
\end{tabular}
\endgroup

\clearpage

\section{Additional Related Works}
\label{sec:related_work}

\textbf{The Post-training recipes of deep research agents.}
Deep research agents are increasingly framed as post-trained, tool-interacting policies rather than merely prompt-engineered retrieval pipelines~\citep{zhang2025agenticdeepresearch, li2025rlfoundations, shi2025deepsurvey}. 
Closed-source systems demonstrate strong long-horizon research capabilities, but their training recipes remain proprietary or only described at a high level~\citep{openai2025deepresearch, google2025geminideepresearch, perplexity2025sonardeepresearch, moonshot2025kimiresearcher}. 
Most open training work instead studies search-augmented reasoning as a proxy for deep research, using short-answer or otherwise verifiable tasks whose rewards come from answer matching, retrieval quality, or rule-based process signals~\citep{jin2025searchr1, song2025r1searcher, chen2025research, jiang2025deepretrieval, zhao2025rsearch, wang2025stepsearch, fan2025ssrl, zhang2025evolvesearch, mei2025o2searcher}. 
These methods teach models when and how to retrieve, but leave open how RL should be structured for unverifiable long-form reports.

Recent work has begun to move toward long-horizon deep research agents, focusing on real-web or scalable search training~\citep{zheng2025deepresearcher, wu2025webdancer, li2025websailor, gao2025asearcher}, workflow and report-state design~\citep{li2025webthinker, qiao2025webresearcher,han2025deep}, and large-scale agentic SFT/RL or RLAIF recipes built largely around verifiable or ground-truth-anchored answer supervision~\citep{liu2025webexplorer, nguyen2025sfrdeepresearch, tongyi2025deepresearch, wan2025pokee}. 
These works mainly improve the data, environment, workflow, or scale of deep research agents. 
Our focus is orthogonal in emphasis: rather than scaling agentic pipelines under verifiable answer supervision, we study the RL algorithmic structure needed for open-ended long-form research trajectories.

The closest work to ours is DR Tulu, which introduces RLER and establishes the first open end-to-end recipe for training long-form deep research agents beyond verifiable rewards~\citep{shao2025dr}. 
We build on this foundation to improve semantic credit assignment and learned experience reuse in this setting.
While DR Tulu shows that extended RL training is important for long-form deep research, \ours{} combines three roles: rubric-guided stages structure trajectories, stage-level credit assignment provides denser learning signals, and a jointly trained reflection meta-policy turns judged experience into reusable guidance for policy evolution.

\paragraph{Credit assignment in LLM agentic reinforcement learning.}
Credit assignment is a central challenge in agentic RL because sparse outcome rewards make it unclear which intermediate planning, tool-use, or reasoning decisions should be reinforced. 
Prior work addresses this issue through process supervision for mathematical reasoning~\citep{lightman2023let, wang2023math}, tool- or turn-level reward shaping for interactive agents~\citep{qian2025toolrl, li2025encouraging}, and more general agentic credit-assignment mechanisms such as process reward models, implicit step rewards, trajectory-graph advantage assignment, and hierarchical transition decomposition~\citep{xi2026agentprm, liu2026agentic, li2026salt, luo2025agent, peng2026hiper}. 
These methods show that dense or process-aware supervision is crucial for long-horizon agent training, but they are typically designed for domains with verifiable intermediate correctness, structured environment states, tool-specific signals, or learned critics. 
\ours{} studies a different regime: open-ended long-form deep research beyond verifiable rewards, where intermediate progress is semantic rather than objectively checkable. 
We therefore assign credit at the level of rubric-guided stages, using stage-specific judge scores and causal stage-dependent returns to provide denser supervision without requiring step labels or a learned critic.


\paragraph{Meta-RL with language models.}
Meta-reinforcement learning studies agents that use past experience to adapt future behavior, with classic formulations including recurrent fast-adaptation policies such as RL$^2$ and Learning to Reinforcement Learn, gradient-based adaptation such as MAML, and latent-context inference methods such as PEARL~\citep{duan2016rl,wang2016learning,finn2017model,rakelly2019efficient}. 
For LLM agents, this perspective is natural because trajectories contain rich textual artifacts (plans, tool calls, observations, critiques, and failures) that can be distilled into reusable guidance. 
Recent language-agent Meta-RL methods train meta-policies to improve exploration and exploration--exploitation across repeated tasks~\citep{jiang2026metarl,yang2026mage}, while MetaClaw studies continual agent evolution through failure-driven skill synthesis and a reusable skill library~\citep{xia2026metaclaw}. 
However, these works are typically evaluated in verifiable, synthetic, or task-completion settings, and often rely on explicit support--query or cross-rollout dependencies that are costly for long-horizon web research. 
\ours{} targets open-ended long-form research beyond verifiable rewards: we train a shared-backbone reflection meta-policy with raw task trajectories as fixed context and judge-assessed reflection rewards. The policy produces reflection rubrics and takeaways for a reusable reflection rubric bank. This reflection branch runs asynchronously without imposing a sequential cross-rollout bottleneck on task-policy RL.

\section{Details on Structured Scaffolds}
\label{appen:scaffold}

This section describes the rubric-guided structured scaffold in detail (Section~\ref{sec:scaffold}), presents the full agent system prompt, and documents the SFT data generation pipeline.

\subsection{Scaffold Description and System Prompt}
\label{app:scaffold_sub_section}

The structured reasoning scaffold decomposes each agent trajectory into four semantically distinct stages, each marked by XML tags and governed by specific behavioral requirements. We describe each stage in detail below, followed by the full system prompt.

\subsubsection{Stage 1: Planning (\textsc{Plan})}

The planning stage is the foundation of the entire trajectory. Upon receiving a user query, the agent must:

\begin{enumerate}[nosep]
    \item \textbf{Exploratory thinking} (\texttt{<think>}): The agent begins with unstructured brainstorming in a computational workspace. This block is used to identify initial questions, obvious roadblocks, missing variables, and to assess the multi-dimensional complexity of the query (retrieval difficulty, reasoning load, intellectual depth, and formatting demands). No structured XML tags are used inside this block.

    \item \textbf{Structured plan} (\texttt{<structured\_plan>}): The agent then produces a visible, structured planning document containing exactly three sub-components:
    \begin{itemize}[nosep]
        \item \texttt{<deep\_analysis>}: Deconstructs the user's query into explicit needs (what they directly ask for), implicit needs and gaps (hidden constraints, missing variables, potential roadblocks), and a complexity assessment that determines how much effort the agent should invest.
        \item \texttt{<rubric>}: The agent generates an agent planning rubric as a prospective checklist for its research and answer. This planning rubric contains three categories: (i) a \emph{knowledge checklist} specifying the exact facts, definitions, comparisons, or data points required; (ii) \emph{analytical and synthesis criteria} describing the intellectual connections the answer must achieve (optional for simple queries); and (iii) \emph{negative constraints} listing what the answer must explicitly avoid. Crucially, these planning rubrics are formulated \emph{before} any search is conducted, as prospective target objectives.
        \item \texttt{<research\_plan>}: A logical roadmap to satisfy the rubric. Simple queries receive a linear one- or two-step plan, while complex queries receive a conditional, look-ahead strategy with explicit routing logic (e.g., ``If X confirms Y, investigate Z; if X is inconclusive, fallback to W'').
    \end{itemize}

    \item \textbf{First tool call} (\texttt{<call\_tool>}): Immediately after the plan, the agent executes the first step of its research plan. The agent is strictly forbidden from producing an \texttt{<answer>} in the first turn—it must always search first.
\end{enumerate}

A key design principle is \emph{adaptive cognitive effort}: the depth of planning scales with query complexity. For a simple factual lookup, the analysis and rubric are brief; for a complex multi-faceted research question, the agent deploys its full planning machinery with detailed rubric criteria and multi-step conditional plans.

\subsubsection{Stage 2: Research (\textsc{Research})}

The research stage is an iterative loop of evidence gathering and evaluation. After each tool output, the agent:

\begin{enumerate}[nosep]
    \item \textbf{Evaluation thinking} (\texttt{<think>}): Digests new evidence in an unstructured inner monologue, noting conflicts, necessary pivots, or writing hurdles.

    \item \textbf{State evaluation} (\texttt{<state\_evaluation>}): Produces a visible evaluation of the current state of evidence, comparing accumulated findings against the rubric and research plan. Based on this evaluation, the agent chooses one of two paths:
    \begin{itemize}[nosep]
        \item \textbf{Path A---Continue research:} If more information is needed, the agent either (a) issues another \texttt{<call\_tool>} directly if the current plan remains valid, or (b) outputs an updated \texttt{<structured\_plan>} with revised analysis, rubric, and/or research plan before the next tool call. This dynamic plan revision enables adaptive multi-hop reasoning when initial assumptions are invalidated.
        \item \textbf{Path B---Proceed to review:} If evidence is sufficient to satisfy the rubric, the agent transitions to the review stage.
    \end{itemize}
\end{enumerate}

The research stage may iterate multiple times. The outer scaffold is sequential (Plan $\rightarrow$ Research $\rightarrow$ Review $\rightarrow$ Answer), but within Research, the agent can freely iterate between tool use and state evaluation, and can revise its plan in place.

\subsubsection{Stage 3: Review (\textsc{Review})}

Before producing the final answer, the agent \emph{must} perform a structured self-evaluation. This is not optional—the scaffold enforces it as a mandatory step. The review block (\texttt{<review>}) contains:

\begin{itemize}[nosep]
    \item \texttt{<rubric\_review>}: The agent maps its retrieved evidence back to the specific criteria in its Phase~1 planning rubric (plus any dynamic updates made during research). This includes: (i) \emph{knowledge verification}---explicitly listing which checklist items have been satisfied with retrieved evidence; and (ii) \emph{synthesis and constraints restatement}---re-articulating the analytical criteria and negative constraints in the context of the actual evidence gathered, transforming abstract rubric targets into concrete, specific directives for the writing phase.

    \item \texttt{<writing\_plan>}: An architectural outline for the final answer, scaled to query complexity. For simple queries, this is a brief statement of the verified facts to output. For complex long-form queries, it includes: a unified core thesis, a value proposition (why this answer is exceptional), a narrative architecture with section-by-section outline, and a citation mapping that specifies how and where verified facts will be woven into specific sections.
\end{itemize}

This mandatory review stage serves two functions: it forces the agent to verify rubric satisfaction before writing (preventing premature or incomplete answers), and it produces a writing plan that structures the final synthesis rather than letting the agent write in a stream-of-consciousness fashion.

\subsubsection{Stage 4: Answer (\textsc{Answer})}

The final stage produces the long-form response within \texttt{<answer>...</answer>} tags. The answer must follow the writing plan established in the review, satisfy all rubric criteria verified in the rubric review, and ground nontrivial claims with inline citations using \texttt{<cite id="SNIPPET\_ID">claim text</cite>} format. The depth, tone, and formatting are dictated by the response format (exact match, short-form, or long-form) and the complexity level assessed during planning.

\subsubsection{Cross-Stage Design Principles}

Several design principles span all four stages:

\begin{itemize}[nosep,leftmargin=*]
    \item \textbf{Agent planning rubric as the execution guide.} The agent planning rubric is generated at the start and actively conditions every subsequent decision: what to search for, how to evaluate evidence, and what the final answer must include or avoid. This makes the planning rubric the shared interface across the entire trajectory.

    \item \textbf{One action per turn.} Each agent generation must end with exactly one of \texttt{</call\_tool>} or \texttt{</answer>}. The agent must stop generating immediately after the closing tag. This prevents the agent from hallucinating tool outputs or bypassing the research phase, and enables the tool execution environment to interleave real search results.

    \item \textbf{Substance over syntax.} The response format instructions (exact/short/long) dictate the surface-level shape of the answer, but the self-generated rubric dictates the substance. The agent is instructed to enforce its rubric before drafting, preventing format instructions from overriding deep analytical criteria.
\end{itemize}

\subsubsection{Full System Prompt}

We present the complete system prompt below. The workflow example section is abbreviated for space.

\begin{tcolorbox}[colback=gray!5, colframe=gray!50, title=Full Agent System Prompt, breakable]
\small
\begin{verbatim}
# Role: Elite Autonomous Research & Synthesis Agent

You are an elite agent designed to analyze complex queries,
autonomously perform rubric-guided plans, execute iterative
research, and synthesize highly rigorous, evidence-backed
answers. You operate in a continuous loop: rubric-guided
planning, searching, evaluating, and answering. You DO NOT
answer the user's question directly until you have sufficient
evidence, and you DO NOT simulate tool outputs.

## CORE PRINCIPLE: ADAPTIVE COGNITIVE EFFORT
Your level of planning rigor, research depth, and synthesis
MUST dynamically scale with the true, multi-dimensional
complexity of the user's query and the evolving search
landscape.
- Complexity is Multi-Dimensional: It encompasses retrieval
  difficulty, reasoning load, intellectual depth, and
  formatting demands.
- Scale Your Effort: Do not over-engineer simple tasks.
  Conversely, for highly complex tasks, deploy your full
  intellectual machinery.

## EXPECTED RESPONSE FORMATS & SYNTAX RULES
Map your drafting strategy to the requested format appended
to the prompt. ALL final outputs MUST be enclosed within
<answer>...</answer>:

1. Exact Match: Format strictly as \boxed{exact answer}.
2. Short-Form: Write a single, cohesive paragraph with
   <cite id="...">...</cite> support.
3. Long-Form: Write a comprehensive, markdown-structured
   response. Synthesize sources into a cohesive narrative.
   Ground all nontrivial claims with <cite id="...">...</cite>.

CRITICAL: formatting instructions appended to the user's
prompt dictate the shape of your answer, but your dynamically
updated <rubric> dictates the substance.

---

## PHASE 1: PLANNING & INITIATION

STEP 0: Exploratory Thinking
Begin with a <think> block. Use this as a computational
workspace for brief, unstructured brainstorming. Identify
initial questions, obvious roadblocks, and missing variables.

STEP 1: The Structured Plan
Output a single <structured_plan> block containing exactly
three sections, building upon each other logically:

1. <deep_analysis>: Discover the True Intent.
   - Complexity Assessment: Evaluate the multi-dimensional
     complexity (retrieval, reasoning, insight, formulation).
   - Explicit Needs: What they are directly asking for.
   - Implicit Needs & Gaps: Hidden constraints, missing
     variables, or potential roadblocks.

2. <rubric>: Define the Strict Grading Criteria. Act as an
   expert grader creating a rigorous checklist. Draw upon the
   <deep_analysis> as a foundation. DO NOT focus on formatting;
   focus on required content, intellectual depth, and logical
   constraints.
   - Knowledge Checklist: The exact facts, definitions,
     comparisons, or data points required.
   - Analytical & Synthesis Criteria (Optional, for complex
     queries): What intellectual connections must the response
     achieve?
   - Negative Constraints (Pitfalls): What must the response
     explicitly AVOID?

3. <research_plan>: Formulate the Strategy. Create a logical
   roadmap to satisfy the rubric. Simple queries need a linear
   step or two; complex queries require a conditional,
   look-ahead strategy.

STEP 2: The First Tool Call
Immediately after closing </structured_plan>, execute ONLY
the first step of your research plan. Under NO circumstances
should you output an <answer> in this first turn.

---

## PHASE 2: SEARCH & SYNTHESIS

STEP 0: Evaluation Thinking
Begin with a <think> block for raw inner monologue to digest
new evidence.

STEP 1: State Evaluation & Action
Output a <state_evaluation> block to analyze the tool output,
then choose exactly ONE path:

PATH A: Continue Research (More Information Needed)
- Direct Search: If the current plan and rubric are still
  valid, output your next <call_tool>.
- Update & Search: If new information invalidates initial
  assumptions, output an updated <structured_plan> with
  revised <deep_analysis>, <rubric>, and/or <research_plan>,
  then output your next <call_tool>.

PATH B: Ready to Answer (Evidence Sufficient)
1. The Review: Output a <review> block containing:
   - <rubric_review>: Systematically map retrieved evidence
     back to the Phase 1 <rubric>. Verify knowledge checklist
     items. Re-articulate synthesis criteria and negative
     constraints in the context of retrieved evidence.
   - <writing_plan>: Outline the final answer's architecture.
     For complex queries: define a unified core thesis, value
     proposition, narrative architecture, and citation mapping.
2. The Final Answer: Output the response within
   <answer>...</answer> tags.

---

## Available Tools
- google_search: Powered by a grounded AI reasoning engine.
  Use direct, highly specific search queries.
  Input: <call_tool name="google_search">query</call_tool>

- snippet_search: Focused retrieval from scientific papers.
  Input: <call_tool name="snippet_search">query</call_tool>
  Optional parameters: limit, year, fieldsOfStudy.

## Tool Output Format
- Results are appended in <tool_output>...</tool_output> tags.
- For google_search: summary text with grounding snippets.
- For snippet_search: <snippet id="ID">content</snippet>.

## WORKFLOW EXAMPLE
[... A complete worked example demonstrating the full scaffold
on a sample query, showing Phase 1 (think, structured_plan,
first call_tool), Phase 2 iterations (state_evaluation,
continued search, dynamic plan updates), and the review and
answer synthesis. Omitted for brevity. ...]

## CRITICAL CONSTRAINTS
- Mandatory Starting Tag: Output MUST begin with <think>.
- Mandatory First Search: First turn MUST end with
  </call_tool>. Never answer without searching.
- One Action Per Turn: End with </call_tool> OR </answer>.
  STOP generating immediately after. Do NOT simulate
  <tool_output> yourself.
- Mandatory Citations: Wrap claims in
  <cite id="S_123">claim text</cite>. Never use empty tags.
  Multiple sources: <cite id="S_1, S_2">claim</cite>.
- Substance over Syntax: The <rubric> dictates substance;
  formatting instructions dictate only surface-level shape.
\end{verbatim}
\end{tcolorbox}

\paragraph{Additional instructions.} Depending on the expected response format, one of three task-specific instructions is appended to the user's query:
\begin{itemize}[nosep,leftmargin=*]
    \item \textbf{Exact answer:} ``Search iteratively to find the precise answer. Format: \texttt{<answer>\textbackslash boxed\{exact answer\}</answer>}.''
    \item \textbf{Short form:} ``Search iteratively to gather evidence from multiple credible sources. Synthesize into a short paragraph within \texttt{<answer>...</answer>} tags.''
    \item \textbf{Long form:} ``Search iteratively to gather evidence from multiple credible sources. Synthesize into a comprehensive, evidence-backed long-form response within \texttt{<answer>...</answer>} tags.''
\end{itemize}

\subsection{SFT Data Generation Process}
\label{app:sft-data}

To instill the rubric-guided scaffold into the base Qwen3-8B model, we perform supervised fine-tuning using teacher-generated trajectories. The data generation process consists of a multi-round, tool-augmented pipeline that produces complete stage-structured trajectories from a teacher model.

\subsubsection{Teacher Model and Prompt Adaptation}

We use Gemini-3.1-Pro as the teacher model. The training queries are drawn from the same dataset used by Dr.~Tulu~\citep{shao2025dr}, which contains ${\sim}$13K diverse, open-ended research questions spanning multiple domains. Each query is paired with a response format label (\texttt{exact\_answer}, \texttt{short\_form}, or \texttt{long\_form}) that determines the additional instruction appended to the prompt.

A key challenge is that the SFT data generation prompt must be adapted for Gemini, which differs from the target Qwen3-8B model in several important ways.

\paragraph{Reasoning traces and \texttt{<scratchpad>}.}
Our scaffold includes unstructured \texttt{<think>} blocks at the start of each turn to preserve the base Qwen3-8B policy's post-training reasoning behavior. However, directly prompting Gemini to produce \texttt{<think>} tags is blocked by the API, and internal thinking traces are inaccessible due to anti-distillation policies. We use \texttt{<scratchpad>} as a substitute reasoning tag during SFT data generation; these are converted to \texttt{<think>} in post-processing. While these traces do not capture Gemini's true internal chain-of-thought, they provide visible planning rationale and evidence assessment that teaches the student model the habit of explicit reasoning before acting, without disrupting its native \texttt{<think>} patterns.

\paragraph{Stagewise prompt separation.}
We separate the system prompt into a first-round variant (containing only Phase 1 instructions) and a later-rounds variant (containing only Phase 2 instructions), since each round is a separate Gemini API call. This separation serves two purposes: it keeps each prompt focused on the current stage's requirements (reducing prompt length and improving adherence), and it ensures the model invests heavily in planning at step 1 (generating the full \texttt{<structured\_plan>} with analysis, rubric, and research plan) rather than skipping ahead to search or answering directly. Each variant explicitly constrains the turn-ending tag: the first round must end with \texttt{</call\_tool>}, and later rounds must end with either \texttt{</call\_tool>} or \texttt{</answer>}.

We present the core of each generation prompt below (omitting tool documentation and workflow examples, which mirror the final agent prompt).

\begin{tcolorbox}[colback=gray!5, colframe=gray!50, title=SFT Generation Prompt: First Round (Phase 1 Only), breakable]
\small
\begin{verbatim}
# Role: Elite Research Planning & Initiation Agent

You are an elite Research Planning Agent. Your purpose is to
analyze complex queries, generate a rigorous, visible research
plan, and initiate the very first search query to kick off the
research process. You DO NOT answer the user's question
directly, and you DO NOT simulate tool outputs.

## CORE PRINCIPLE: ADAPTIVE COGNITIVE EFFORT
[... same as final prompt ...]

## Process

STEP 0: Exploratory Scratchpad
Begin with a <scratchpad> block. This is your mandatory
computational workspace. Use this space for a brief,
unstructured exploration of the user's query --- identify
initial hurdles and missing variables before moving to
structured output. Do not output XML tags inside scratchpad.

STEP 1: The Structured Plan
Output a single <structured_plan> block containing exactly
three sections, building upon each other logically:

1. <deep_analysis>: Discover the True Intent.
   - Complexity Assessment: Evaluate the multi-dimensional
     complexity (retrieval, reasoning, insight, formulation).
   - Explicit Needs: What they are directly asking for
     (including structural/formatting instructions).
   - Implicit Needs & Gaps: Hidden constraints, missing
     variables, or potential roadblocks.

2. <rubric>: Define the Strict Grading Criteria. Act as an
   expert grader creating a rigorous checklist. DO NOT focus
   on formatting; focus on content, depth, and constraints.
   - Knowledge Checklist: Exact facts, definitions, data
     points required (e.g., "The response explicitly defines
     the von Neumann bottleneck").
   - Analytical & Synthesis Criteria (Optional, for complex
     queries): What intellectual connections must the response
     achieve?
   - Negative Constraints (Pitfalls): What must the response
     explicitly AVOID? (e.g., "The response avoids presenting
     industry blogs as academic consensus").

3. <research_plan>: Formulate the Strategy. Create a logical
   roadmap to satisfy the rubric. Simple queries: linear
   steps. Complex queries: conditional, look-ahead strategy
   (e.g., "Step 1: Find X. Step 2: If X confirms Y,
   investigate Z; if inconclusive, fallback to W").

STEP 2: The First Tool Call
Immediately after closing </structured_plan>, execute ONLY
the first step of your <research_plan>. Once you output
</call_tool>, you must STOP generating text.

## CRITICAL CONSTRAINTS
- Output MUST begin with <scratchpad>, then <structured_plan>.
- Execute EXACTLY ONE tool call.
- DO NOT simulate <tool_output>.
- DO NOT attempt to write the final <answer>.

[... tool documentation and workflow examples omitted ...]
\end{verbatim}
\end{tcolorbox}

\begin{tcolorbox}[colback=gray!5, colframe=gray!50, title=SFT Generation Prompt: Later Rounds (Phase 2 Only), breakable]
\small
\begin{verbatim}
# Role: Elite Iterative Research & Synthesis Agent

You are in the active research phase. You have already
generated an initial Rubric and Research Plan. Your goal is to
evaluate the latest retrieved information, dynamically adjust
your strategy if necessary, and either continue searching or
generate the final answer.

## CONTEXT: THE PHASE 1 FOUNDATION
Prior to your current phase, an initial <structured_plan> was
generated. This plan is your foundational blueprint consisting
of <deep_analysis>, <rubric>, and <research_plan>. You must
constantly evaluate incoming evidence against this baseline.

## Process & Decision Tree

STEP 0: Evaluation Scratchpad
Begin with a <scratchpad> block. This is your mandatory
computational workspace for raw, unstructured inner monologue
to digest newly retrieved evidence. If data contradicts
assumptions, briefly identify the conflict and note your
pivot. If evidence is sufficient for a complex query, pinpoint
writing hurdles before resolving them in <review>.

STEP 1: State Evaluation & Action
Output a <state_evaluation> block to formally analyze the
latest tool output, then choose exactly ONE path:

PATH A: Continue Research (More Information Needed)
- Direct Search: If the current plan and rubric are still
  valid, output your next <call_tool> immediately.
- Update & Search: If new information invalidates initial
  assumptions, triggers a fallback, or reveals a deeper
  multi-hop requirement:
  1. Output an updated <structured_plan> with revised
     <deep_analysis>, <rubric>, and/or <research_plan>.
     (Only include sections that require updates.)
  2. Then output your next <call_tool>.

PATH B: Ready to Answer (Evidence Sufficient)
1. The Review: Output a <review> block containing:
   - <rubric_review>: Systematically map retrieved evidence
     back to the Phase 1 <rubric> (plus any dynamic updates).
     * Knowledge Verification: List the critical facts and
       data points that satisfy the Knowledge Checklist.
     * Synthesis & Constraints Restatement: Re-articulate
       Analytical Criteria and Negative Constraints in the
       context of retrieved evidence. Transform abstract
       rubric targets into concrete directives for drafting.
   - <writing_plan>: Outline the final answer architecture,
     scaling effort to query complexity:
     * Low-Complexity: State verified facts directly.
     * High-Complexity: Define (1) unified core thesis,
       (2) value proposition, (3) narrative architecture
       with section-by-section outline, (4) citation mapping
       specifying how verified facts integrate into sections.
2. The Final Answer: Output within <answer>...</answer> tags.
   Depth and formatting MUST match the writing plan. Follow
   the Knowledge Checklist and honor all constraints from
   the <rubric_review>.

## CRITICAL CONSTRAINTS
- Output MUST begin with <scratchpad>, then <state_evaluation>.
- One Action Per Turn: end with </call_tool> OR </answer>.
  STOP generating immediately after the closing tag.
- Mandatory Citations: wrap claims in
  <cite id="S_123">claim text</cite>. Never use empty tags.
- Substance over Syntax: formatting instructions dictate
  shape; the <rubric> dictates substance. Enforce the rubric
  before drafting.

[... tool documentation and workflow examples omitted ...]
\end{verbatim}
\end{tcolorbox}

The key difference from the unified final agent prompt is the strict separation of responsibilities: the first-round prompt forbids answering and requires a full planning phase, while the later-rounds prompt assumes the plan already exists and focuses on evidence evaluation and synthesis. This separation ensures the teacher invests in thorough planning rather than short-circuiting to an answer.

\paragraph{Common failure modes.}
Despite these adaptations, Gemini-3.1-Pro frequently violates the scaffold constraints. Common failures include answering directly from internal knowledge without searching, producing output that does not end with the required closing tag, omitting required structural elements such as \texttt{<rubric>} or \texttt{<state\_evaluation>}, and generating variant tool names that do not match the expected schema. These violations necessitate aggressive rejection sampling (described below).

\subsubsection{Multi-Round Generation Pipeline}

Because the scaffold requires iterative tool interaction, we cannot generate complete trajectories in a single model call. Instead, we use a multi-round pipeline that alternates between model generation and real tool execution:

\begin{enumerate}[nosep]
    \item \textbf{Round 1 (Initiation):} The teacher receives the first-round system prompt and user query. It generates Phase 1 output: \texttt{<scratchpad>}, \texttt{<structured\_plan>} (with \texttt{<deep\_analysis>}, \texttt{<rubric>}, \texttt{<research\_plan>}), and the first \texttt{<call\_tool>}. Generation stops at the closing \texttt{</call\_tool>} tag.

    \item \textbf{Tool routing and execution:} Each \texttt{<call\_tool>} is parsed to extract the tool name and query. The pipeline routes calls to one of two backends:
    \begin{itemize}[nosep]
        \item \texttt{google\_search}: Submitted as Vertex AI batch prediction jobs using Gemini with Google Search grounding enabled, which returns AI-synthesized summaries with grounding snippets.
        \item \texttt{snippet\_search}: Executed asynchronously against the Semantic Scholar API, which returns paper excerpts matching the query.
    \end{itemize}
    Both backends run in parallel to maximize throughput. Responses where the tool fails to call a valid tool (e.g., answering from internal knowledge) are separated into an error stream for potential reprocessing.

    \item \textbf{Rounds 2--$N$ (Research iteration):} The tool output is appended to the conversation history, and the teacher receives the later-rounds system prompt. It generates the next Phase 2 turn: \texttt{<scratchpad>}, \texttt{<state\_evaluation>}, and either another \texttt{<call\_tool>} or the \texttt{<review>} and \texttt{<answer>}. This loop continues until the teacher produces a closing \texttt{</answer>} tag or a maximum of 10 rounds is reached.

    \item \textbf{Completion detection:} A trajectory is considered complete when the teacher's output ends with \texttt{</answer>}. Incomplete trajectories that exhaust the maximum rounds are discarded. Because Gemini often fails to produce a well-formed closing tag, a substantial fraction of trajectories (${\sim}$15--25\%) are discarded at this stage.
\end{enumerate}

\subsubsection{Quality Filtering and Rejection Sampling}

Raw teacher trajectories undergo rejection sampling and several filtering steps. This is critical because Gemini-3.1-Pro, despite being a strong model, frequently violates the scaffold's structural constraints when generating multi-round tool-augmented trajectories.

\paragraph{Rejection criteria.} The following trajectories are rejected:

\begin{itemize}[nosep,leftmargin=*]
    \item \textbf{Missing \texttt{</answer>} tag (hard reject):} The most common failure mode. The teacher generates extensive content but never produces the closing \texttt{</answer>} tag, either because it runs out of tokens, enters an endless research loop, or simply stops mid-sentence. These trajectories are discarded entirely.

    \item \textbf{No valid tool call in non-final rounds:} Each non-final round must end with a valid \texttt{</call\_tool>} tag containing a parseable tool name and query. Trajectories where the model skips tool use and attempts to answer from internal knowledge (e.g., producing a long response without any \texttt{<call\_tool>}) are rejected. This is the second most common failure mode, particularly on simple factual queries where Gemini ``knows'' the answer and refuses to search.

    \item \textbf{Missing structural elements:} Trajectories must contain valid \texttt{<structured\_plan>} (with \texttt{<deep\_analysis>} and \texttt{<rubric>}) in the first round, and \texttt{<state\_evaluation>} in subsequent rounds. Those with malformed XML or missing required sections are discarded.

    \item \textbf{Consecutive tool errors:} Trajectories with two or more consecutive rounds where the tool backend returned an error (e.g., search blocked by safety filters, API rate limits, ``no grounding data available'') are discarded. Single error rounds followed by successful recovery are retained, as these teach the model error-resilient behavior.
\end{itemize}

\paragraph{Post-processing transformations.} Accepted trajectories undergo the following conversions:

\begin{itemize}[nosep,leftmargin=*]
    \item \textbf{Reasoning tag conversion:} All \texttt{<scratchpad>...</scratchpad>} tags from Gemini's output are converted to \texttt{<think>...</think>} for compatibility with Qwen3's chat template, which uses \texttt{<think>} as the native reasoning tag.

    \item \textbf{Tool name normalization:} Gemini occasionally produces variant tool names (e.g., \texttt{google\_web\_search}, \texttt{Scholar\_Search}) or malformed attribute strings. These are normalized to the canonical \texttt{google\_search} and \texttt{snippet\_search} names.
\end{itemize}

\subsubsection{Training Data Format}

Each filtered trajectory is converted into a single-turn ChatML conversation with three messages:
\begin{itemize}[nosep,leftmargin=*]
    \item \textbf{System:} The full agent system prompt (the unified Qwen3 version from Sec.~\ref{app:scaffold_sub_section}, not the Gemini-adapted generation prompt).
    \item \textbf{User:} The original query with the format-specific additional instruction appended.
    \item \textbf{Assistant:} The complete multi-round trajectory, with tool outputs wrapped in special span-masking markers. Only the model-generated tokens (thinking, planning, tool calls, state evaluation, review, and answer) receive gradient during training; tool output tokens are masked so the model does not learn to memorize search results.
\end{itemize}

Finally, this process yields around 11k SFT samples, approximately 2k fewer than DR Tulu due to the repeated errors described above, which are filtered out.

\section{Details on Stagewise Evolving Rubric Evaluation}
\label{app:stagewise_judge}

This section provides implementation details for the stagewise evolving rubric design described in Section~\ref{sec:ssgrpo}. The judge operates in two phases: \emph{adaptive rubric generation} proposes new rubrics by comparing trajectories for the same question, and \emph{stagewise scoring} evaluates each trajectory against the current rubric set. Both phases use an LLM judge (Gemini-3-Flash) with structured JSON output.

\subsection*{Clarification: Rubric Roles and Evaluation Protocol}
\phantomsection
\label{app:rubric-roles}

The four rubric roles below have distinct sources and uses. Their information flow follows the training and evaluation protocol.

\begin{center}
\small
\setlength{\tabcolsep}{3pt}
\renewcommand{\arraystretch}{1.12}
\begin{tabular}{@{}p{0.15\linewidth}p{0.24\linewidth}p{0.22\linewidth}p{0.34\linewidth}@{}}
\toprule
\textbf{Rubric} & \textbf{Source} & \textbf{Purpose} & \textbf{Visibility and use} \\
\midrule
Agent planning & Task policy, during Plan & Guide research, review, and writing & Part of the rollout; may inform judge criterion discovery. \\
Judge scoring & Training-data criteria and adaptive judge criteria & Score task trajectories for SS-GRPO & Shared within each query's rollout group; applied separately to each trajectory. \\
Reflection & Reflection policy, from query and raw trajectory & Guide later attempts through the bank & Inserted as context when retrieval is enabled. \\
Benchmark\newline evaluation & Each benchmark's evaluation protocol & Grade final outputs & Evaluator prompts and grading criteria are withheld from the policy. \\
\bottomrule
\end{tabular}
\end{center}

\textbf{Criterion discovery and scoring are separate.} Agent planning rubrics can inform the discovery of judge scoring rubrics, but a trajectory's own planning rubric is not its scoring standard.
The reflection generator sees the query and raw trajectory, including the agent's own final answer; \textbf{only the reflection judge receives privileged judge scoring rubrics, scores, and justifications}.
All scored candidates supply reflection-training signals, while only the highest-scored valid candidate enters the bank (Appendix~\ref{app:async-implementation}).

\textbf{Benchmark grading uses benchmark-specific criteria, independently of training rewards.}
Tables~\ref{tab:model_performance} and~\ref{tab:model_performance_2} and Fig.~\ref{fig:ablation_study} use \textbf{no reflection-bank retrieval or previous reflections} at evaluation; Fig.~\ref{fig:ablation_scaffolding_scaling}(d) evaluates experience reuse separately (Appendix~\ref{app:reflection-eval}).

\subsection{Judge Scoring Rubric Buffer}
\label{app:judge-buffer}

For each training question, we maintain a \emph{judge scoring rubric buffer} with two types of criteria: \emph{persistent rubrics} (static, ground-truth rubrics from the training data, never removed) and \emph{active rubrics} (adaptive rubrics proposed online by the judge, organized by stage and subject to per-stage capacity caps).
The persistent rubrics are adapted from~\cite{shao2025dr} and only used for the final answer stage.
At each training step, for each unique question in the batch, the judge (i) generates new adaptive rubrics by comparing the rollout group, (ii) adds them to the buffer, (iii) scores all trajectories against the combined rubric set, and (iv) removes low-discrimination rubrics that did not meaningfully separate trajectory quality.

\subsection{Adaptive Rubric Generation}
\label{app:rubric-generation}

The rubric generation prompt instructs the judge to identify the most discriminative, stage-local criteria that explain why some trajectories are better than others. Several design principles considered:

\begin{itemize}[nosep,leftmargin=*]
    \item \textbf{Discriminative and specific:} Each rubric must actually separate stronger from weaker trajectories for the same question. Descriptions must reference concrete aspects of the question so a separate scorer can unambiguously judge them. Vague rubrics like ``good research quality'' are prohibited.
    \item \textbf{Stage-local and non-redundant:} Rubrics target each stage's specific responsibility. The judge avoids creating positive and negative mirror versions of the same criterion, and skips criteria already covered by existing rubrics.
    \item \textbf{Anti-hack:} The judge must not create rubrics that merely check XML tag existence, reward formal obedience to a weak plan, overfit to one trajectory's wording, or confuse self-consistency with true quality.
\end{itemize}

A key feature enabled by our rubric-guided scaffold is that the judge can use agent planning rubrics (\texttt{<rubric>} blocks) from different trajectories as \emph{references} when proposing adaptive judge scoring rubrics. Strong planning rubrics across trajectories can reveal important task dimensions and failure modes. However, a trajectory's own planning rubric is never the scoring standard: the judge is instructed to never reward following a weak self-rubric or penalize evidence-based improvement beyond the original plan.

Rather than prescribing fixed templates, the prompt provides stage-specific guidance as starting points. For each stage, a central question frames the evaluation: ``Did the agent understand the problem?'' (Plan), ``Did it search effectively and adapt?'' (Research), ``Did it honestly audit readiness?'' (Review), and ``Is the final answer high-quality?'' (Answer). Common useful dimensions are suggested (e.g., query specificity, evidence-based pivoting for Research), but the judge is encouraged to discover what matters for each specific question.

Each rubric item contains a \texttt{title}, \texttt{description}, and \texttt{weight} (1=minor, 2=important, 3=critical), organized into positive and negative rubrics per stage. Target counts are 1--4 rubrics per stage for stages 1--3 and 2--5 for stage 4, treated as targets rather than quotas. The full rubric generation prompt is shown below:

\begin{tcolorbox}[colback=gray!5, colframe=gray!50, title=Stagewise Adaptive Rubric Generation Prompt (Core), breakable]
\small
\begin{verbatim}
You are an expert evaluator generating stagewise comparative
rubrics for multiple full trajectories produced by the SAME
research agent on the SAME question.

## Goal
Identify the most discriminative, stage-local criteria that
explain why some trajectories are better than others. Evaluate
real process quality and answer quality — not XML compliance,
verbosity, or shallow self-consistency.

## How to Read the Trajectories
Each trajectory is a structured full execution trace:
- Stage 1 — Plan: <structured_plan> with <deep_analysis>,
  <rubric>, <research_plan>
- Stage 2 — Research: iterative <state_evaluation>, tool
  calls, and possible plan/rubric updates
- Stage 3 — Review: <review> with <rubric_review> and
  <writing_plan>
- Stage 4 — Answer: final <answer>

Do not reward a stage just because the relevant tags exist.
Judge whether that stage actually performs its intended
function well.

## How to Use Agent-Generated Rubrics
A trajectory's own Stage 1 <rubric> is never the scoring
standard. Use it only as a reference signal:
- it shows what the agent recognized as important,
- it may surface useful task dimensions,
- it may expose shallow planning if it misses critical needs.
But: never reward following a weak self-rubric, never let a
narrow self-rubric lower the real standard, never penalize
evidence-based improvement beyond the original plan.

## Stage-by-Stage Guidance
For each stage, the most discriminative criteria often emerge
from the specific question and trajectories themselves. The
dimensions below are common starting points:

Stage 1 — PLAN: task understanding depth, identification of
  hidden constraints, rubric specificity, plan proportionality.
Stage 2 — RESEARCH: query specificity, evidence-based
  pivoting, triangulation, stopping criteria.
Stage 3 — REVIEW: honest gap assessment, evidence-to-
  structure mapping, drafting plan quality.
Stage 4 — ANSWER: readability, instruction following,
  insightfulness, comprehensiveness, correctness, evidence
  grounding. Evaluated independently from earlier trajectory.

## Anti-Hack Rules
Do NOT create rubrics that: only check tag existence, reward
formal obedience to a bad plan, overfit to one trajectory's
wording, confuse self-consistency with true quality, punish
evidence-based improvement, or reward unnecessary overhead.

## Selection Rules
- Fewer, sharper rubrics over many generic ones.
- Skip criteria where all trajectories perform similarly.
- Never create positive and negative mirrors of the same idea.
- Conservative negative rubrics: target clear failure modes.

## Output Format
Return rubrics by stage (stage_1 through stage_4), each with
positive_rubrics and negative_rubrics. Each item: title,
description, weight (1-3).
\end{verbatim}
\end{tcolorbox}

\subsection{Stagewise Scoring and Stability}
\label{app:stagewise-scoring}

After rubric generation, each trajectory is scored independently against the combined rubric set (persistent + active). Positive rubrics use a 0/1/2 scale (absent / partial / full), and negative rubrics use the same scale but inverted during aggregation: a score of 0 (no flaw) contributes positively. Each evaluation includes a brief justification grounded in specific trajectory evidence. For each stage $k$, the per-stage score $R_{i,k} \in [0,1]$ is a weighted average of rubric scores within that stage, with negative scores inverted before aggregation.

\begin{samepage}
\paragraph{Sensitivity to final-answer visibility.}
\phantomsection
\label{app:answer-masking}
We randomly sample 265 trajectories from training steps 100, 500, and 1,000, covering 2,045 intermediate rubric judgments across Plan, Research, and Review.
Holding the judge and rubric items fixed, we rescore each trajectory in two conditions: the full trajectory and the same trajectory truncated immediately before \texttt{<answer>}.
The intermediate judgments remain largely stable under final-answer masking, with 89.9\% exact agreement and a Pearson correlation of 0.876 overall.

\begin{center}
\small
\setlength{\tabcolsep}{6pt}
\renewcommand{\arraystretch}{1.1}
\begin{tabular}{@{}lrrrr@{}}
\toprule
\textbf{Visible vs.\ masked} & \textbf{Overall} & \textbf{Plan} & \textbf{Research} & \textbf{Review} \\
\midrule
Exact agreement (\%) & 89.9 & 90.5 & 89.6 & 89.6 \\
Pearson $r$ & 0.876 & 0.853 & 0.899 & 0.845 \\
\bottomrule
\end{tabular}
\end{center}
\end{samepage}

\begin{tcolorbox}[colback=gray!5, colframe=gray!50, title=Stagewise Scoring Prompt, breakable]
\small
\begin{verbatim}
You are an expert evaluator. You will receive an AI research
agent's full trajectory and a set of rubrics grouped by stage.
Score each rubric strictly based on the trajectory evidence.

## Trajectory Structure
- Stage 1 (Plan): start through end of </structured_plan>
- Stage 2 (Research): after </structured_plan> through all
  tool-call loops
- Stage 3 (Review): last <think> before <review> through
  </review>
- Stage 4 (Answer): <answer> through end of text

## Scoring Rules
[POSITIVE] rubrics:
  2 = FULLY exhibits this quality
  1 = PARTIALLY exhibits it
  0 = Completely ABSENT or fails

[NEGATIVE] rubrics:
  2 = Fully EXHIBITS the flaw
  1 = PARTIALLY exhibits the flaw
  0 = Successfully AVOIDS the flaw

## Instructions
- Score each rubric independently.
- Use the rubric description as the sole scoring criterion.
- Provide brief justification grounded in specific trajectory
  evidence (quote or reference concrete content).
- Return the exact rubric index with each score.
\end{verbatim}
\end{tcolorbox}

\subsection{Buffer Management and Implementation}
\label{app:judge-implementation}

\paragraph{Buffer settings.} \textbf{Active-rubric caps:} 3, 2, 2, and 3 for Plan, Research, Review, and Answer, respectively. \textbf{Pruning:} after scoring, stages above capacity discard rubrics with the lowest score variance across the rollout group. \textbf{Persistent rubrics:} retained throughout training and used only for Answer.
When a question reappears, newly generated adaptive rubrics enter its buffer and may replace less discriminative ones. The criteria thus co-evolve with the policy, extending the evolving rubric design of~\citet{shao2025dr} from answer-only to stagewise evaluation.

\paragraph{Structured JSON output.} Both rubric generation and scoring use Gemini's structured output mode (\texttt{response\_mime\_type="application/json"}) with Pydantic-derived JSON schemas. This guarantees schema compliance (rubrics organized by stage with title/description/weight, scores with indices and justifications) without fragile post-hoc text parsing.

\paragraph{Batch parallelism.} At each training step, the batch contains 32 unique questions $\times$ 8 rollouts = 256 trajectories. All 32 rubric generation calls are launched concurrently via \texttt{asyncio.gather}, followed by all 256 scoring calls in parallel. Each call includes exponential-backoff retry logic (up to 5 retries). Failed calls are treated as ``no new rubrics'' (generation) or fall back to terminal answer score only (scoring).

\paragraph{Efficiency.} Rubric generation requires only 32 calls per step (one per question, not per trajectory). Scoring requires 256 calls but each is lightweight (${\sim}$8--15 rubrics per trajectory). The combined rubric generation + scoring phase takes approximately 5 minutes per step, overlapping with the training engine's gradient computation. The judge model (Gemini-3-Flash) is chosen for cost and latency.

\section{Asynchronous Reflection Pipeline and Windowed Curriculum}
\label{app:async-implementation}

This appendix provides the implementation details for the meta-policy training pipeline
and the windowed retrieval curriculum described in Sections~\ref{sec:evolving-policy} and~\ref{sec:implementation}.

\subsection{Training Pipeline Architecture}
\label{app:pipeline-architecture}

Each RL training step in our agentic loop involves several sequential phases:
rollout generation (multi-turn tool-augmented trajectories via vLLM),
stagewise judge scoring (SS-GRPO reward computation),
and policy gradient update.
Reflection generation (prompting the backbone to produce reflection rubrics and takeaways
from sampled trajectories) introduces an additional compute requirement that we overlap with the main loop.

Our pipeline uses three concurrent threads:
\begin{enumerate}[nosep]
    \item \textbf{Main thread (training engine):} Performs gradient updates on the policy. Alternates between Phase~A (meta-policy update on deferred reflections from the previous step) and Phase~B (task-policy SS-GRPO update on the current step's rollouts).
    \item \textbf{Inference thread (vLLM):} Generates multi-turn tool-augmented rollouts. During Phase~A, the inference engine begins generating the next batch of rollouts while the training engine trains on completed deferred reflections, allowing inference and reflection training to overlap.
    \item \textbf{Data preparation thread:} Overlaps task judging (async Gemini API calls) with reflection generation, judges the reflections once their inputs are ready, and prepares packed training batches. Reward scoring results are returned immediately for task-policy training; reflection scoring continues asynchronously.
\end{enumerate}

\subsection{One-Step Deferred Reflection Training}
\phantomsection
\label{app:deferred-training}

A synchronous implementation would block the next task rollout until reflection generation, judging, and the meta-policy update are finished.
We instead defer reflection training by one RL step, as illustrated in the bottom panel of Fig.~\ref{fig:detailed_RL}.

At step~$N$:
\begin{enumerate}[nosep]
    \item The inference engine generates task rollouts for step $N$.
    \item The data preparation thread judges step~$N$ task rollouts while vLLM generates multiple reflection candidates per sampled trajectory. Reflection judging begins once both the candidates and task-scoring evidence are available.
    \item Reward results are returned immediately; the training engine performs the SS-GRPO task-policy update on step~$N$.
    \item Meanwhile, reflection scoring completes in the background. Accepted reflections are inserted into the reflection rubric bank, and the scored reflection samples are placed in a \emph{deferred buffer}.
    \item At the start of step~$N{+}1$, during \textbf{Phase~A}, the training engine trains on the completed deferred reflection buffer from step~$N$ \emph{while} the inference engine generates step~$N{+}1$ rollouts.
\end{enumerate}

The intended one-step deferral trades exact synchrony for opportunities to overlap inference and training. When the deferred batch is ready and its update fits within the concurrent rollout interval, this work need not extend that interval. Actual latency depends on job completion, API limits, and infrastructure interruptions; the schedule does not imply zero additional compute or API expenditure.

\subsection{Windowed Curriculum}
\label{app:windowed-curriculum}

Since reflection training is deferred by one step and bank insertion happens asynchronously,
a na\"ive curriculum that immediately revisits a query after its first encounter may
attempt within-episode retrieval before the corresponding reflection has been generated and accepted.
To allow time for these background jobs, we introduce a windowed curriculum with window size~$K$
that separates a query's first and second encounters.

Training alternates between two phases within each window of $2K$ steps:

\begin{enumerate}
    \item \textbf{New-query phase} (steps $1, \ldots, K$ within the window):
    The dataloader samples $K$ fresh query batches $\{B_1, B_2, \ldots, B_K\}$.
    For each batch, the bank performs \emph{cross-episode retrieval}:
    semantically similar items from past (different) queries are injected as few-shot exemplars.
    At the end of each step, reflection generation is launched asynchronously in the background.

    \item \textbf{Repeat phase} (steps $K{+}1, \ldots, 2K$ within the window):
    The same batches are replayed in order: $B_1, B_2, \ldots, B_K$.
    For each batch, the bank performs \emph{within-episode retrieval}:
    the exact reflection generated during the new-query phase is retrieved
    and injected as direct self-guidance.
\end{enumerate}

The $K$-step gap between a query's first encounter and its repeat provides time for reflection generation, judging, bank insertion, and a deferred training update before within-episode retrieval. In our experiments, we use $K{=}3$. The intended schedule assumes these jobs complete within that interval during normal operation; a fixed step gap alone cannot guarantee completion under API delays or infrastructure interruptions.

\begin{table}[htbp]
\centering
\caption{Windowed curriculum with $K{=}3$. Each window of $2K{=}6$ steps alternates between new queries with cross-episode retrieval and repeated queries with within-episode exact retrieval. The 3-step gap allows time for the deferred reflection pipeline; availability at the repeat depends on completion of its background jobs.}
\label{tab:windowed-curriculum}
\small
\begin{tabular}{c|c|c|c}
\toprule
\textbf{Step in window} & \textbf{Phase} & \textbf{Data} & \textbf{Retrieval mode} \\
\midrule
1 & New & $B_1$ (fresh) & Cross-episode (similar) \\
2 & New & $B_2$ (fresh) & Cross-episode (similar) \\
3 & New & $B_3$ (fresh) & Cross-episode (similar) \\
4 & Repeat & $B_1$ (replay) & Within-episode (exact) \\
5 & Repeat & $B_2$ (replay) & Within-episode (exact) \\
6 & Repeat & $B_3$ (replay) & Within-episode (exact) \\
\bottomrule
\end{tabular}
\end{table}

\subsection{Trajectory Sampling and Candidate Generation}
\label{app:reflection-candidates}

For each query in the new-query phase,
we randomly sample one trajectory from the rollout group for reflection generation.
The sampling is uniform random over trajectory indices,
without length bias or score-based selection,
mirroring inference conditions where the agent does not have access to privileged scoring information.
From the sampled trajectory, we generate $n{=}8$ reflection candidates
via vLLM parallel sampling (temperature~$0.7$, single prompt, $n$~completions).
Each candidate is independently scored by the judge,
and only the highest-scored candidate with valid output format is accepted into the bank.
Each query's reflection generation is independent,
so a failure for one query does not affect others.

\subsection{Bank Persistence and Retrieval}
\label{app:bank-retrieval}

\paragraph{Index and lookup.} \textbf{Encoder:} Qwen3-Embedding-0.6B on CPU. \textbf{Index:} in-memory FAISS over query embeddings. \textbf{Cross-episode lookup:} top-$k$ items by inner-product similarity (default: $k{=}2$). \textbf{Within-episode lookup:} exact question-hash match (SHA-256).

\paragraph{Updates and persistence.} When a query already has a bank item, its new reflection overwrites the old one.
The bank is persisted \textbf{every 10 steps and at each model checkpoint}, using atomic writes (write to temp file, then rename) to prevent corruption.
On training resume, the bank is restored from the checkpoint corresponding to
the same global step as the model weights, ensuring consistency between
the policy parameters and the memory contents.

\subsection{Prompt Templates}
\label{app:rubric-bank-prompts}

\subsubsection{Reflection Generation Prompt}

The reflection policy receives the query and one raw trajectory, and produces reflection rubrics and takeaways for future attempts. Its system prompt is:

\begin{tcolorbox}[colback=gray!5, colframe=gray!50, title=Reflection System Prompt, breakable]
\small
\begin{verbatim}
You are a strict postmortem editor for long-form
deep research QA.

You will receive:
- a question
- a research trajectory ending in a final
  <answer>...</answer>

Your job is not to answer the question again. Your
job is to extract the most useful guidance for:
1. a stronger next attempt on this question, and
2. similar long-form research questions.

First, read the final <answer>.
Judge it on: correctness and calibration;
instruction-following and coverage; research quality
and verification; synthesis and insight;
communication and structure.

Then identify the 2-4 highest-leverage strengths or
weaknesses, trace them back to the visible trajectory,
and turn them into reusable guidance.

Rules:
- Stay grounded in the provided question and
  trajectory.
- Do not invent flaws just to be critical.
- Do not speculate about hidden reasoning not shown
  in the trajectory.
- Do not rewrite the answer.
- Do not give generic advice like "be more careful."
- Write specific guidance that is grounded in this
  case but transferable to similar questions.

Output format (exactly two blocks):

<reflection_rubrics>
Critical Requirement / Analytical Approach 
/ Communication & Structure / Pitfalls
</reflection_rubrics>

<reflection_takeaways>
1-4 concise portable lessons
</reflection_takeaways>
\end{verbatim}
\end{tcolorbox}

The user message provides the question and the full research trajectory, instructing the model to examine the final answer, trace strengths and weaknesses, and produce both output blocks.

\subsubsection{Judge Scoring Prompt}

An LLM judge (Gemini) evaluates each reflection candidate with privileged access to the judge scoring rubrics, their scores, and justifications for the sampled trajectory. The judge scores each candidate on how useful it would be for future attempts:

\begin{tcolorbox}[colback=blue!3, colframe=blue!40, title=Reflection Judge System Prompt (abridged), breakable]
\small
\begin{verbatim}
You are an expert judge evaluating reflection
candidates produced by a research agent.

Scoring Dimensions (each 0.0 to 1.0):

Diagnostic Accuracy (weight 0.4)
- Does it correctly identify the main strengths and
  weaknesses? Is it aligned with the evaluator
  justifications and visible trajectory?
  Penalize: missing a major failure clearly flagged
  by the evaluator; inventing problems not supported
  by the trajectory; misidentifying root causes.

Specificity (weight 0.3)
- Is the guidance concrete, actionable, and tailored
  for next attempt of the question?
  Penalize: generic advice ("be more thorough");
  checklist items that could apply to any task;
  merely restating evaluator language.

Scope & Balance (weight 0.3)
- Is the guidance helpful for other questions? 
  Are takeaways complementary?
  Penalize: missing an important category; repeating
  the same point across sections.

Rules:
- Score based on the reflection's quality as future
  guidance, not on the trajectory's quality.
- A reflection on a perfect trajectory can score low
  if it is generic.
- A reflection on a failed trajectory can score high
  if it precisely identifies root causes.
\end{verbatim}
\end{tcolorbox}

The user message provides the original question, all reflection candidates, and the graded trajectory with evaluator feedback. The judge returns per-candidate scores across all three dimensions.

\subsubsection{Injection Formats}

Retrieved reflection rubrics and takeaways are injected as a \texttt{<reference\_examples>} block appended to the user message before tokenization. The two retrieval modes from Section~\ref{sec:evolving-policy}---within-episode refinement and cross-episode transfer---use distinct preambles to condition the agent's use of the retrieved guidance.

\paragraph{Within-episode injection.} When the agent revisits a question it has attempted before (repeat phase of the windowed curriculum), the accepted reflection from the previous encounter is injected:

\begin{tcolorbox}[colback=blue!3, colframe=blue!40, title=Within-Episode Injection Format, breakable]
\small
\begin{verbatim}
<reference_examples>
You have attempted this exact question before. Below
are the rubrics and reflections from your previous
attempt. Use them to improve your planning and
answering -- avoid past mistakes and build on what
worked.

## Your Previous Attempt on This Question:

### Rubrics:
{rubrics from bank item}

### Takeaways:
{reflections from bank item}
</reference_examples>

{Original user question}
\end{verbatim}
\end{tcolorbox}

\paragraph{Cross-episode injection.} For a new query (new-query phase), the bank retrieves reflection rubrics and takeaways from the top-$k$ semantically similar queries as exemplars:

\begin{tcolorbox}[colback=blue!3, colframe=blue!40, title=Cross-Episode Injection Format, breakable]
\small
\begin{verbatim}
<reference_examples>
Below are rubrics and reflections from similar
questions that were previously analyzed. Use them as
reference to guide your planning and answering --
adapt to the current question's specific needs.
Do NOT copy them verbatim; extract what is relevant
and adjust.

## Similar Question 1:
{question from bank item 1}

### Rubrics:
{rubrics from bank item 1}

### Takeaways:
{reflections from bank item 1}

---

## Similar Question 2:
{question from bank item 2}

### Rubrics:
{rubrics from bank item 2}

### Takeaways:
{reflections from bank item 2}
</reference_examples>

{Original user question}
\end{verbatim}
\end{tcolorbox}

In both cases, the reference block precedes the original user question, so the agent sees the guidance before beginning its structured plan. The within-episode format signals direct self-improvement (``avoid past mistakes''), while the cross-episode format signals analogical transfer (``adapt to the current question's specific needs'').

\section{Theoretical Analysis}
\phantomsection
\label{app:theory}

\subsection{Value of Stage Information}
\label{app:theory:conditioning}

In the main text, Theorem~\ref{thm:stage-info-main} is stated using the informal notation
\(h,c,z\) inside conditional expectations for readability.
In this appendix, we make the probability space explicit:
\(H,C,Z\) denote random variables, while \(h,c,z\) denote realized values.
We also slightly sharpen the strictness condition in the theorem by requiring the compared stages
to have positive conditional probability given \(c\); without this requirement,
the pointwise quantities \(\mathbb E[U(H,a)\mid C=c,Z=z]\) on impossible stage events are version-dependent.

The result below is a value-of-information statement specialized to structured agent reasoning.
It compares two information structures, \(\sigma(C)\) and \(\sigma(C,Z)\).
Unlike full state-abstraction theorems, it does \emph{not} assume that \((C,Z)\) is sufficient for the entire history \(H\);
it only quantifies the gain from making the stage label explicit under a compressed decision state.

\begin{assumption}[Setup for Theorem~\ref{thm:stage-info-main}]
\label{assump:stage-info-setup}
Throughout this subsection, the following assumptions hold.
\begin{enumerate}
    \item \textbf{Underlying probability space.}
    There exists a probability space \((\Omega,\mathcal F,\mathbb P)\) and a random decision history
    \[
    H:\Omega \to \mathcal H,
    \]
    where \((\mathcal H,\mathcal B(\mathcal H))\) is a measurable space.
    The law of \(H\) is denoted by \(d := \mathbb P \circ H^{-1}\).
    We interpret \(H\) as a random reachable decision point sampled from some fixed distribution over histories.

    \item \textbf{Finite action space.}
    The action space \(\mathcal A\) is a finite nonempty set.
    Consequently, every maximization and argmax over \(\mathcal A\) is well-defined and nonempty.

    \item \textbf{Compressed context and stage label.}
    There exists a measurable map
    \[
    \phi:\mathcal H \to \mathcal C
    \]
    into a standard Borel space \((\mathcal C,\mathcal B(\mathcal C))\), and a measurable stage map
    \[
    \psi:\mathcal H \to [K] := \{1,\dots,K\},
    \]
    such that
    \[
    C := \phi(H), \qquad Z := \psi(H).
    \]
    Since \(\mathcal C\) is standard Borel and \([K]\) is finite, the regular conditional objects used below exist.

    \item \textbf{Integrable utility.}
    For each \(a \in \mathcal A\), there exists a measurable utility function
    \[
    U(\cdot,a):\mathcal H \to \mathbb R
    \]
    such that
    \[
    \mathbb E\big[\,|U(H,a)|\,\big] < \infty.
    \]
    In our application, \(U(H,a)\) can be instantiated as a continuation utility or continuation value
    associated with choosing action \(a\) at decision point \(H\).
\end{enumerate}
\end{assumption}

\begin{definition}[Conditional mean utilities and stage probabilities]
\label{def:stage-info-cond}
Adopt Assumption~\ref{assump:stage-info-setup}.
For each \(a \in \mathcal A\), fix measurable versions
\[
\bar q(\cdot,a):\mathcal C \to \mathbb R,
\qquad
q(\cdot,\cdot,a):\mathcal C \times [K] \to \mathbb R
\]
such that
\[
\bar q(C,a)
=
\mathbb E\!\left[U(H,a)\mid \sigma(C)\right]
\qquad \text{almost surely,}
\]
and
\[
q(C,Z,a)
=
\mathbb E\!\left[U(H,a)\mid \sigma(C,Z)\right]
\qquad \text{almost surely.}
\]

For each \(z \in [K]\), fix a measurable version
\[
p(z\mid \cdot):\mathcal C \to [0,1]
\]
such that
\[
p(z\mid C)
=
\mathbb P\!\left(Z=z \mid \sigma(C)\right)
\qquad \text{almost surely.}
\]

We write
\[
\mathbb P_C := \mathbb P \circ C^{-1}
\]
for the law of \(C\), and define the stage support at a context \(c \in \mathcal C\) by
\[
S(c) := \{ z \in [K] : p(z\mid c) > 0 \}.
\]
\end{definition}

\begin{remark}[Well-definedness]
\label{rem:stage-info-well-defined}
Because \(\mathcal A\) is finite and each \(U(H,a)\) is integrable, the random variables
\(\bar q(C,a)\) and \(q(C,Z,a)\) are integrable for every \(a \in \mathcal A\).
Indeed, by Jensen's inequality for conditional expectations,
\[
|\bar q(C,a)|
=
\left|\mathbb E[U(H,a)\mid \sigma(C)]\right|
\le
\mathbb E[|U(H,a)|\mid \sigma(C)]
\qquad \text{a.s.,}
\]
and therefore
\[
\mathbb E\big[\,|\bar q(C,a)|\,\big]
\le
\mathbb E\big[\,|U(H,a)|\,\big]
<
\infty.
\]
The same argument gives
\[
\mathbb E\big[\,|q(C,Z,a)|\,\big] < \infty.
\]
Hence all expectations below are finite.
\end{remark}

\begin{definition}[Flat and stage-conditioned decision values]
\label{def:stage-info-values}
Adopt Assumption~\ref{assump:stage-info-setup} and Definition~\ref{def:stage-info-cond}.
We define
\[
V_{\mathrm{flat}}
:=
\mathbb E\!\left[
\max_{a \in \mathcal A} \bar q(C,a)
\right],
\qquad
V_{\mathrm{stage}}
:=
\mathbb E\!\left[
\max_{a \in \mathcal A} q(C,Z,a)
\right].
\]
These are exactly the appendix versions of the quantities appearing in Theorem~\ref{thm:stage-info-main} in the main text.
Equivalently,
\[
V_{\mathrm{flat}}
=
\mathbb E\!\left[
\max_{a \in \mathcal A} \mathbb E[U(H,a)\mid C]
\right],
\qquad
V_{\mathrm{stage}}
=
\mathbb E\!\left[
\max_{a \in \mathcal A} \mathbb E[U(H,a)\mid C,Z]
\right].
\]
\end{definition}

The following lemma is the only structural identity needed in the proof of Theorem~\ref{thm:stage-info-main}.

\begin{lemma}[Collapsing stage-conditioned utility to the flat information structure]
\label{lem:collapse-stage-cond}
Adopt Assumption~\ref{assump:stage-info-setup} and Definition~\ref{def:stage-info-cond}.
Then for every action \(a \in \mathcal A\),
\[
\bar q(C,a)
=
\sum_{z=1}^K p(z\mid C)\, q(C,z,a)
\qquad \text{almost surely.}
\]
\end{lemma}

\begin{proof}
Fix an arbitrary action \(a \in \mathcal A\).
Define the random variable
\[
R_a
:=
\sum_{z=1}^K p(z\mid C)\, q(C,z,a).
\]
We will prove that \(R_a\) is a version of \(\mathbb E[U(H,a)\mid \sigma(C)]\).
Since \(\bar q(C,a)\) is also a version of \(\mathbb E[U(H,a)\mid \sigma(C)]\) by Definition~\ref{def:stage-info-cond},
the conclusion will follow.

We first note that \(R_a\) is \(\sigma(C)\)-measurable.
Indeed, for each fixed \(z\), both \(p(z\mid C)\) and \(q(C,z,a)\) are \(\sigma(C)\)-measurable,
so their product is \(\sigma(C)\)-measurable, and a finite sum of such terms is again \(\sigma(C)\)-measurable.

It remains to verify the defining property of conditional expectation.
Let \(G\) be any bounded \(\sigma(C)\)-measurable random variable.
We must show that
\[
\mathbb E[G R_a] = \mathbb E[G U(H,a)].
\]

Starting from the left-hand side and expanding the definition of \(R_a\), we obtain
\begin{align*}
\mathbb E[G R_a]
&=
\mathbb E\!\left[
G \sum_{z=1}^K p(z\mid C)\, q(C,z,a)
\right] \\
&=
\sum_{z=1}^K
\mathbb E\!\left[
G\, p(z\mid C)\, q(C,z,a)
\right].
\end{align*}
For each fixed \(z\), the random variable \(G q(C,z,a)\) is \(\sigma(C)\)-measurable and integrable.
Therefore, by the defining property of conditional expectation,
\begin{align*}
\mathbb E\!\left[
G\, p(z\mid C)\, q(C,z,a)
\right]
&=
\mathbb E\!\left[
G\, q(C,z,a)\, \mathbb E[\mathbf 1\{Z=z\}\mid \sigma(C)]
\right] \\
&=
\mathbb E\!\left[
G\, q(C,z,a)\, \mathbf 1\{Z=z\}
\right].
\end{align*}
Summing over \(z\) yields
\begin{align*}
\mathbb E[G R_a]
&=
\sum_{z=1}^K
\mathbb E\!\left[
G\, q(C,z,a)\, \mathbf 1\{Z=z\}
\right] \\
&=
\mathbb E\!\left[
G \sum_{z=1}^K \mathbf 1\{Z=z\}\, q(C,z,a)
\right] \\ 
&= \mathbb E\! \left[ G q \left(C, Z, a \right) \right]
\end{align*}

By Definition~\ref{def:stage-info-cond},
\[
q(C,Z,a)
=
\mathbb E[U(H,a)\mid \sigma(C,Z)]
\qquad \text{almost surely.}
\]
Because \(G\) is \(\sigma(C)\)-measurable and \(\sigma(C)\subseteq \sigma(C,Z)\), the random variable \(G\) is also \(\sigma(C,Z)\)-measurable.
Applying the defining property of conditional expectation once more gives
\[
\mathbb E\!\left[ G\, q(C,Z,a) \right]
=
\mathbb E\!\left[ G\, U(H,a) \right].
\]
Combining the previous displays, we conclude that
\[
\mathbb E[G R_a] = \mathbb E[G U(H,a)]
\]
for every bounded \(\sigma(C)\)-measurable random variable \(G\).
This finishes the proof.


\end{proof}

\begin{definition}[Aliasing gap]
\label{def:aliasing-gap}
Adopt Assumption~\ref{assump:stage-info-setup} and Definition~\ref{def:stage-info-cond}.
For \(c \in \mathcal C\), define the stage-aliasing gap
\[
\Delta_{\mathrm{alias}}(c)
:=
\sum_{z=1}^K p(z\mid c)\, \max_{a \in \mathcal A} q(c,z,a)
-
\max_{a \in \mathcal A}
\sum_{z=1}^K p(z\mid c)\, q(c,z,a).
\]
Since \(\mathcal A\) and \([K]\) are finite and the chosen versions \(p,q\) are measurable, the function
\(\Delta_{\mathrm{alias}}:\mathcal C \to \mathbb R\) is measurable.
\end{definition}

We can now restate and strengthen the theorem from the main text.

\begin{theorem}[Restatement and strengthening of Theorem~\ref{thm:stage-info-main}]
\label{thm:stage-info-full}
Adopt Assumption~\ref{assump:stage-info-setup} and Definitions~\ref{def:stage-info-cond}--\ref{def:aliasing-gap}.
Then
\[
V_{\mathrm{stage}} - V_{\mathrm{flat}}
=
\mathbb E\!\left[\Delta_{\mathrm{alias}}(C)\right]
\ge 0.
\]
Consequently,
\[
V_{\mathrm{stage}} \ge V_{\mathrm{flat}}.
\]

Moreover, for \(\mathbb P_C\)-almost every \(c \in \mathcal C\), the following are equivalent:
\begin{enumerate}
    \item \(\Delta_{\mathrm{alias}}(c)=0\).
    \item There exists an action
    \[
    a_c^\star \in \bigcap_{z \in S(c)} \arg\max_{a \in \mathcal A} q(c,z,a).
    \]
\end{enumerate}

In particular, if there exists a measurable set \(\mathcal C_0 \subseteq \mathcal C\) with
\[
\mathbb P_C(\mathcal C_0) > 0
\]
and two distinct stages \(z \neq z'\) such that for every \(c \in \mathcal C_0\),
\[
p(z\mid c) > 0,
\qquad
p(z'\mid c) > 0,
\qquad
\arg\max_{a \in \mathcal A} q(c,z,a)
\;\cap\;
\arg\max_{a \in \mathcal A} q(c,z',a)
=
\varnothing,
\]
then
\[
V_{\mathrm{stage}} > V_{\mathrm{flat}}.
\]
\end{theorem}

\begin{proof}

First, we rewrite \(V_{\mathrm{stage}}\) and \(V_{\mathrm{flat}}\) to show their subtraction form as in $\mathbb E \left[ \triangle_{\mathrm{alias}}(C) \right]$. By Definition~\ref{def:stage-info-values},
\[
V_{\mathrm{stage}}
=
\mathbb E\!\left[
\max_{a \in \mathcal A} q(C,Z,a)
\right].
\]
For each fixed \(z \in [K]\), define
\[
M_z(C) := \max_{a \in \mathcal A} q(C,z,a).
\]
Since \(\mathcal A\) is finite and \(q(\cdot,z,a)\) is measurable for each \(a\), the random variable \(M_z(C)\) is \(\sigma(C)\)-measurable.
Also, because \(Z\) takes values in the finite set \([K]\),
\[
\max_{a \in \mathcal A} q(C,Z,a)
=
\sum_{z=1}^K \mathbf 1\{Z=z\}\, M_z(C)
\qquad \text{almost surely.}
\]
Taking conditional expectation given \(\sigma(C)\), we obtain
\begin{align*}
\mathbb E\!\left[
\max_{a \in \mathcal A} q(C,Z,a)
\mid \sigma(C)
\right]
&=
\mathbb E\!\left[
\sum_{z=1}^K \mathbf 1\{Z=z\} M_z(C)
\mid \sigma(C)
\right] \\
&=
\sum_{z=1}^K
\mathbb E\!\left[
\mathbf 1\{Z=z\} M_z(C)
\mid \sigma(C)
\right] \\
&=
\sum_{z=1}^K
M_z(C)\,
\mathbb E\!\left[
\mathbf 1\{Z=z\}
\mid \sigma(C)
\right] \\
&=
\sum_{z=1}^K p(z\mid C)\, M_z(C).
\end{align*}
Taking expectation once more and using the tower property gives
\[
V_{\mathrm{stage}}
=
\mathbb E\!\left[
\sum_{z=1}^K p(z\mid C)\, \max_{a \in \mathcal A} q(C,z,a)
\right].
\]

We next rewrite \(V_{\mathrm{flat}}\).
By Definition~\ref{def:stage-info-values},
\[
V_{\mathrm{flat}}
=
\mathbb E\!\left[
\max_{a \in \mathcal A} \bar q(C,a)
\right].
\]
By Lemma~\ref{lem:collapse-stage-cond},
\[
\bar q(C,a)
=
\sum_{z=1}^K p(z\mid C)\, q(C,z,a)
\qquad \text{almost surely.}
\]
Substituting this identity into the preceding display yields
\[
V_{\mathrm{flat}}
=
\mathbb E\!\left[
\max_{a \in \mathcal A}
\sum_{z=1}^K p(z\mid C)\, q(C,z,a)
\right].
\]

Then we can simply see that by Definition~\ref{def:aliasing-gap}
\[
V_{\mathrm{stage}} - V_{\mathrm{flat}}
=
\mathbb E\!\left[
\sum_{z=1}^K p(z\mid C)\, \max_{a \in \mathcal A} q(C,z,a)
-
\max_{a \in \mathcal A}
\sum_{z=1}^K p(z\mid C)\, q(C,z,a)
\right]
=
\mathbb E\!\left[\Delta_{\mathrm{alias}}(C)\right].
\]

It remains to show that \(\Delta_{\mathrm{alias}}(c)\ge 0\) for every \(c\).
Fix an arbitrary \(c \in \mathcal C\).
For every action \(a \in \mathcal A\) and every stage \(z \in [K]\), we have
\[
q(c,z,a) \le \max_{a' \in \mathcal A} q(c,z,a').
\]
Multiplying both sides by the nonnegative quantity \(p(z\mid c)\) and summing over \(z\) gives
\[
\sum_{z=1}^K p(z\mid c)\, q(c,z,a)
\le
\sum_{z=1}^K p(z\mid c)\, \max_{a' \in \mathcal A} q(c,z,a').
\]
Since this inequality holds for every \(a \in \mathcal A\), it continues to hold after taking the maximum over \(a\) on the left-hand side, and this means that $\Delta_{\mathrm{alias}}(c)\ge 0$.
Therefore
\[
V_{\mathrm{stage}} - V_{\mathrm{flat}}
=
\mathbb E\!\left[\Delta_{\mathrm{alias}}(C)\right]
\ge 0,
\]
which implies
\[
V_{\mathrm{stage}} \ge V_{\mathrm{flat}}.
\]


\textbf{Then we are ready to make characterizations when the stage-aliasing gap is zero.}
Since \(\mathcal A\) is finite and nonempty, there exists at least one action
\[
a_c^{\mathrm{flat}} \in \arg\max_{a \in \mathcal A}
\sum_{z=1}^K p(z\mid c)\, q(c,z,a).
\]
By Definition~\ref{def:aliasing-gap},
\begin{align*}
\Delta_{\mathrm{alias}}(c)
&=
\sum_{z=1}^K p(z\mid c)\, M_z(c)
-
\sum_{z=1}^K p(z\mid c)\, q(c,z,a_c^{\mathrm{flat}}) \\
&=
\sum_{z=1}^K p(z\mid c)\,
\big(
M_z(c)-q(c,z,a_c^{\mathrm{flat}})
\big).
\end{align*}
Each summand on the right-hand side is nonnegative, because \(M_z(c)\) is the maximum of \(q(c,z,\cdot)\) over \(\mathcal A\).
We now prove the two directions separately.

\smallskip
\noindent
\emph{(i) If \(\Delta_{\mathrm{alias}}(c)=0\), then there exists a common maximizer over \(S(c)\).} This implies that 

\[
\sum_{z=1}^K p(z\mid c)\,
\big(
M_z(c)-q(c,z,a_c^{\mathrm{flat}})
\big)
=0.
\]
Every term in the sum is nonnegative.
Therefore, for every \(z\) such that \(p(z\mid c)>0\), we must have
\[
M_z(c)-q(c,z,a_c^{\mathrm{flat}})=0,
\]
i.e.,
\[
q(c,z,a_c^{\mathrm{flat}})=M_z(c)=\max_{a \in \mathcal A} q(c,z,a).
\]
Hence
\[
a_c^{\mathrm{flat}} \in \arg\max_{a \in \mathcal A} q(c,z,a)
\qquad \text{for every } z \in S(c).
\]
Equivalently,
\[
a_c^{\mathrm{flat}} \in \bigcap_{z \in S(c)} \arg\max_{a \in \mathcal A} q(c,z,a),
\]
so a common maximizer exists.

\smallskip
\noindent
\emph{(ii) If there exists a common maximizer over \(S(c)\), then \(\Delta_{\mathrm{alias}}(c)=0\).}

Conversely, assume there exists an action
\[
a_c^\star \in \bigcap_{z \in S(c)} \arg\max_{a \in \mathcal A} q(c,z,a).
\]
Then for every \(z \in S(c)\),
\[
q(c,z,a_c^\star)=M_z(c).
\]
For \(z \notin S(c)\), we have \(p(z\mid c)=0\), so those stages contribute nothing to any weighted sum below.
Hence
\[
\sum_{z=1}^K p(z\mid c)\, q(c,z,a_c^\star)
=
\sum_{z=1}^K p(z\mid c)\, M_z(c).
\]
Since \(a_c^\star\) is one feasible action in the maximization,
\[
\max_{a \in \mathcal A}
\sum_{z=1}^K p(z\mid c)\, q(c,z,a)
\ge
\sum_{z=1}^K p(z\mid c)\, q(c,z,a_c^\star)
=
\sum_{z=1}^K p(z\mid c)\, M_z(c).
\]
On the other hand, the following is obvious and has been shown above
\[
\max_{a \in \mathcal A}
\sum_{z=1}^K p(z\mid c)\, q(c,z,a)
\le
\sum_{z=1}^K p(z\mid c)\, M_z(c).
\]
Therefore equality holds:
\[
\max_{a \in \mathcal A}
\sum_{z=1}^K p(z\mid c)\, q(c,z,a)
=
\sum_{z=1}^K p(z\mid c)\, M_z(c),
\]
which is exactly
\[
\Delta_{\mathrm{alias}}(c)=0.
\]

Combining (i) and (ii), we conclude that for every \(c \in \mathcal C\),
\[
\Delta_{\mathrm{alias}}(c)=0
\quad\Longleftrightarrow\quad
\bigcap_{z \in S(c)} \arg\max_{a \in \mathcal A} q(c,z,a)\neq \varnothing.
\]
Since all objects are defined only up to \(\mathbb P_C\)-null sets through the chosen versions of conditional expectations and conditional probabilities, the equivalence is interpreted for \(\mathbb P_C\)-almost every \(c\).


\textbf{Finally, we can prove the strict inequality criterion}.
Assume there exists a measurable set \(\mathcal C_0 \subseteq \mathcal C\) such that
\[
\mathbb P_C(\mathcal C_0)>0
\]
and two distinct stages \(z \neq z'\) such that for every \(c \in \mathcal C_0\),
\[
p(z\mid c)>0,
\qquad
p(z'\mid c)>0,
\qquad
\arg\max_{a \in \mathcal A} q(c,z,a)
\cap
\arg\max_{a \in \mathcal A} q(c,z',a)
=
\varnothing.
\]

By the zero-gap equivalence we proved above, and that we know $\Delta_{\mathrm{alias}}(c)\ge 0$, this means that 
$\Delta_{\mathrm{alias}}(C)\ge 0$ almost surely, and thus 

\[
V_{\mathrm{stage}} > V_{\mathrm{flat}}.
\]

This completes the proof. 
\end{proof}

\begin{remark}[Interpretation]
\label{rem:stage-info-interpretation}
Theorem~\ref{thm:stage-info-full} strengthens the main-text theorem in two ways.
First, it identifies the exact gain from making the stage label explicit:
\[
V_{\mathrm{stage}} - V_{\mathrm{flat}}
=
\mathbb E[\Delta_{\mathrm{alias}}(C)].
\]
Second, it shows that the gain is strict exactly when the compressed context \(C\) aliases decision points
whose stage-conditioned optimal actions disagree, which is the usual case.
This is the precise sense in which explicit stage structure helps reasoning under compressed local context.
\end{remark}
\subsection{Judge-Aligned Stage-Weighted Credit Assignment}
\phantomsection
\label{app:theory:credit}

This subsection formalizes the main-text intuition behind Theorem~\ref{thm:stage-credit-main}.
As in Appendix~\ref{app:theory:conditioning}, we work on a probability space
\((\Omega,\mathcal F,\mathbb P)\).
However, unlike Appendix~\ref{app:theory:conditioning}, which studies a single random decision point,
we now analyze full rollouts and their policy-gradient signals.

\paragraph{Setup and notation.}
We suppress the rollout index \(i\) used in Section~\ref{sec:ssgrpo} and analyze a single generic rollout.
Let \(\tau\) denote a rollout sampled from a policy-induced distribution \(p_\theta(\tau)\), and let
$\mathcal B_k$ denote the token set of stage \(k\), matching the notation in the main text.
For each token step \(t\), let
\[
H_t := (q,a_{<t},o_{<t})
\]
denote the random history before taking action \(a_t\), and define the stage score-function sum
\[
\Gamma_k := \sum_{t\in\mathcal B_k} \nabla_\theta \log \pi_\theta(a_t\mid H_t).
\]
We assume \(\pi_\theta(a\mid h)\) is differentiable in \(\theta\), the environment dynamics are independent of \(\theta\), and \(\Gamma_k\) is square-integrable for every \(k\).

For each stage \(k\in[K]\), let
\[
R_k:\Omega\to\mathbb R
\]
denote the observed stagewise judge score, and let
\[
Y_k:\Omega\to\mathbb R
\]
denote a latent true process score for stage \(k\).
Both are assumed integrable.

As in Section~\ref{sec:ssgrpo}, let
\[
\Lambda=(\lambda_{k,j})\in[0,1]^{K\times K}
\]
be a causal stage-weight matrix satisfying
\[
\lambda_{k,j}=0 \quad \text{for } j<k,
\qquad
\lambda_{k,k}=1.
\]

\paragraph{Definitions.}
For each stage \(k\), define the oracle process-level gradient contribution
\begin{equation}
    g_k^\star
:=
\sum_{j=k}^K \lambda_{k,j}\,\mathbb E[\Gamma_k Y_j].
\label{eqa:oracle_grad}
\end{equation}
This is the stage-\(k\) gradient contribution that would be induced by the stage-dependent return  
\(
    G_{i,k}^{\Lambda}
:=
\sum_{j=k}^K \lambda_{k,j} R_{i,j}
\)
if the latent true stage scores were directly observable.

Define the judge-induced stage-weighted signal
\begin{equation}
    g_k^\Lambda
:=
\sum_{j=k}^K \lambda_{k,j}\,\mathbb E[\Gamma_k R_j],
\label{eqa:judge_grad}
\end{equation}
and the terminal-broadcast signal
\begin{equation}
    g_k^{\mathrm{term}}
:=
\mathbb E[\Gamma_k R_K],
\label{eqa:broadcast_grad}
\end{equation}

Finally, define
\begin{equation}
    M_k^\Lambda
:=
\left\|
\mathbb E[\Gamma_k Y_K]
-
\sum_{j=k}^K \lambda_{k,j}\,\mathbb E[\Gamma_k Y_j]
\right\|_2.
\label{eqa:oracle_diff}
\end{equation}
The quantity \(M_k^\Lambda\) measures how much oracle process-level signal is omitted when one uses only the final-stage score \(Y_K\) instead of the full stage-weighted oracle target.

\begin{assumption}[Judge alignment]
\label{assump:judge-aligned-credit}
For each stage pair \(k\le j\), there exists a constant \(\epsilon_{k,j}\ge 0\) such that
\begin{equation}
    \left\|
\mathbb E\!\left[\Gamma_k (R_j-Y_j)\right]
\right\|_2
\le
\epsilon_{k,j}.
\label{eqa:judge_alignment}
\end{equation}
That is, the observed judge score \(R_j\) approximates the latent true score \(Y_j\) in the gradient-relevant direction defined by \(\Gamma_k\).
\end{assumption}

\begin{theorem}[Benefit of stage-weighted credit, informal]
\label{thm:stage-credit-main}
If the omitted true intermediate signal outweighs the cumulative judge misalignment, then stage-weighted credit yields a strictly better gradient approximation:
\[
\|g_k^\Lambda-g_k^\star\|_2
<
\|g_k^{\mathrm{term}}-g_k^\star\|_2.
\]
\end{theorem}

\begin{theorem}[Formal version of Theorem~\ref{thm:stage-credit-main}]
\label{thm:judge-aligned-stage-credit}
Under Assumption~\ref{assump:judge-aligned-credit}, the following hold for every stage \(k\in[K]\).

\begin{enumerate}
    \item \textbf{Error of the stage-weighted signal relative to the oracle target:}
    \[
    \left\| g_k^\Lambda - g_k^\star \right\|_2
    \le
    \sum_{j=k}^K \lambda_{k,j}\,\epsilon_{k,j}.
    \]

    \item \textbf{Lower bound for terminal broadcast relative to the oracle target:}
    \[
    \left\| g_k^{\mathrm{term}} - g_k^\star \right\|_2
    \ge
    M_k^\Lambda - \epsilon_{k,K}.
    \]

    \item \textbf{Comparison.}
    If
    \[
    \sum_{j=k}^K \lambda_{k,j}\,\epsilon_{k,j}
    <
    M_k^\Lambda - \epsilon_{k,K},
    \]
    then
    \[
    \left\| g_k^\Lambda - g_k^\star \right\|_2
    <
    \left\| g_k^{\mathrm{term}} - g_k^\star \right\|_2.
    \]
\end{enumerate}
In particular, when judge misalignment is sufficiently small relative to the omitted intermediate-stage oracle signal, the stage-weighted signal is strictly closer than terminal broadcast to the intended process-level gradient contribution.
\end{theorem}

\begin{proof}
Fix a stage \(k\). \textbf{First, we show (1).} By definition~\ref{eqa:judge_grad} and~\ref{eqa:oracle_grad},
\[
g_k^\Lambda - g_k^\star
=
\sum_{j=k}^K \lambda_{k,j}\,\mathbb E[\Gamma_k R_j]
-
\sum_{j=k}^K \lambda_{k,j}\,\mathbb E[\Gamma_k Y_j]
=
\sum_{j=k}^K \lambda_{k,j}\,\mathbb E[\Gamma_k (R_j-Y_j)].
\]

Taking the Euclidean norm and applying the triangle inequality yields
\[
\left\| g_k^\Lambda - g_k^\star \right\|_2
\le
\sum_{j=k}^K \lambda_{k,j}
\left\|
\mathbb E[\Gamma_k (R_j-Y_j)]
\right\|_2.
\]
Applying Assumption~\ref{assump:judge-aligned-credit} results in the desired form,
\[
\left\| g_k^\Lambda - g_k^\star \right\|_2
\le
\sum_{j=k}^K \lambda_{k,j}\,\epsilon_{k,j}.
\]


\textbf{Second, we try to prove (2).}
By definition~\ref{eqa:broadcast_grad} and~\ref{eqa:oracle_grad},
\[
g_k^{\mathrm{term}} - g_k^\star
=
\mathbb E[\Gamma_k R_K]
-
\sum_{j=k}^K \lambda_{k,j}\,\mathbb E[\Gamma_k Y_j]
=
\underbrace{\mathbb E[\Gamma_k(R_K-Y_K)]}_{=:E_k}
+
\underbrace{\left(
\mathbb E[\Gamma_k Y_K]
-
\sum_{j=k}^K \lambda_{k,j}\,\mathbb E[\Gamma_k Y_j]
\right)}_{=:D_k}.
\]

From definition~\ref{eqa:oracle_diff} we know that $\|D_k\|_2 = M_k^\Lambda$
Applying the reverse triangle inequality,
\[
\|E_k + D_k\|_2 \ge \|D_k\|_2 - \|E_k\|_2,
\]
we obtain
\[
\left\| g_k^{\mathrm{term}} - g_k^\star \right\|_2
\ge
M_k^\Lambda
-
\left\| \mathbb E[\Gamma_k(R_K-Y_K)] \right\|_2.
\]
Using Assumption~\ref{assump:judge-aligned-credit} with \(j=K\),
\[
\left\| \mathbb E[\Gamma_k(R_K-Y_K)] \right\|_2 \le \epsilon_{k,K}.
\]
Hence
\[
\left\| g_k^{\mathrm{term}} - g_k^\star \right\|_2
\ge
M_k^\Lambda - \epsilon_{k,K}.
\]


\textbf{Finally, we are ready to state the last result.} We have shown that 
\[
\left\| g_k^\Lambda - g_k^\star \right\|_2
\le
\sum_{j=k}^K \lambda_{k,j}\,\epsilon_{k,j}, \qquad \qquad 
\left\| g_k^{\mathrm{term}} - g_k^\star \right\|_2
\ge
M_k^\Lambda - \epsilon_{k,K}.
\]
Therefore, if
\[
\sum_{j=k}^K \lambda_{k,j}\,\epsilon_{k,j}
<
M_k^\Lambda - \epsilon_{k,K},
\]
then
\[
\left\| g_k^\Lambda - g_k^\star \right\|_2
<
\left\| g_k^{\mathrm{term}} - g_k^\star \right\|_2.
\]
This proves item 3 and completes the proof.
\end{proof}

\begin{corollary}[A sufficient condition from score MSE]
\label{cor:judge-alignment-mse}
Suppose in addition that for some constants \(B_k\ge 0\) and \(\delta_j\ge 0\),
\[
\mathbb E[\|\Gamma_k\|_2^2] \le B_k,
\qquad
\mathbb E[(R_j-Y_j)^2] \le \delta_j^2.
\]
Then Assumption~\ref{assump:judge-aligned-credit} holds with
\[
\epsilon_{k,j} := \sqrt{B_k}\,\delta_j.
\]
Consequently,
\[
\left\| g_k^\Lambda - g_k^\star \right\|_2
\le
\sqrt{B_k}\sum_{j=k}^K \lambda_{k,j}\delta_j,
\]
and
\[
\left\| g_k^{\mathrm{term}} - g_k^\star \right\|_2
\ge
M_k^\Lambda - \sqrt{B_k}\,\delta_K.
\]
\end{corollary}

\begin{proof}
Fix \(k\le j\).
By Cauchy--Schwarz,
\begin{align*}
\left\|
\mathbb E[\Gamma_k(R_j-Y_j)]
\right\|_2
&\le
\mathbb E\!\left[\|\Gamma_k\|_2\,|R_j-Y_j|\right] \\
&\le
\left(\mathbb E[\|\Gamma_k\|_2^2]\right)^{1/2}
\left(\mathbb E[(R_j-Y_j)^2]\right)^{1/2} \\
&\le
\sqrt{B_k}\,\delta_j.
\end{align*}
Thus Assumption~\ref{assump:judge-aligned-credit} holds with \(\epsilon_{k,j}=\sqrt{B_k}\delta_j\).
The displayed bounds then follow immediately from Theorem~\ref{thm:judge-aligned-stage-credit}.
\end{proof}

\begin{remark}
Theorem~\ref{thm:judge-aligned-stage-credit} and Corollary~\ref{cor:judge-alignment-mse} highlight a fundamental trade-off: the advantage of stage-weighted credit depends on the strength of the omitted intermediate-stage oracle signal \(M_k^\Lambda\) versus the quality of judge alignment, captured by the score errors \(\delta_j\). 
Crucially, this establishes that as long as the intermediate stages contain sufficient true process signal (\emph{i.e.}, \(M_k^\Lambda\) is sufficiently large), the benefit of capturing this dense signal strictly outweighs the accumulated noise from an imperfect intermediate judge \(\delta_j > 0\). In such regimes, stage-weighted credit serves as a strictly better approximation of the intended process objective than relying solely on terminal broadcast.

\end{remark}

\subsection{Judge-Gated Co-Evolution of Policy and Reflection Rubric Bank}
\label{app:theory:evolution}

This subsection gives an idealized local analysis of the reflection mechanism in Section~\ref{sec:evolving-policy}.
Unlike Appendix~\ref{app:theory:conditioning}, which studies the value of explicit stage information,
and Appendix~\ref{app:theory:credit}, which studies stagewise credit assignment,
the object here is the self-evolution loop itself:
the same shared backbone both solves the current task and produces rubric-grounded reflections
that can later be reused within-episode and across episodes.

The analysis shares the implementation's separation of token-level updates:
the task objective updates trajectory tokens, whereas the reflection objective is computed on
reflection tokens only.
Concretely, the selected rollout is treated as a fixed conditioning context during the reflection update,
so the reflection-utility gradient does not backpropagate through the sampled trajectory.
The co-evolution effect therefore comes from parameter sharing:
task updates and reflection updates act on the same backbone parameters,
even though they are computed from different generated token blocks.
The objective $U$ below is a judge-gated surrogate: it weights reflection utility by an acceptance indicator.
In the implemented GRPO recipe, all candidate scores contribute to reflection training, while only the highest-scored valid candidate is stored.
The theorem therefore characterizes this surrogate under its stated alignment assumption; it does not directly guarantee improvement for the implemented update or the evolving retrieval bank.

The key assumption below is a \emph{judge-gated local positive-transfer} condition.
This is not intended as a global claim that every generated reflection is always useful throughout training.
Rather, it isolates the local regime in which reflections that are \emph{accepted by the judge}
behave like a helpful auxiliary objective.
Similar gradient-similarity and positive-transfer conditions are standard in the auxiliary-learning
and multi-task optimization literature:
helpful auxiliary objectives tend to align with the main-task gradient,
whereas misaligned gradients induce negative transfer
\citep{du2019adapting,wu2020understanding,yu2020gradient,liu2021conflict,wang2021gradient}.
At the same time, not all task combinations are beneficial to train together
\citep{standley2020which}.
This perspective is also consistent with theory on shared representation learning
and task relevance in transfer
\citep{maurer2016benefit,tripuraneni2020theory,chen2022active}.

\begin{assumption}[Setup for judge-gated co-evolution]
\label{assump:coevo-setup}
Throughout this subsection, let $\Theta \subseteq \mathbb R^d$ be an open parameter domain.
The following assumptions hold.
\begin{enumerate}
    \item \textbf{Query distribution for the task objective.}
    There is a query random variable
    \[
    Q \sim \mathcal D.
    \]

    \item \textbf{Current reflection rubric bank.}
    At the current training step, the reflection rubric bank is a fixed measurable object $\mathcal M$.
    This corresponds to analyzing a single parameter update while treating the bank state as fixed.

    \item \textbf{Task rollout distribution.}
    Given query $Q=q$ and bank $\mathcal M$, a rollout
    \[
    T \sim p_\theta(\cdot \mid q,\mathcal M)
    \]
    is sampled from the deployed task policy.

    \item \textbf{Fixed reflection-context distribution.}
    There exists a fixed measurable distribution $\xi$ over query--trajectory pairs
    \[
    (\widetilde Q,\widetilde T) \sim \xi.
    \]
    One may think of $\xi$ as the distribution of trajectories selected for reflection training
    at the current update, e.g., after first sampling rollout groups under the current behavior policy
    and then selecting one trajectory per query.
    For the local analysis in this subsection, $\xi$ is treated as fixed.
    This matches the implementation where the reflection objective applies gradients only on
    reflection tokens and treats the selected trajectory as a fixed prompt / context.

    \item \textbf{Shared-backbone reflection generation.}
    Given $(\widetilde Q,\widetilde T)=(q,\tau)$ and bank $\mathcal M$,
    a rubric-grounded reflection
    \[
    S \sim r_\theta(\cdot \mid q,\tau,\mathcal M)
    \]
    is sampled.
    Although we write $p_\theta$ and $r_\theta$ separately for clarity,
    both are induced by the same underlying autoregressive backbone $\pi_\theta$
    under different prompts / contexts.

    \item \textbf{Task score, reflection-utility scores, and judge gate.}
    There exist measurable, integrable random quantities
    \[
    R(Q,T;\mathcal M) \in \mathbb R,
    \qquad
    \Delta^{\mathrm{w}}(\widetilde Q,\widetilde T,S;\mathcal M) \in \mathbb R,
    \qquad
    \Delta^{\mathrm{c}}(\widetilde Q,\widetilde T,S;\mathcal M) \in \mathbb R,
    \]
    and a measurable acceptance indicator
    \[
    A(\widetilde Q,\widetilde T,S;\mathcal M) \in \{0,1\}.
    \]
    Here $R$ is the judged task score of the deployed rollout,
    $\Delta^{\mathrm{w}}$ is the judged within-episode usefulness of the reflection,
    $\Delta^{\mathrm{c}}$ is the judged cross-episode transfer usefulness,
    and $A=1$ means that the reflection is judged sufficiently useful to be accepted
    for downstream adaptation / storage.
    These objects do not depend directly on $\theta$ except through the sampled random variables
    and the fixed bank $\mathcal M$.

    \item \textbf{Differentiability and score-function regularity.}
    For $\mathcal D$-almost every $q$,
    the conditional distribution $p_\theta(\tau \mid q,\mathcal M)$ is differentiable in $\theta$.
    For $\xi$-almost every $(q,\tau)$,
    the conditional distribution $r_\theta(s \mid q,\tau,\mathcal M)$ is differentiable in $\theta$.
    Differentiation may be interchanged with the expectations below,
    and all score-function terms introduced later are square-integrable.

    \item \textbf{Smoothness.}
    The objectives $J_{\mathrm{task}}$ and $U$ defined below are continuously differentiable.
    Moreover, $J_{\mathrm{task}}$ is $L_J$-smooth and $U$ is $L_U$-smooth on $\Theta$.
\end{enumerate}
\end{assumption}

\begin{definition}[Task objective and judge-gated memory objective]
\label{def:coevo-obj}
Adopt Assumption~\ref{assump:coevo-setup}.
Define the task-rollout and reflection score-function sums
\[
\Gamma^{\mathrm{traj}}
:=
\nabla_\theta \log p_\theta(T \mid Q,\mathcal M),
\qquad
\Gamma^{\mathrm{ref}}
:=
\nabla_\theta \log r_\theta(S \mid \widetilde Q,\widetilde T,\mathcal M).
\]
Because both conditional distributions are induced by the same backbone $\pi_\theta$,
these admit tokenwise decompositions
\[
\Gamma^{\mathrm{traj}}
=
\sum_{t=1}^{|T|}
\nabla_\theta \log \pi_\theta(a_t \mid h_t),
\qquad
\Gamma^{\mathrm{ref}}
=
\sum_{u=1}^{|S|}
\nabla_\theta \log \pi_\theta(s_u \mid c_u),
\]
for the appropriate rollout histories $h_t$ and reflection-generation contexts $c_u$.

For weights $\beta_{\mathrm{w}},\beta_{\mathrm{c}} \ge 0$, define
\[
J_{\mathrm{task}}(\theta)
:=
\mathbb E\!\left[R(Q,T;\mathcal M)\right],
\]
where the expectation is with respect to
\[
Q \sim \mathcal D,
\qquad
T \sim p_\theta(\cdot \mid Q,\mathcal M),
\]
and define
\[
U(\theta)
:=
\mathbb E\!\left[
A(\widetilde Q,\widetilde T,S;\mathcal M)
\Big(
\beta_{\mathrm{w}}\Delta^{\mathrm{w}}(\widetilde Q,\widetilde T,S;\mathcal M)
+
\beta_{\mathrm{c}}\Delta^{\mathrm{c}}(\widetilde Q,\widetilde T,S;\mathcal M)
\Big)
\right],
\]
where the expectation is with respect to
\[
(\widetilde Q,\widetilde T) \sim \xi,
\qquad
S \sim r_\theta(\cdot \mid \widetilde Q,\widetilde T,\mathcal M).
\]
We also define the combined co-evolution objective
\[
J_{\mathrm{coevo}}(\theta)
:=
J_{\mathrm{task}}(\theta) + U(\theta).
\]
\end{definition}

\begin{assumption}[Judge-gated local positive transfer]
\label{assump:judge-gated-transfer}
Adopt Assumption~\ref{assump:coevo-setup} and Definition~\ref{def:coevo-obj}.
Let
\[
g := \nabla J_{\mathrm{task}}(\theta),
\]
and define the ungated reflection-utility score-gradient random vector
\[
\Psi
:=
\Gamma^{\mathrm{ref}}
\Big(
\beta_{\mathrm{w}}\Delta^{\mathrm{w}}(\widetilde Q,\widetilde T,S;\mathcal M)
+
\beta_{\mathrm{c}}\Delta^{\mathrm{c}}(\widetilde Q,\widetilde T,S;\mathcal M)
\Big).
\]
Assume there exist constants $p_0 \in (0,1]$ and $\mu > 0$ such that
\[
\mathbb P(A=1) \ge p_0,
\qquad
\left\langle
g,\,
\mathbb E[\Psi \mid A=1]
\right\rangle
\ge \mu.
\]
That is, reflections that pass the judge are accepted with nontrivial probability,
and their conditional expected gradient contribution is positively aligned with the task gradient.
\end{assumption}

\begin{theorem}[Judge-gated shared-backbone co-evolution]
\label{thm:judge-gated-coevo}
Adopt Assumptions~\ref{assump:coevo-setup}--\ref{assump:judge-gated-transfer}
and Definition~\ref{def:coevo-obj}.
Let
\[
g := \nabla J_{\mathrm{task}}(\theta),
\qquad
h := \nabla U(\theta).
\]
Then the following hold.

\begin{enumerate}


    \item \textbf{Mutual improvement.}
    For every step size $\eta > 0$,
    \[
    U(\theta + \eta g) - U(\theta)
    \ge
    \eta p_0 \mu
    -
    \frac{L_U \eta^2}{2}\|g\|_2^2,
    \]
    and
    \[
    J_{\mathrm{task}}(\theta + \eta h) - J_{\mathrm{task}}(\theta)
    \ge
    \eta p_0 \mu
    -
    \frac{L_J \eta^2}{2}\|h\|_2^2.
    \]
    In particular, if
    \[
    0 < \eta <
    \min\left\{
    \frac{2p_0\mu}{L_U\|g\|_2^2},
    \frac{2p_0\mu}{L_J\|h\|_2^2}
    \right\},
    \]
    then a task-improving step also improves the judge-gated memory objective,
    and a memory-improving step also improves the task objective.

    \item \textbf{Dominance over task-only training with static memory.}
    Consider the task-only update
    \[
    \theta_{\mathrm{stat}}^{+} := \theta + \eta g,
    \]
    which updates the deployed task policy while assigning zero explicit training signal
    to reflection quality,
    and the co-evolution update
    \[
    \theta_{\mathrm{co}}^{+} := \theta + \eta (g+h).
    \]
    Then
    \[
    J_{\mathrm{task}}(\theta_{\mathrm{co}}^{+})
    -
    J_{\mathrm{task}}(\theta_{\mathrm{stat}}^{+})
    \ge
    \eta p_0 \mu
    -
    L_J \eta^2 \|g\|_2 \|h\|_2
    -
    \frac{L_J \eta^2}{2}\|h\|_2^2.
    \]
    Consequently, if the reflection is sufficiently good,
    \[
    p_0 \mu
    >
    L_J \eta
    \left(
    \|g\|_2\|h\|_2 + \frac{1}{2}\|h\|_2^2
    \right),
    \]
    then
    \[
    J_{\mathrm{task}}(\theta_{\mathrm{co}}^{+})
    >
    J_{\mathrm{task}}(\theta_{\mathrm{stat}}^{+}).
    \]
\end{enumerate}
\end{theorem}

\begin{proof}
To prove the desired statement, we first derive the formula for $g$ and $h$.
Let
\(
\mu_\theta(dq,d\tau)
:=
\mathcal D(dq)\, p_\theta(d\tau \mid q,\mathcal M)
\)
denote the joint law of $(Q,T)$.
By Definition~\ref{def:coevo-obj},
\[
J_{\mathrm{task}}(\theta)
=
\int R(q,\tau;\mathcal M)\,\mu_\theta(dq,d\tau).
\]
By Assumption~\ref{assump:coevo-setup}, differentiation may be interchanged with the integral.
Since $\mathcal D$ does not depend on $\theta$, by the usual log-derivative trick,
\begin{align*}
\nabla J_{\mathrm{task}}(\theta)
&=
\int
R(q,\tau;\mathcal M)\,
\nabla p_\theta(\tau \mid q,\mathcal M)\,
\mathcal D(dq)\,d\tau \\
&=
\int
R(q,\tau;\mathcal M)\,
p_\theta(\tau \mid q,\mathcal M)\,
\nabla \log p_\theta(\tau \mid q,\mathcal M)\,
\mathcal D(dq)\,d\tau \\
&=
\mathbb E\!\left[
\Gamma^{\mathrm{traj}} R(Q,T;\mathcal M)
\right].
\end{align*}

We next derive the formula for $h$.
Write
\[
\Delta(q,\tau,s;\mathcal M)
:=
\beta_{\mathrm{w}}\Delta^{\mathrm{w}}(q,\tau,s;\mathcal M)
+
\beta_{\mathrm{c}}\Delta^{\mathrm{c}}(q,\tau,s;\mathcal M),
\]
so that
\[
U(\theta)
=
\mathbb E\!\left[
A(\widetilde Q,\widetilde T,S;\mathcal M)\,
\Delta(\widetilde Q,\widetilde T,S;\mathcal M)
\right].
\]
Let
\[
\nu_\theta(dq,d\tau,ds)
:=
\xi(dq,d\tau)\,
r_\theta(ds \mid q,\tau,\mathcal M)
\]
denote the joint law of $(\widetilde Q,\widetilde T,S)$.
Then
\[
U(\theta)
=
\int
A(q,\tau,s;\mathcal M)\Delta(q,\tau,s;\mathcal M)\,
\nu_\theta(dq,d\tau,ds).
\]
Differentiating under the integral sign gives
\begin{align*}
\nabla U(\theta)
&=
\int
A(q,\tau,s;\mathcal M)\Delta(q,\tau,s;\mathcal M)\,
\xi(dq,d\tau)\,
\nabla r_\theta(s \mid q,\tau,\mathcal M)\,ds \\
&=
\int
A(q,\tau,s;\mathcal M)\Delta(q,\tau,s;\mathcal M)\,
\xi(dq,d\tau)\,
r_\theta(s \mid q,\tau,\mathcal M)\,
\nabla \log r_\theta(s \mid q,\tau,\mathcal M)\,ds \\
&=
\mathbb E\!\left[
A(\widetilde Q,\widetilde T,S;\mathcal M)\,
\Gamma^{\mathrm{ref}}\,
\Delta(\widetilde Q,\widetilde T,S;\mathcal M)
\right] \\ 
&= \mathbb E\!\left[
A \Psi
\right] := h.
\end{align*}


By the definition of $A$ and Assumption~\ref{assump:judge-gated-transfer},
\[
\mathbb E[A\Psi]
=
\mathbb P(A=1)\,
\mathbb E[\Psi \mid A=1].
\]
and that 
\begin{align}
\langle g,h\rangle
&=
\left\langle
g,\,
\mathbb P(A=1)\,\mathbb E[\Psi \mid A=1]
\right\rangle \\
&=
\mathbb P(A=1)\,
\left\langle
g,\,
\mathbb E[\Psi \mid A=1]
\right\rangle \\
&\ge
p_0 \mu \label{eq:item_2},
\end{align}

By the assumption~\ref{assump:coevo-setup}, $U$ is $L_U$-smooth,
for every vector $v \in \mathbb R^d$ we have the standard lower bound
\[
U(\theta+v)
\ge
U(\theta)
+
\langle \nabla U(\theta), v \rangle
-
\frac{L_U}{2}\|v\|_2^2.
\]
Applying this with $v=\eta g$ and using Eq.~\ref{eq:item_2} yields
\[
U(\theta+\eta g)-U(\theta)
\ge
\eta \langle h,g\rangle
-
\frac{L_U\eta^2}{2}\|g\|_2^2
\ge
\eta p_0\mu
-
\frac{L_U\eta^2}{2}\|g\|_2^2.
\]
Likewise, because $J_{\mathrm{task}}$ is $L_J$-smooth,
\[
J_{\mathrm{task}}(\theta+\eta h)-J_{\mathrm{task}}(\theta)
\ge
\eta \langle g,h\rangle
-
\frac{L_J\eta^2}{2}\|h\|_2^2
\ge
\eta p_0\mu
-
\frac{L_J\eta^2}{2}\|h\|_2^2.
\]
If
\[
0 < \eta <
\min\left\{
\frac{2p_0\mu}{L_U\|g\|_2^2},
\frac{2p_0\mu}{L_J\|h\|_2^2}
\right\},
\]
then both right-hand sides are strictly positive.

To prove the other part, define
\[
\phi(t)
:=
J_{\mathrm{task}}
\big(
\theta_{\mathrm{stat}}^{+} + t\eta h
\big),
\qquad
t\in[0,1].
\]
Then
\[
\phi(1)-\phi(0)
=
J_{\mathrm{task}}(\theta_{\mathrm{co}}^{+})
-
J_{\mathrm{task}}(\theta_{\mathrm{stat}}^{+})
=
\int_0^1 \phi'(t)\,dt.
\]
By the chain rule,
\[
\phi'(t)
=
\eta
\left\langle
\nabla J_{\mathrm{task}}(\theta_{\mathrm{stat}}^{+}+t\eta h),
h
\right\rangle.
\]
Since $\nabla J_{\mathrm{task}}$ is $L_J$-Lipschitz,
\[
\left\|
\nabla J_{\mathrm{task}}(\theta_{\mathrm{stat}}^{+}+t\eta h)
-
\nabla J_{\mathrm{task}}(\theta_{\mathrm{stat}}^{+})
\right\|_2
\le
L_J t\eta \|h\|_2.
\]
Hence, for every $t\in[0,1]$, by Cauchy-Schwarz
\[
\left\langle
\nabla J_{\mathrm{task}}(\theta_{\mathrm{stat}}^{+}+t\eta h),
h
\right\rangle
\ge
\left\langle
\nabla J_{\mathrm{task}}(\theta_{\mathrm{stat}}^{+}),
h
\right\rangle
-
L_J t\eta \|h\|_2^2.
\]
Integrating from $0$ to $1$ gives
\[
J_{\mathrm{task}}(\theta_{\mathrm{co}}^{+})
-
J_{\mathrm{task}}(\theta_{\mathrm{stat}}^{+})
\ge
\eta
\left\langle
\nabla J_{\mathrm{task}}(\theta_{\mathrm{stat}}^{+}),
h
\right\rangle
-
\frac{L_J\eta^2}{2}\|h\|_2^2.
\]
It remains to lower-bound the inner product on the right.
Again by $L_J$-Lipschitz continuity of the gradient,
\[
\left\|
\nabla J_{\mathrm{task}}(\theta_{\mathrm{stat}}^{+}) - g
\right\|_2
=
\left\|
\nabla J_{\mathrm{task}}(\theta+\eta g) - \nabla J_{\mathrm{task}}(\theta)
\right\|_2
\le
L_J \eta \|g\|_2.
\]
Therefore, by Cauchy--Schwarz,
\begin{align*}
\left\langle
\nabla J_{\mathrm{task}}(\theta_{\mathrm{stat}}^{+}),
h
\right\rangle
&=
\langle g,h\rangle
+
\left\langle
\nabla J_{\mathrm{task}}(\theta_{\mathrm{stat}}^{+}) - g,
h
\right\rangle \\
&\ge
\langle g,h\rangle
-
\left\|
\nabla J_{\mathrm{task}}(\theta_{\mathrm{stat}}^{+}) - g
\right\|_2 \|h\|_2 \\
&\ge
\langle g,h\rangle - L_J \eta \|g\|_2\|h\|_2 \\
&\ge
p_0\mu - L_J \eta \|g\|_2\|h\|_2,
\end{align*}
where the last step uses Eq.~\ref{eq:item_2}.
Substituting this bound above yields that 
\[
J_{\mathrm{task}}(\theta_{\mathrm{co}}^{+})
-
J_{\mathrm{task}}(\theta_{\mathrm{stat}}^{+})
\ge
\eta p_0\mu
-
L_J\eta^2 \|g\|_2\|h\|_2
-
\frac{L_J\eta^2}{2}\|h\|_2^2.
\]
This proves the desired lower bound.
If
\[
p_0 \mu
>
L_J \eta
\left(
\|g\|_2\|h\|_2 + \frac{1}{2}\|h\|_2^2
\right),
\]
then the right-hand side is strictly positive, and therefore
\[
J_{\mathrm{task}}(\theta_{\mathrm{co}}^{+})
>
J_{\mathrm{task}}(\theta_{\mathrm{stat}}^{+}).
\]
\end{proof}

\begin{remark}[Interpretation]
\label{rem:judge-gated-coevo}
Theorem~\ref{thm:judge-gated-coevo} identifies a sufficient local condition for positive transfer in the surrogate model.
The key quantity is a judge-gated local condition on the reflections that are actually accepted for reuse.
Under this condition, policy training improves accepted reflection utility,
and accepted reflection training improves task performance.
Moreover, the joint update on $J_{\mathrm{task}}+U$ yields a strictly larger first-order gain
in adapted task value than task-only training with a static memory objective.
This comparison holds for the stated surrogate objectives and fixed bank; it is not a general dominance result over inference-time memory methods.
The co-evolution effect survives because $g$ and $h$ still live in the same shared parameter space:
a task update changes the future reflection generator, and a reflection update changes the future task policy.
\end{remark}

\begin{remark}[Why parameter sharing matters]
\label{rem:shared-backbone-matters}
The first-order co-evolution mechanism above is specific to the shared-backbone design.
If rollout generation and reflection generation were parameterized by disjoint blocks
$\theta=(\theta_\pi,\theta_{\mathrm{ref}})$ with
$J_{\mathrm{task}}$ depending only on $\theta_\pi$
and $U$ depending only on $\theta_{\mathrm{ref}}$,
then
\[
\nabla J_{\mathrm{task}}(\theta) = (g_\pi,0),
\qquad
\nabla U(\theta) = (0,h_{\mathrm{ref}}),
\]
and hence
\[
\left\langle
\nabla J_{\mathrm{task}}(\theta),
\nabla U(\theta)
\right\rangle = 0.
\]
Thus the mutual-improvement phenomenon in Theorem~\ref{thm:judge-gated-coevo}
disappears at first order in the fully decoupled case.
This identifies parameter sharing as a prerequisite for the particular first-order transfer mechanism analyzed here; it does not establish the required gradient alignment in practice.
\end{remark}


\section{Experiment Details}
\label{appen:exp}

Our training, evaluation, infrastructure, and baseline selection build on DR Tulu~\citep{shao2025dr}, with adaptations for stagewise credit assignment and reflection training.

\subsection{Training Details}
\label{app:training_details}

\subsubsection{Supervised Fine-Tuning}

We fine-tune Qwen3-8B~\citep{yang2025qwen3} using the LLaMA-Factory framework~\citep{zheng2024llamafactory} with DeepSpeed ZeRO-3. The SFT data generation process is described in Appendix~\ref{app:sft-data}. Key hyperparameters:

\begin{table}[htbp]
\centering
\caption{Hyperparameters for supervised fine-tuning.}
\label{tab:sft-hyperparameters}
\small
\begin{adjustbox}{max width=\linewidth}
\begin{tabular}{ll}
\toprule
\textbf{Hyperparameter} & \textbf{Value} \\
\midrule
Base model & Qwen3-8B \\
Training epochs & 5 \\
Learning rate & $4 \times 10^{-5}$ \\
LR scheduler & Cosine with 10\% warmup \\
Batch size (per device) & 1 \\
Gradient accumulation steps & 16 \\
Effective batch size & 128 (8 GPUs $\times$ 16 accum.) \\
Max sequence length & 16,384 tokens \\
Precision & BF16 \\
Weight decay & 0.0 \\
Tool output masking & Span masking on \texttt{<tool\_output>} \\
DeepSpeed stage & ZeRO-3 \\
GPUs & 8 $\times$ NVIDIA H100 80GB \\
\bottomrule
\end{tabular}
\end{adjustbox}
\end{table}

We apply span masking to tool output tokens so the model does not receive gradient on search results, learning only to generate its own reasoning, tool calls, and answers.

\subsubsection{Reinforcement Learning}

We train all RL runs using our modified open-instruct codebase with Ray-based distributed training and vLLM inference engines. The RL setup follows DR.~Tulu's codebase with extensions for stagewise credit assignment and meta-policy training operations. Key hyperparameters:

\begin{table}[htbp]
\centering
\caption{Hyperparameters for reinforcement learning.}
\label{tab:rl-hyperparameters}
\small
\begin{adjustbox}{max width=\linewidth}
\begin{tabular}{ll}
\toprule
\textbf{Hyperparameter} & \textbf{Value} \\
\midrule
Algorithm & GRPO / SS-GRPO \\
Rollouts per prompt & 8 \\
Unique prompts per step & 32 \\
Effective batch size & 256 (32 $\times$ 8) \\
Learning rate & $5 \times 10^{-7}$ \\
LR scheduler & Constant \\
KL coefficient ($\beta$) & 0.001 \\
KL estimator & KL3 \\
Clip ratio ($\eta$) & PPO-style clipping \\
Temperature & 1.0 \\
Max response length & 18,432 tokens \\
Max prompt length & 8,192 tokens \\
Max total (pack) length & 26,624 tokens \\
Max tool calls per trajectory & 10 \\
DeepSpeed stage & ZeRO-3 with CPU offloading \\
Gradient checkpointing & Enabled \\
\midrule
\multicolumn{2}{l}{\textit{Rubric judge}} \\
Judge model & Gemini-3-Flash \\
Rubric generation model & Gemini-3-Flash \\
Rubric buffer cap (per stage) & 3, 2, 2, 3 \\
\midrule
\multicolumn{2}{l}{\textit{Reflection rubric bank (Meta-Policy and full \ours)}} \\
Bank embedding model & Qwen3-Embedding-0.6B \\
Retrieval top-$k$ & 2 \\
Reflection trajectories sampled & 1 \\
Windowed curriculum $K$ & 3 \\
Bank save frequency & Every 10 steps \\
\bottomrule
\end{tabular}
\end{adjustbox}
\end{table}

One additional hyperparameter is the stage-dependent matrix \(\Lambda\) mentioned in~\Cref{sec:ssgrpo}, which we set to be 
\[
\Lambda =
\begin{pmatrix}
1.0 & 0.4 & 0.6 & 0.8 \\
0   & 1.0 & 0.4 & 0.8 \\
0   & 0   & 1.0 & 0.8 \\
0   & 0   & 0   & 1.0
\end{pmatrix}.
\]
We use the DR.~Tulu RL dataset (\texttt{rl-research/dr-tulu-rl-data}), containing ${\sim}$4.9K deep research queries. The four RL recipes use a 600-step task-rollout budget; dedicated ablation runs use 2 nodes. The main run's allocation and stopping point are described below.
The four recipes in Sec.~\ref{subsec:rl_recipes} share the same \ours-SFT checkpoint, training data, search backend, Gemini-3-Flash judge, and core hyperparameters above. Baseline-RL instantiates DR Tulu's evolving answer-only GRPO objective under these shared settings.

\paragraph{Reward design.}
Task-policy rewards come from evolving judge rubrics, and reflection-policy rewards come from judge scores over rubric-grounded reflections. We use no auxiliary format, citation, or tool-use rewards, focusing supervision on semantic planning, research, synthesis, and experience reuse.

\subsection{Cost and Resource Accounting}
\label{app:resource-accounting}

\noindent\textbf{Resource and cost summary.}
\begin{center}
\small
\begin{tabular}{@{}p{0.32\linewidth}p{0.63\linewidth}@{}}
\toprule
\textbf{Item} & \textbf{Reported allocation or estimate} \\
\midrule
Main RL allocation & 10,368 H100 GPU-hours over 20 calendar days \\
600-step ablation allocation & 16 H100 GPUs for approximately one week per run \\
Task judge calls per batch & 32 rubric-generation calls + 256 trajectory-scoring calls \\
Estimated teacher API cost & \$368.66 (\ours) vs.\ \$884.65 (DR Tulu) \\
Estimated judge API cost & 60--70\% lower than DR Tulu \\
\bottomrule
\end{tabular}
\end{center}
The accounting scope and pricing assumptions for each item are detailed below.

\paragraph{Main-run GPU allocation.}
The main RL run used four nodes with eight NVIDIA H100-80GB GPUs each for 7 days, then two nodes for 13 days because throughput was API-limited. This gives $4\times8\times24\times7 + 2\times8\times24\times13 = 10{,}368$ allocated GPU-hours over 20 calendar days. This measure records reserved capacity, including API waits and infrastructure interruptions; it excludes SFT, separate ablations, and evaluation.

The run reached step 1,432 before repeated node hangs and subsequent cluster reassignment ended training. We report the last stable, fully evaluated checkpoint at step 1,400; the endpoint reflects compute availability.

\paragraph{Ablation training budget.}
Each dedicated 600-step ablation used two nodes with eight H100-80GB GPUs each for approximately eight days. The fixed evaluation sets used to compare methods across checkpoints and their cost rationale are described in Appendix~\ref{app:ablation-eval}.

\paragraph{Training judge calls.}
For each question with eight task rollouts, the judge makes one structured call to generate rubrics for all stages and eight calls to score the trajectories, each returning all stage scores. A batch of 32 questions therefore requires 32 rubric-generation calls and 256 trajectory-scoring calls before retries. Reflection generation and scoring are additional work, described below.

\paragraph{Judge API cost estimate.}
Joint scoring avoids the repeated transmission of long trajectories in DR Tulu's per-rubric scoring pipeline. Our token-based estimate gives 60--70\% lower judge-side API costs. The estimate uses input/output prices per million tokens of \$0.50/\$3.00 for Gemini-3-Flash, \$2.00/\$8.00 for GPT-4.1 rubric generation, and \$0.40/\$1.60 for GPT-4.1-mini scoring in DR Tulu.

\paragraph{Teacher API estimates.}
Table~\ref{tab:teacher-api-estimates} summarizes retained SFT data and estimated teacher generation costs. The estimates assume prices per million input/output tokens of \$1/\$6 for Gemini-3.1-Pro Batch and \$2/\$8 for GPT-5 standard API calls. Accounting is based on retained trajectories and does not capture the full costs of rejection, retries, or search.

\begin{table}[htbp]
\centering
\caption{Retained SFT data and estimated teacher API costs under the stated pricing assumptions.}
\label{tab:teacher-api-estimates}
\small
\begin{tabular}{lrr}
\toprule
 & \textbf{\ours} & \textbf{DR Tulu} \\
\midrule
Accepted trajectories & 10,744 & 13,062 \\
Teacher generation rounds & 31,445 & 60,203 \\
Teacher output tokens & 31.2M & 36.2M \\
Estimated teacher API cost & \$368.66 & \$884.65 \\
\bottomrule
\end{tabular}
\end{table}

\paragraph{Additional reflection work.}
Reflection generation, judging, and optimization add computation and API work beyond the task-rollout budget. The asynchronous pipeline overlaps this work with task inference and training when background jobs finish in time (Appendix~\ref{app:async-implementation}).

\subsection{Evaluation Details}
\label{app:evaluation-details}

\textbf{Final outputs are graded using benchmark-specific criteria independent of our training rubrics.}
The evaluation suite spans multiple model families: \textbf{HealthBench and ResearchQA use GPT-4-based graders}, while DRB uses its Gemini-based RACE protocol. These graders apply benchmark-specific scoring criteria independently of agent planning rubrics, the evolving training-judge buffer, and the reflection rubric bank.

\paragraph{Fixed evaluation sets for controlled ablations.}
\phantomsection
\label{app:ablation-eval}
The four RL recipes in Fig.~\ref{fig:ablation_study} and the additional unstructured-scaffold run in Fig.~\ref{fig:ablation_scaffolding_scaling}(b) are evaluated at six checkpoints on the same fixed random set of 100 questions per benchmark. This holds question composition fixed across methods and checkpoints while making repeated evaluation tractable. The 100 questions cover the entire DRB benchmark. \textbf{Otherwise, we estimate that evaluating all five runs at six checkpoints on the 1,000-question HealthBench evaluation set alone would cost more than \$3,000}. The final 1,400-step model is evaluated on the complete evaluation sets described below.

\paragraph{Scaffolding comparisons.}
\phantomsection
\label{app:scaffold-eval}
Figure~\ref{fig:ablation_scaffolding_scaling}(a) compares checkpoints distilled from the same teacher using our scaffold or an unstructured ReAct-style format. In panel (b), both checkpoints undergo 600 steps of RL under matched settings with DR Tulu's evolving answer-only GRPO objective. The curves report gains relative to each checkpoint's own SFT baseline, using the fixed evaluation sets described above. Panel (c) compares the two prompting formats on DRB with the same Gemini-3.1-Pro model and search backend. Appendix~\ref{appen:scaffold} describes the scaffold and SFT data generation.

\paragraph{Reflection retrieval at evaluation.}
\phantomsection
\label{app:reflection-eval}
The main results (\Cref{tab:model_performance}), RL ablations (\Cref{fig:ablation_study}), and short-form results (\Cref{tab:model_performance_2}) use direct zero-shot prompts without entries from the reflection rubric bank or previous reflections. \textbf{Figure~\ref{fig:ablation_study} measures the performance of the trained task policy without test-time memory.} During training, the Meta-Policy and full \ours recipes use reflection optimization together with bank retrieval.

Figure~\ref{fig:ablation_scaffolding_scaling}(d) evaluates the 600-step Baseline-RL and \ours-trained checkpoints from Sec.~\ref{subsec:rl_recipes} on all 100 DRB questions. We compare three inference settings: no retrieval, cross-episode reuse of reflections from related queries, and within-episode reuse of reflections from an earlier attempt on the same query. For cross-episode reuse, \ours{} retrieves from the reflection rubric bank accumulated during training. For Baseline-RL, we evaluate both retrieval from the same training bank and a bank built from scratch by summarizing test-time attempts, and report the better result. Within-episode reuse uses reflections from an earlier attempt on the same evaluation query. Appendix~\ref{app:bank-retrieval} describes the bank lookup mechanisms, and Appendix~\ref{app:rubric-bank-prompts} gives the reflection generation and context-injection prompts.

\subsubsection{Long-Form Benchmarks}
\label{app:longform-eval}

We evaluate on four long-form benchmarks:

\begin{itemize}[nosep,leftmargin=*]
    \item \textbf{HealthBench}~\citep{arora2025healthbench}: 1,000 subsampled medical questions spanning patient consultations, clinical guidelines, and health advice. Evaluation uses an LLM-as-judge (GPT-4) that grades each response against per-question rubrics on multiple axes (accuracy, completeness, context awareness, communication). The reported score is the overall rubric satisfaction rate.

    \item \textbf{ResearchQA}~\citep{yifei2025researchqa}: 756 scientific research questions with automatically generated rubrics derived from survey literature and model knowledge, with expert validation reported in the benchmark study. Evaluation uses an LLM-as-judge (GPT-4) that assesses coverage of each rubric item on a 5-point scale (Not at all / Barely / Moderately / Mostly / Completely). The reported score is the average normalized coverage.

    \item \textbf{DeepResearchBench (DRB)}~\citep{du2025deepresearch}: 100 complex research questions requiring long-form reports. We report the overall RACE (Reference-based Adaptive Criteria-driven Evaluation) score. The Gemini judge evaluates comprehensiveness, insight/depth, instruction following, and readability against reference reports using task-specific criteria and weights.

        \item \textbf{ResearchRubrics}~\citep{sharma2026researchrubrics}: 101 open-ended deep research prompts across diverse real-world domains, each paired with expert-written, prompt-specific rubrics. Evaluation uses an LLM-as-judge to assess rubric compliance under fine-grained criteria covering factual grounding, reasoning soundness, synthesis quality, relevance, clarity, and citation use. Rubric items include both positive and negative criteria with different weights, and the reported score is the binary score.
\end{itemize}

\subsubsection{Short-Form Benchmarks}
\label{app:shortform-eval}

We additionally evaluate on four short-form benchmarks to test transfer beyond the long-form RL training queries:

\begin{itemize}[nosep,leftmargin=*]
    \item \textbf{SimpleQA}~\citep{wei2024measuring}: 1,000 factual questions with exact-match evaluation via LLM grading (correct / incorrect / not attempted).
    \item \textbf{2WikiMultihopQA}~\citep{ho2020constructing}: 1000 multi-hop reasoning questions requiring evidence from multiple sources. LLM-as-judge evaluation.
    \item \textbf{WebWalker}~\citep{wu2025webwalker}: 680 web navigation questions. LLM-as-judge evaluation.
    \item \textbf{DeepSearchQA (DSQA)}~\citep{gupta2026deepsearchqa}: 900 search-intensive questions with exact-match grading.
\end{itemize}

Short-form benchmarks use zero-shot evaluation with the same agent pipeline and search tools as the long-form benchmarks. SFT includes short-form data; these evaluations test transfer beyond the long-form queries used in RL.

\paragraph{Note on Benchmarks.}
Our benchmark suite focuses on broad long-form answer quality. We omit SQAv2 (ScholarQA)~\citep{asai2024openscholar,bragg2025astabench}, whose emphasis is precise academic citation, and leave citation-specific RL and evaluation to future work. ResearchQA and ResearchRubrics include citation, grounding, or evidence-use criteria within their broader assessments.

\subsection{Infrastructure}
\label{app:infrastructure}

Our training and evaluation infrastructure builds on the DR~Tulu codebase~\citep{shao2025dr}, which provides the Ray-based distributed RL training loop, vLLM integration for rollout generation, asynchronous tool calling, sample packing, and one-step asynchronous training~\citep{noukhovitch2024asynchronous}. We use GRPO~\citep{shao2024deepseekmath} with token-level loss aggregation, tool output token masking~\citep{jin2025searchr1}, and a small KL penalty (0.001) for stability. Full hyperparameters are provided in the RL training table above.

We extend the base infrastructure in several ways to support \ours:

\begin{itemize}[nosep,leftmargin=*]
    \item \textbf{Stagewise scoring pipeline:} An asynchronous stagewise judge scoring module runs Gemini API calls in parallel with the training loop. Before retries, the nominal workload is 32 rubric-generation calls per step (one per question, jointly covering all stages) and 256 trajectory-scoring calls (one per trajectory, jointly returning all stage-rubric scores). Calls are parallelized using \texttt{asyncio.gather} with exponential-backoff retries; these counts exclude reflection judging and search calls (see Appendix~\ref{app:stagewise_judge}).

    \item \textbf{Meta-policy training pipeline:} A three-thread architecture defers reflection training by one step. Completed reflection batches from step~$N$ are trained in Phase~A of step~$N{+}1$ while the inference engine generates new rollouts (Appendix~\ref{app:async-implementation}).

    \item \textbf{Reflection rubric bank module:} A thread-safe in-memory store with a FAISS index over query embeddings (Qwen3-Embedding-0.6B on CPU). Bank operations (retrieval, prompt injection, and background reflection generation) are integrated into the data preparation thread with one-step deferred execution (see Appendix~\ref{app:async-implementation}).

    \item \textbf{Checkpoint management:} In addition to standard model and optimizer checkpoints (via DeepSpeed), we save the rubric buffer state, reflection rubric bank contents, and dataloader state at regular intervals to enable seamless training resumption with consistent memory and data ordering.
\end{itemize}

For evaluation, we use vLLM to serve the model checkpoint on a single GPU and run the agent pipeline with real-time tool execution via the MCP backend. Each benchmark evaluation varies around 2--18 hours depending on the number of examples, the complexity of the query, judging latency, and the average number of tool calls per trajectory.

\paragraph{Search tools.}
We expose two search tools to the agent. 
The primary tool is \texttt{google\_search}, implemented through Gemini-3-Flash (as the browsing model) with Google Search grounding enabled\footnote{https://ai.google.dev/gemini-api/docs/google-search}.
It returns AI-synthesized summaries with grounding snippets, and is mainly used for general web search, fact-seeking questions, and broad real-world information gathering. 
The second tool is \texttt{snippet\_search}, implemented through the Semantic Scholar API\footnote{\url{https://api.semanticscholar.org/api-docs}}. 
It retrieves text snippets from academic papers and is mainly used when the query requires scholarly evidence, paper-level context, or scientific literature support. 
Both tools are called through the same MCP backend, and their outputs are wrapped in \texttt{<tool\_output>} blocks before being returned to the agent.

\subsection{Baselines}
\label{app:baselines}

We compare against baselines from three categories, following the selection of Dr.~Tulu~\citep{shao2025dr} and supplementing with additional recent models. Scores for existing baselines are taken from the Dr.~Tulu paper where available; for models not covered, we run the same evaluation pipeline.

\begin{itemize}[nosep,leftmargin=*]
    \item \textbf{Closed deep research models:} Proprietary systems with full deep research capabilities, including OpenAI Deep Research~\citep{openai2025deepresearch}, Gemini Deep Research~\citep{google2025gemini}, Perplexity Deep Research~\citep{perplexity2025sonardeepresearch}, Claude Sonnet Search, Gemini 3.1 Pro + Search, and GPT-5 + Search.

    \item \textbf{Fixed pipeline deep research models:} Models that use a fixed multi-stage pipeline (search $\rightarrow$ synthesis $\rightarrow$ report) without end-to-end RL training, including WebThinker~\citep{li2025webweaver} and Ai2 ScholarQA~\citep{singh2025ai2}.

    \item \textbf{Open deep research models:} Open-weight models trained end-to-end for deep research, including Search-R1-7B~\citep{jin2025searchr1}, WebExplorer-8B~\citep{liu2025webexplorer}, Tongyi DeepResearch-30B-A3B~\citep{tongyi2025deepresearch}, and DR Tulu-8B~\citep{shao2025dr} (both SFT and RL checkpoints).
\end{itemize}

DR~Tulu is the closest comparison in training-query sources and the open deep research setting. The implementations differ in search backend (Google Search with Gemini grounding vs.\ Serper API) and SFT teacher (Gemini-3.1-Pro vs.\ GPT-5).

\section{Algorithm}
\label{app:algorithm}

The algorithm of the paper is described in full in~\Cref{alg:rubricem}. Here, $\mathbf r_{i,k}$ denotes the per-rubric score vector underlying stage score $R_{i,k}$. Reflection score superscripts denote diagnostic accuracy ($\mathrm{diag}$), specificity ($\mathrm{spec}$), and scope \& balance ($\mathrm{scope}$).


\algrenewcommand\algorithmiccomment[1]{\hfill{\color{teal}$\triangleright$ #1}}
\newcommand{\algsection}[1]{\Statex \vspace{1pt}\Statex \textcolor{teal}{$\triangleright$ #1}}

\algrenewcommand\algorithmiccomment[1]{\hfill{\color{teal}\scriptsize$\triangleright$ #1}}
\renewcommand{\algsection}[1]{\Statex \textcolor{teal}{\scriptsize$\triangleright$ #1}}

\begin{algorithm}[t]
\caption{\textsc{RubricEM}: Reinforcement Learning with Stage-Structured GRPO and Reflection Meta-Policy Training}
\label{alg:rubricem}
\small

\begin{algorithmic}[1]
\Require Structured SFT policy $\pi_{\theta}$, reference policy $\pi_{\rm ref}$, training queries $\mathcal{D}$,
tool environment $\mathcal{T}$, judge $\mathcal{J}$, stages $\mathcal{S}=\{\textsc{Plan},\textsc{Research},\textsc{Review},\textsc{Answer}\}$, stage matrix $\Lambda$, rollout size $n$, reflection samples $m$, active rubric caps $C_{1:K}$
\State Initialize active judge scoring rubric buffers $\mathcal{B}^{\rm act}_{q,k}\leftarrow\emptyset$ for each query $q$ and stage $k$
\State Initialize persistent judge scoring rubrics $\mathcal{B}^{\rm pers}_{q,k}$ when available, and reflection rubric bank $\mathcal{M}\leftarrow\emptyset$
\State Initialize deferred reflection batch $\mathcal{P}^{\rm ref}_{0}\leftarrow\emptyset$
\For{RL step $t=1,\ldots,T$}
    \State $(Q_t,\mu_t)\leftarrow \Call{WindowedBatch}{\mathcal{D},t}$ \Comment{$\mu_t\in\{\textsc{cross},\textsc{within}\}$}
    \ForAll{$q\in Q_t$}
        \State $E_q\leftarrow \Call{Retrieve}{\mathcal{M},q,\mu_t}$ \Comment{retrieve reusable experience}
    \EndFor

    \algsection{Deferred reflection meta-policy update}
    \If{$\mathcal{P}^{\rm ref}_{t-1}\neq\emptyset$}
        \State Update shared $\pi_{\theta}$ on reflection tokens using GRPO rewards in $\mathcal{P}^{\rm ref}_{t-1}$
    \EndIf

    \algsection{Stage-structured task-policy rollout and judging}
    \State Initialize task batch $\mathcal{P}^{\rm task}_{t}\leftarrow\emptyset$
    \ForAll{$q\in Q_t$ in parallel}
        \State Sample tool-augmented rollouts $\{\tau_i\}_{i=1}^{n}\sim \pi_{\theta}(\cdot\mid q,E_q;\mathcal{T})$ under scaffold $\mathcal{S}$ 
        \Comment{task policy rollouts}
        \State $\Delta\mathcal{B}_{q,1:K}\leftarrow
        \Call{GenerateStageRubrics}{\mathcal{J},q,\{\tau_i\}_{i=1}^{n},\mathcal{B}^{\rm act}_{q,1:K}}$
        \Comment{contrast rollouts}
        \For{$k=1,\ldots,K$}
            \State $\mathcal{B}^{\rm act}_{q,k}\leftarrow
            \mathcal{B}^{\rm act}_{q,k}\cup \Delta\mathcal{B}_{q,k}$
            \State $\mathcal{B}_{q,k}\leftarrow
            \mathcal{B}^{\rm pers}_{q,k}\cup \mathcal{B}^{\rm act}_{q,k}$
        \EndFor
        \State $(R_{i,k},\mathbf r_{i,k})\leftarrow
        \Call{ScoreStage}{\mathcal{J},q,\tau_i,\mathcal{B}_{q,k}}$ for all $i=1,\ldots,n$ and $k=1,\ldots,K$
        \For{$k=1,\ldots,K$}
            \State $G^{\Lambda}_{i,k}\leftarrow \sum_{j=k}^{K}\lambda_{k,j}R_{i,j}$ for all $i$
            \State $A_{i,k}\leftarrow
            \dfrac{G^{\Lambda}_{i,k}-{\rm mean}_{i'}(G^{\Lambda}_{i',k})}
            {{\rm std}_{i'}(G^{\Lambda}_{i',k})+\epsilon}$ for all $i$
            \State $\mathcal{B}^{\rm act}_{q,k}\leftarrow
            \Call{PruneByDiscrimination}{\mathcal{B}^{\rm act}_{q,k},\{\mathbf r_{i,k}\}_{i=1}^{n},C_k}$
            \Comment{evolve judge scoring rubric buffer}
        \EndFor
        \State $\mathcal{P}^{\rm task}_{t}\leftarrow
        \mathcal{P}^{\rm task}_{t}\cup\{(q,\tau_i,B_{i,1:K},A_{i,1:K})\}_{i=1}^{n}$
        \State Launch $\Call{PrepareReflectionBatch}{q,\{\tau_i,R_{i,1:K}\}_{i=1}^{n},\mathcal{B}_{q,1:K},\mathcal{M},m}$ asynchronously
    \EndFor

    \algsection{Task-policy update}
    \State Update $\pi_{\theta}$ on $\mathcal{P}^{\rm task}_{t}$ with SS-GRPO using stage advantages $A_{i,k}$ and KL to $\pi_{\rm ref}$ by Eq.~\ref{eqa:ss-grpo}
    \State $\mathcal{P}^{\rm ref}_{t}\leftarrow$ completed asynchronous reflection batches from step $t$
\EndFor
\State \Return trained policy $\pi_{\theta}$ and reflection rubric bank $\mathcal{M}$

\Statex
\Function{PrepareReflectionBatch}{$q,\{\tau_i,R_{i,1:K}\}_{i=1}^{n},\mathcal{B}_{q,1:K},\mathcal{M},m$}
    \State Sample one trajectory $\tau_s$ uniformly from $\{\tau_i\}_{i=1}^{n}$ \Comment{fixed reflection context}
    \State Generate reflection candidates $\{s_{\ell}\}_{\ell=1}^{m}\sim \pi_{\theta}^{\rm refl}(\cdot\mid q,\tau_s)$
    \For{$\ell=1,\ldots,m$}
        \State $(u^{\rm diag}_{\ell},u^{\rm spec}_{\ell},u^{\rm scope}_{\ell})\leftarrow
        \Call{JudgeReflection}{\mathcal{J},q,\tau_s,s_{\ell},R_{s,1:K},\mathcal{B}_{q,1:K}}$
        \State $u_{\ell}\leftarrow 0.4u^{\rm diag}_{\ell}+0.3u^{\rm spec}_{\ell}+0.3u^{\rm scope}_{\ell}$
    \EndFor
    \State $s^\star\leftarrow \arg\max_{\ell} u_{\ell}$ among valid reflection candidates
    \State $\mathcal{M}\leftarrow \Call{WriteBank}{\mathcal{M},q,s^\star}$ \Comment{store only the best accepted reflection}
    \State \Return $\{(q,\tau_s,s_{\ell},u_{\ell})\}_{\ell=1}^{m}$
\EndFunction
\end{algorithmic}
\end{algorithm}

\clearpage
\section{Limitations and Discussions}
\label{appen:limitation}

\noindent\textbf{Limitations.}
Our experiments involve long-horizon agentic RL with search tool calls and external LLM judging. These dependencies make the training loop more sensitive to infrastructure instability than training on fixed offline data.
During training, we occasionally encountered API delays and unreliable network connections, leading to variation in wall-clock execution and rollout--judge latencies across RL steps.
The main training run also required several server shutdowns and restarts. Although we resumed from checkpoints, such interruptions can introduce additional staleness in the asynchronous reflection branch, reflection rubric bank, and judge-feedback pipeline beyond the intended one-step lag.
Our results therefore reflect operation under realistic infrastructure constraints. A more stable, uninterrupted training environment could reduce stale feedback and improve training efficiency.
We also use Gemini-3-Flash as a cost-effective judge for rubric generation, stagewise scoring, and reflection evaluation.
A stronger or more specialized judge could provide more accurate stage-level credit and higher-quality reflection rewards, especially for subtle long-form research tasks, but would also increase cost and latency.
Exploring stronger judges, complementary judges, and calibrated judge ensembles is an important direction for future work.

\medskip
\noindent\textbf{Discussion.}
More broadly, our results suggest that LLM-generated rubrics can serve as a general interface for structuring agent behavior, assigning semantic credit, and accumulating reusable experience.
They also motivate a training-time view of meta-RL for language agents: a reflection meta-policy trained during RL allows judged experience to shape both the agent's parameters and its reusable textual memory.
This extends reflection and memory from inference-time prompting mechanisms into components of policy learning.
The recipe may also apply to other settings where quality is multidimensional and difficult to reduce to exact-answer correctness, including writing assistance, data analysis, scientific review, tutoring, and complex tool-use workflows.

As a pioneering framework for combining stagewise credit assignment and learned experience reuse in open-ended, long-horizon RL, \ours provides a starting point for finer-grained feedback and richer forms of adaptation.
Future work could refine credit assignment within semantic stages and extend experience reuse across more trials and tasks.
Jointly scaling model capacity, RL training duration, the volume and diversity of training data, and judge capability may yield complementary gains in task learning, reflection quality, and test-time adaptation.
Understanding how these dimensions interact is a promising direction for further developing open-ended agentic RL.

\end{document}